\documentclass[final,12pt]{colt2025} 
\makeatletter
\let\Ginclude@graphics\@org@Ginclude@graphics
\makeatother

\usepackage{amssymb}
\usepackage{dsfont}
\usepackage{amsmath}
\usepackage{mathtools}
\usepackage{mathtools}
\usepackage{pb-diagram}
\usepackage{comment}
\usepackage{float}
\usepackage{wrapfig}
\usepackage{caption}
\usepackage{graphicx}
\usepackage{enumitem}

\usepackage{graphicx}

\usepackage[T1]{fontenc}    
\usepackage{booktabs}       

\usepackage{amsfonts,amssymb,bbm}
\usepackage{amsmath}
\usepackage{graphicx}
\usepackage{adjustbox}
\usepackage[UKenglish]{isodate}
\usepackage{graphicx}
\usepackage{silence}
\usepackage[format=default,font=small,labelfont=it]{caption}
\definecolor{richblack}{rgb}{0.0, 0.25, 0.25}
\RequirePackage[x11names]{xcolor}         
    \definecolor{darkcerulean}{rgb}{0.03, 0.27, 0.49}
    \definecolor{smokyblack}{rgb}{0.06, 0.05, 0.03}
    \definecolor{warmblack}{rgb}{0.0, 0.26, 0.26}
    \definecolor{cobalt}{rgb}{0.0, 0.28, 0.67}
    \definecolor{darkcobalt}{rgb}{0.1, 0.38, 0.77}

\usepackage[overload]{empheq}

\usepackage{float}
\usepackage{afterpage}
\usepackage{enumitem}
\usepackage{multirow}
\usepackage{fancyvrb}
\usepackage{rotating}
    \usepackage{tikz}
    \usepackage{tikz-cd} 
    \pgfdeclarelayer{edgelayer}
    \pgfdeclarelayer{nodelayer}
    \pgfsetlayers{edgelayer,nodelayer,main}
    \tikzstyle{new style 0}=[fill={rgb,255: red,255; green,94; blue,247}, draw=black, shape=circle]
    \tikzstyle{pointy}=[fill=white, draw=black, shape=circle]
    \tikzstyle{pointy}=[->]
\usetikzlibrary{decorations.pathmorphing}
\usetikzlibrary{positioning, shapes.geometric, fit, backgrounds}
\usetikzlibrary{calc}

\usepackage{algpseudocode}

\definecolor{ForestGreen}{rgb}{0.0, 0.27, 0.13}

\newcommand{\ra}[1]{\renewcommand{\arraystretch}{#1}}
\usepackage{lscape}

\makeatletter
\newcommand{\pushright}[1]{\ifmeasuring@#1\else\omit\hfill$\displaystyle#1$\fi\ignorespaces}
\newcommand{\pushleft}[1]{\ifmeasuring@#1\else\omit$\displaystyle#1$\hfill\fi\ignorespaces}
\makeatother

\DeclareMathOperator{\reshape}{reshape}
\DeclareMathOperator{\vectorize}{vec}
\DeclareMathOperator{\layernorm}{\mathcal{L\!N}} 

\DeclareMathOperator{\softmax}{softmax}
\DeclareMathOperator{\smax}{smax}
\DeclareMathOperator{\softplus}{splus} 
\DeclareMathOperator{\sigmoid}{sig}
\DeclareMathOperator{\attention}{Att} 
\DeclareMathOperator{\multihead}{\mathcal{M\!H}} 
\DeclareMathOperator{\dotp}{dp}
\DeclareMathOperator{\neuraln}{\operatorname{\mathcal{P\!L}}}
\DeclareMathOperator{\tblock}{\mathcal{T\!B}}
\DeclareMathOperator{\tblockclass}{\mathcal{T\!B\!C}}
\DeclareMathOperator{\transformer}{\mathcal{T}}
\DeclareMathOperator{\transformerclass}{\mathcal{T}\!\mathcal{C}}

\DeclareMathOperator{\cleq}{\mbox{\tiny$\leq$}\!}

\newcommand{\tagA}{^{(3)}}
\newcommand{\tagB}{^{(2)}}
\newcommand{\tagC}{^{(1)}}

\newcommand{\order}{\mathfrak{o}}
\newcommand{\Order}{\mathfrak{O}}

\renewcommand{\phi}{\varphi}

\let\emptyset\varnothing

\newcommand{\myset}[2]{\lbrace{#1}, \ldots, {#2}\rbrace}
\newcommand{\mySet}[2]{\left\lbrace {#1} \middle\vert {#2} \right\rbrace}

\newcommand{\N}{\mathbb{N}}

\newcommand{\R}{\mathbb{R}}

\newcommand{\calP}{\mathcal{P}}

\newcommand{\eqdef}{\ensuremath{\,\raisebox{-1.2pt}{${\stackrel{\mbox{\upshape \scalebox{.42}{def.}}}{=}}$}}\,}
\newcommand{\eqorder}{\sim} 

\newcommand{\idim}{d_{\mathrm{in}}} 
\newcommand{\odim}{d_{\mathrm{out}}} 
\newcommand{\kdim}{d_{K}} 
\newcommand{\vdim}{d_{V}} 
\newcommand{\ldim}{d_{\mathrm{ff}}} 

\newtheorem{assumption}{Assumption}
\newtheorem{notation}{Notation}
\newtheorem{setting}{Setting}[section]

\NewDocumentCommand{\luca}{mo}{
    \IfValueF{#2}{
                        {{\scriptsize
                            \textcolor{green}{ 
                            \textbf{L:}
                            \textit{{#1}}
                            }
                        }}
        }
    \IfValueT{#2}{
                        \marginnote{{\scriptsize
                            \textcolor{green}{ 
                            \textbf{L:}
                            \textit{{#1}}
                            }
                        }}
        }
                    }
\NewDocumentCommand{\giulia}{mo}{
    \IfValueF{#2}{
                        {{\scriptsize
                            \textcolor{red}{ 
                            \textbf{GL:}
                            \textit{{#1}}
                            }
                        }}
        }
    \IfValueT{#2}{
                        \marginnote{{\scriptsize
                            \textcolor{red}{ 
                            \textbf{GL:}
                            \textit{{#1}}
                            }
                        }}
        }
}
\NewDocumentCommand{\anastasis}{mo}{
    \IfValueF{#2}{
                        {{\scriptsize
                            \textcolor{violet}{ 
                            \textbf{A:}
                            \textit{{#1}}
                            }
                        }}
        }
    \IfValueT{#2}{
                        \marginnote{{\scriptsize
                            \textcolor{violet}{ 
                            \textbf{A:}
                            \textit{{#1}}
                            }
                        }}
        }
                    }
                    
\NewDocumentCommand{\cody}{mo}{
    \IfValueF{#2}{
                        {{\scriptsize
                            \textcolor{orange}{ 
                            \textbf{A:}
                            \textit{{#1}}
                            }
                        }}
        }
    \IfValueT{#2}{
                        \marginnote{{\scriptsize
                            \textcolor{orange}{ 
                            \textbf{A:}
                            \textit{{#1}}
                            }
                        }}
        }
                    }

\NewDocumentCommand{\yannick}{mo}{
    \IfValueF{#2}{
                        {{\scriptsize
                            \textcolor{cyan}{ 
                            \textbf{Y:}
                            \textit{{#1}}
                            }
                        }}
        }
    \IfValueT{#2}{
                        \marginnote{{\scriptsize
                            \textcolor{cyan}{ 
                            \textbf{Y:}
                            \textit{{#1}}
                            }
                        }}
        }
                    } 

\definecolor{darkgreen}{rgb}{0.0, 0.2, 0.13}
\NewDocumentCommand{\xuwei}{mo}{
    \IfValueF{#2}{
                        {{\scriptsize
                            \textcolor{darkgreen}{ 
                            \textbf{X:}
                            \textit{{#1}}
                            }
                        }}
        }
    \IfValueT{#2}{
                        \marginnote{{\scriptsize
                            \textcolor{darkgreen}{ 
                            \textbf{X:}
                            \textit{{#1}}
                            }
                        }}
        }
                    }

\usepackage{appendix}
\usepackage{titletoc}
\newcommand\DoToC{%
  \startcontents
  \printcontents{}{1}{\textbf{Appendix Contents}\vskip3pt\hrule\vskip5pt}
  \vskip3pt\hrule\vskip5pt
}

\renewcommand \thepart{}
\renewcommand \partname{}

\usepackage{cleveref}
\newcounter{termcounter}
\renewcommand{\thetermcounter}{\Roman{termcounter}}
\crefname{term}{term}{terms}
\creflabelformat{term}{#2\textup{(#1)}#3}

\makeatletter
\def\term{\@ifnextchar[\term@optarg\term@noarg}
\def\term@optarg[#1]#2{%
  \textup{#1}%
  \def\@currentlabel{#1}%
  \def\cref@currentlabel{[][2147483647][]#1}%
  \cref@label[term]{#2}}
\def\term@noarg#1{%
  \refstepcounter{termcounter}%
  \textup{(\thetermcounter)}%
  \cref@label[term]{#1}}
\makeatother

\crefname{lemma}{lemma}{lemmata}
\Crefname{lemma}{Lemma}{Lemmata}
\crefname{notation}{notation}{notations}
\Crefname{notation}{Notation}{Notations}
\crefname{assumption}{assumption}{assumptions}
\Crefname{assumption}{Assumption}{Assumptions}
\crefname{example}{Example}{Examples}
\crefname{proposition}{Proposition}{Proposition}

\usepackage{comment}
\usepackage{microtype}

\usepackage{soul}

\NewDocumentCommand{\AK}{mo}{
    \IfValueF{#2}{
                        {{\scriptsize
                            \textcolor{violet}{ 
                            \textbf{A:}
                            \textit{{#1}}
                            }
                        }}
        }
    \IfValueT{#2}{
                        \marginnote{{\scriptsize
                            \textcolor{violet}{ 
                            \textbf{A:}
                            \textit{{#1}}
                            }
                        }}
        }
                    }
                    
\NewDocumentCommand{\Yan}{mo}{
    \IfValueF{#2}{
                        {{\scriptsize
                            \textcolor{cyan}{ 
                            \textbf{Y:}
                            \textit{{#1}}
                            }
                        }}
        }
    \IfValueT{#2}{
                        \marginnote{{\scriptsize
                            \textcolor{cyan}{ 
                            \textbf{Y:}
                            \textit{{#1}}
                            }
                        }}
        }
                    }

\definecolor{darkgreen}{rgb}{0.0, 0.2, 0.13}

\definecolor{cobaltt}{rgb}{0.0, 0.28, 0.67}
\definecolor{deepcerise}{rgb}{0.85, 0.2, 0.53}

\usepackage[utf8]{inputenc} 
\usepackage[T1]{fontenc}    
\usepackage{url}            
\usepackage{booktabs}       
\usepackage{amsfonts}       
\usepackage{nicefrac}       
\usepackage{microtype}      

\usepackage[utf8]{inputenc}

\usepackage{float}
\usepackage{lscape}

\renewcommand{\ge}{	\geq }
\renewcommand{\le}{	\leq }
\renewcommand{\geq}{\geqslant}
\renewcommand{\leq}{\leqslant}

\usepackage{wrapfig}

\title[Higher-Order Transformer Derivative Estimates for Pathwise Learning]{Higher-Order Transformer Derivative Estimates \hfill\\
for Explicit Pathwise Learning Guarantees}

\usepackage{times}

\coltauthor{
\Name{Yannick Limmer}
\Email{limmery@maths.ox.ac.uk} \\
\addr Department of Mathematics, Oxford-Man Institute, University of Oxford
\AND 
\Name{Anastasis Kratsios}
\Email{kratsioa@mcmaster.ca} \\
\addr Department of Mathematics, McMaster University, Vector Institute 
\AND 
\Name{Xuwei Yang}
\Email{yangx212@mcmaster.ca} \\
\addr Department of Mathematics, McMaster University
\AND 
\Name{Raeid Saqur}
\Email{kratsioa@mcmaster.ca} \\
\addr Department of Computer Science, University of Toronto, Vector Institute
\AND 
\Name{Blanka Horvath}
\Email{horvath@maths.ox.ac.uk} \\
\addr Department of Mathematics, Oxford-Man Institute, University of Oxford
} 

\begin{document}
\maketitle

\begin{abstract}
An inherent challenge in computing fully-\textit{explicit} generalization bounds for transformers involves obtaining covering number estimates for the given transformer class $\mathcal{T}$.  Crude estimates rely on a uniform upper bound on the local-Lipschitz constants of transformers in $\mathcal{T}$, and finer estimates require an analysis of their higher-order partial derivatives.  Unfortunately, these precise higher-order derivative estimates for (realistic) transformer models are not currently available in the literature as they are combinatorially delicate due to the intricate compositional structure of transformer blocks.  

This paper fills this gap by precisely estimating all the higher-order derivatives of all orders for the transformer model.  We consider realistic transformers with multiple (non-linearized) attention heads per block and layer normalization.  We obtain fully-explicit estimates of all constants in terms of the number of attention heads, the depth and width of each transformer block, and the number of normalization layers.  Further, we explicitly analyze the impact of various standard activation function choices (e.g.\ SWISH and GeLU).
As an application, we obtain explicit \textit{pathwise generalization bounds} for transformers on a \textit{single trajectory of an exponentially-ergodic Markov process} valid at a fixed future time horizon.  We conclude that real-world transformers can learn from $N$ (non-i.i.d.) samples of a single Markov process's trajectory at a rate of $\mathcal{O}\big(\operatorname{polylog}(N)/\sqrt{N}\big)$.  
\end{abstract}


\section{Introduction}
\label{s:Intro}

Transformers~\cite{vaswani2017attention} have become the main architectural building block in deep learning-based state-of-the-art foundation models~\cite{bommasani2021opportunities, zhao2023survey, wei2022emergent}.  The undeniable success of these large language models (LLMs) in predicting on \textit{non-i.i.d}.\ sequential data has inspired a deluge of research into the statistical properties of transformers, ranging from minimax (order) optimal guarantees transformers trained in context~\cite{kim2024transformers},
statistical analyses of their training dynamics under linearized/surrogate attention mechanisms~\cite{chen2024training,lu2024asymptotic,zhang2024context,siyu2024training,kim2024transformersdyna,duraisamy2024finite} and their infinite-width limits surrogates~\cite{zhang2024trained}, as well as PAC-Bayesian guarantees~\cite{mitarchuk2024length}. Though each result helps to explain a piece of the LLM puzzle, each has one of the following drawbacks which we address:
\begin{itemize}
    \item[(i)] \textbf{Non-Explicit Constants:} They only provide learning rates with unknown constants,
    \item[(ii)] \textbf{Non-Sequential Data:} Their guarantees do not apply to non-i.i.d.\ (sequential) data,
    \item[(iii)] \textbf{Unrealistic Transformers:} They rely on linearized attention surrogates, a single attention head, omitting key normalization layers, or non-standard activation functions.
\end{itemize}

This paper fills this gap by deriving statistical guarantees for transformers which do not suffer from the limitations in \textit{(i)-(iii)}.  We guarantee (Theorem~\ref{thrm:main}) that transformers trained on a single time-series trajectory can generalize at future moments in time{, with explicit constants}.  We, therefore, consider the learning problem where the user is supplied with $N$ paired samples $(X_1,Y_1),\dots,(X_N,Y_N)$, where each input $Y_n=f^{\star}(X_n)$ for a smooth (unknown) \textit{target function} $f^{\star}:\mathbb{R}^{d\times M}\to \mathbb{R}^D$ is to be learned, depending on a history length $M$, and where the inputs are generated by a time-homogeneous Markov process $\smash{X_{\cdot}\eqdef (X_n)_{n=1}^{\infty}}$, or any process admitting a finite-dimensional Markovian lift in the sense of~\cite{cuchiero2019markovian}.  The assumption $Y_n=f^{\star}(X_n)$ results in only a mild loss of generality since if $X_{\cdot}$ is a discretized solution to a stochastic differential equation then $Y_n \approx \text{signal} \textbf{+} \text{noise}$ due to stochastic calculus considerations (see Appendix~\ref{a:stoch_example}).

The performance of any transformer model  $\mathcal{T}:\mathbb{R}^{d\times M}\to \mathbb{R}^D$ is quantified via a smooth loss function $\ell:\mathbb{R}^D\times \mathbb{R}^D\to \mathbb{R}$.  When $M=1$, the generalization of such a $\mathcal{T}$ is measured by the gap between its \textit{empirical risk} $\mathcal{R}^{(N)}$, computed from the single-path training data, and its \textit{(true) $t$-future risk} $\mathcal{R}_t$ at a (possibly infinite) future time $N\le t\le \infty$ ($t\in \mathbb{N}_+$) 
%
defined by
\begin{align*}
    \mathcal{R}_t(\transformer)
    \eqdef 
    \mathbb{E}\big[
        \ell(\mathcal{T}(X_t),f^{\star}(X_t))
    \big]
\end{align*}
where $\mathcal{R}_{t}$ (resp.\ $\mathcal{R}_{\infty}$) is computed with respect to the distribution of $X_t$ (resp.\ \textit{stationary distribution} of $X_{\cdot}$).
The time-$t$ excess-risk $\mathcal{R}_t$, which is generally unobservable, is estimated by a single-path estimator known as the empirical risk computed using all the noisy samples observed thus far
\begin{align*}
    \mathcal{R}^{(N)}(\transformer)
    \eqdef 
    \frac1{N}\,\sum_{n=1}^N\,
            \ell(\mathcal{T}(X_n),f^{\star}(X_n))
    .
\end{align*}
Our objective is to obtain a statistical learning guarantee bounding the gap between the empirical risk and the $t$-future risk of transformer models trained on a single path.


\paragraph{Contributions.}
Our main \textit{statistical guarantee} (Theorem~\ref{thrm:main}) is a bound on the \textit{future-generalization}, free of limitations (i)-(iii), valid at any given time $t\ge N$, of a transformer trained from $N$ samples collected from an unknown transformation ($f^{\star}$) of any suitable unknown Markov process ($X_{\cdot}$).  
Fix a class of transformers $\transformerclass$ for respective input and output dimensions $d$ and $D$, i.e. determine the number of transformer blocks, the number of attention heads, channel sizes, and specify a constraints on its weights.
Then, the first takeaway of our main result is that with probability at least $1-\delta$
\begin{equation}
\label{eq:FutureG}
\tag{FutureGen}
        \sup_{\mathcal{T} \in \transformerclass}
        \,
            \big|
                \mathcal{R}_t(\mathcal{T})-\mathcal{R}^{(N)}(\mathcal{T})
            \big|
    \in 
    \mathcal{O}
        \biggl(
            \frac{\log(1/\delta) + \log(N)^{1/s}}{
            \sqrt{N}
            }
        \biggr)
\end{equation}  
where $s>0$ can be made arbitrarily large and $\mathcal{O}$ hides a constant which is challenging to \textit{estimate accurately} due to the involved \textit{combinatorics}, and is explicitly computed in (Theorem~\ref{thrm:main_b}).

\begin{figure}[tp!]
\begin{minipage}{.32\textwidth}
\begin{center}
\vspace{-1em}
    \includegraphics[height=5cm]{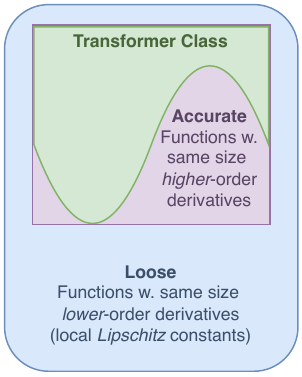}
\end{center}
\end{minipage}
\hfill
\begin{minipage}{.68\textwidth}
    \caption{\textbf{Our Results:} Tight estimates on the constants under $\mathcal{O}$ in~\eqref{eq:FutureG} (Theorem~\ref{thrm:main}) is a consequence of our main result (Theorems~\ref{thm:TransformerBlockSobolevBound}).  The latter accurately estimates the \textit{radius} of 
    the \textit{smallest} $C^s$-ball containing the realistic transformer class $\transformerclass$; this radius is precisely the in the largest $s^{th}$-order partial derivative of any transformer in the class $\transformerclass$.
    \hfill\\
    \textbf{Why it works:} Our higher-order derivative-based analysis accurately approximates the transformer class and provides both \textit{dimension-free} convergence \textit{rates} and accurate constants.  In contrast, low-order local-Lipschitz analysis provides loose generalization bounds by overestimating the size of the class $\transformerclass$ via inclusion in a large class of functions with comparable first derivatives.}
    \label{fig:CoveringWell}
\end{minipage}
\end{figure}

\pagebreak[4]
\medskip
{Our {\color{richblack}{\textit{primary contribution}}} (Theorems~\ref{thrm:main_b}) is a \textit{full analysis} of the higher order sensitivities/\ derivatives of the transformer network; which in-turn yield explicit accuracy constants under the big $\mathcal{O}$ in our main statistical guarantees (Theorem~\ref{thrm:main}).  
With this, our result provides a generalization bound applicable to transformers trained on non-i.i.d.\ data with \emph{explicit constants}.  

\textit{Insights from explicit constants:}  Our result enables us to analyse the effects of the number of attention heads, depth, and width of the transformed model, and weight and bias restriction, as well as on the activation functions used on the generalization of the transformer model in detail.  This is because the explicit constants in our main results are clearly stated in terms of these quantities, for which we provide efficient code to compute.
In particular, we are able to perform an in-depth analysis for the $\operatorname{Swish}$~\cite{ramachandran2017searching}, $\operatorname{GeLU}$~\cite{hendrycks2016gaussian}, and the $\tanh$ activation functions.  Our analysis aligns with the empirical evidence that the activation functions such as $\operatorname{Swish}$ provide superior performance than unconventional choices such as $\tanh$.

\textit{Further consequences:} Although not explicitly addressed here, our results can also be applied to derive statistical guarantees for transformers using alternative frameworks (see \Cref{s:RelatedWork} for an overview). This can, for instance, be achieved for i.i.d.\ data through chaining arguments relying on covering number estimates for balls in smoothness spaces ~\citep[Theorem 2.71]{vaart2023empirical} containing the relevant transformer. The primary gap in the existing literature, which Theorems~\ref{thrm:main_b} fills, is the lack of partial derivative bounds necessary to compute the radii of the balls in these smoothness spaces, which contain the transformer architecture (as illustrated in Figure~\ref{fig:CoveringWell}).

}

\subsection{Related Work}
 \label{s:RelatedWork}
\noindent\paragraph{Statistical Guarantees for Non-I.i.d.\ Data.}
 
 There are several results in the literature addressing learning with non-i.i.d.\ data satisfying a mixing/ergodicity condition dating back, e.g.~\cite{yu1994rates,mohri2008rademacher,mohri2010stability,kuznetsov2017generalization,simchowitz2018learning,foster2020learning,ziemann2022learning}, dating at least back to~\cite{yu1994rates}.  
 However, \textit{none} of these results provide \textit{explicit} generalization bounds.  Rather, they only provide convergence rates without constants for the class $\transformerclass$ since they require covering number estimates, which would ultimately rely on the partial derivative bounds (Theorems~\ref{thm:TransformerBlockSobolevBound} and~\ref{thm:TransformerBlockSobolevBoundLevel}).
 
 This is because they either rely on bounding the Rademacher complexity of the transformer class, e.g.\ in applying~\cite{mohri2008rademacher}, or they rely on computing the cardinality of delta nets~\cite{ziemann2022learning}, both of which necessitate the computation of the worst-case Lipschitz (or $C^s$ norm) of any transformer in the hypothesis class using~\citep[Theorem 2.7.4]{vaart2023empirical} and~\citep[Equation (15.1.8)]{LorentzGolitschekMakovoz_CA_AdvancedProblems_1996}.  Since these higher-order derivatives where \textit{previously unavailable} in the literature, then these results did not work out explicit constants for their generalization bounds.    
We also note the results of~\cite{simchowitz2018learning,foster2020learning}, which obtained statistical guarantees for non i.i.d.\ data under restrictive assumptions on the data-generating process, in addition to the aforementioned restrictions.

We highlight the martingale arguments e.g.~\cite{pmlr-v32-kontorovicha14} and concentration of measure phenomena for martingale sums e.g.~\cite{MartingaleSumsbook2015,boucheron2013concentration} techniques for deriving generalization bounds for models trained on non-i.i.d.\ data.  Again, these results do not yield explicit constants for $\transformerclass$ without estimating the derivatives (or at least local-Lipschitz constants) of the transformers therein, again requiring  Theorems~\ref{thm:TransformerBlockSobolevBound} and~\ref{thm:TransformerBlockSobolevBoundLevel} for a completely explicit result.

\noindent\paragraph{Statistical Guarantees for LLMs Trained on I.i.d.\ Data.}
In addition to the aforementioned statistical guarantees for LLMs; the available statistical guarantees for transformers which the authors are aware of either concern in-context learning for linear transformers \cite{zhang2024trained,garg2022can}, transformers \cite{von2023transformers,akyrek2023what} trained with gradient descent, or instance-dependent bounds \cite{trauger2023sequence} for general transformers.  
These results do not apply in contexts of time series analysis where each training sample is not independent of the others but is rather generated by some recursive stochastic process.

\noindent\paragraph{Ergodic Markov Processes and Mixing Times.}
We require that the data-generating Markov process has an exponentially contracting Markov kernel \cite{Kloeckner_2020CounterCurse_ESAIM}.  For Markov chains, i.e.\ finite-state space Markov processes, this means that the generator ($Q$-matrix) of the Markov chain has a \textit{spectral gap}.  These spectral gaps are actively studied in the Markov chain literature \cite{mufa1996estimation,kontoyiannis2012geometric,atchade2021approximate,MR3383563,kloeckner2019effective} since these have a finite mixing time, meaning that the distribution of such Markov chains approaches their stationary limit after a large finite time has elapsed; i.e.\ they have well-behaved (approximate) mixing times~\cite{montenegro2006mathematical,hsu2015mixing,wolfer2019estimating,BZ_Mixing}.
{We rely on actively-studied optimal transport-theoretic notions of mixing since it is easily verified, or already known, for most data-generating processes than more classical notions; e.g.~\cite{kuznetsov2017generalization,mohri2010stability}.  }

We note that our generalization bounds rely on the concentration of measure arguments for the ``smooth'' optimal transport distances in~\cite{Kloeckner_2020CounterCurse_ESAIM,Riekert_ConcentrationOfMeasure2022}.  The last step of our analysis refines the concentration of measure arguments of~\cite{hou2023instance,benitez2023out,kratsios2024tighter} to the non-i.i.d.\ setting using exponentially ergodic Markov process theory.


\section{Background and Preliminaries}
\label{s:Prelim}

This section overviews the necessary background for a self-contained formulation of our main results.  This includes the definition of transformers and examples of data-generating processes treatable within our framework.

\subsection{Admissible Data-Generating Processes}
\label{s:Prelim__ss:AdmissibleProcesses}

Our statistical guarantees are driven by Markovian arguments; we emphasize that these are not heavily restrictive since we can Markovianize processes in higher dimensional state spaces.  Such state-space extensions, known as Markovian lifts in the literature~\cite{cuchiero2020generalized}, typically require tracking only a few recent movements of the data-generating process rather than solely its current state.

We operate in the following setting:
Fix dimensions $d,D\in \mathbb{N}_+${, a finite memory ${M}\in \mathbb{N}_+$, and let $X_{\cdot}\eqdef (X_n)_{n\in \mathbb{N}_0}$ be a stochastic process taking values in $\mathbb{R}^d$, such that the lifted process $X^M_{\cdot}\eqdef (X^M_{[0\vee (n-M),\dots,0\vee n]})_{n\in \mathbb{N}_0}$ is Markovian on $\mathbb{R}^{Md}$.}
Let $P$ be a Markov kernel on a non-empty Borel $\mathcal{X}^{M}\subseteq \mathbb{R}^{{M}d}$ with initial distribution $X_0\sim \mu_0\in \mathcal{P}(\mathbb{R}^{{M}d})$
given by $X_n^{M}\sim \mu_n\eqdef P^n \mu_0\eqdef \mathbb{P}(X_n^{M}\in \cdot)$ and for each $x\in \mathbb{R}^{{M}d}$ and $n\in \mathbb{N}_+$, set $P^n(x,\cdot)\eqdef \mathbb{P}(X_n^{M}\in \cdot|X_0^{M}=x)$.  
{The process $X_{\cdot}^M$ is a \textit{Markovian lift} of $X_{\cdot}$.
Examples of processes with \textit{finite-dimensional} Markovian lifts are ARIMA times-series models, see e.g.~\cite{cryer1991time}.
}

\begin{assumption}[Bounded Trajectories]
\label{ass:Comp_Support}
There is a $c>0$ such that $\mathbb{P}(\sup_{t\in \mathbb{N}}\,\|X_t\| \le c)=1$.
\end{assumption}
\begin{assumption}[Exponential Ergodicity]
\label{ass:contractivity}
There is a $\kappa\in (0,1)$ such that: for each $\mu,\nu\in \mathcal{P}(\mathbb{R}^{{M}d})$ and every $t\in \mathbb{N}_+$ one has $\mathcal{W}_1(P^t\mu,P^t\nu)\le 
\kappa^t\, \mathcal{W}_1(\mu,\nu)$.  
\end{assumption}
\subsubsection{Examples: Projected SDEs - From Langevin Dynamics to Martingales}
\label{s:Prelim__ss:AdmissibleProcesses___sss:SDEs}
A broad class of non-i.i.d. data-generating processes satisfying our assumptions is a broad generalization of any Markov processes obtained by ``projecting'' the strong solution to a stochastic differential equation (SDE) with overdampened drift onto a compact convex subset of $\mathbb{R}^d$.  The processes which we can project are vast generalizations of the forward process used in \textit{denoising diffusion models}; see e.g.~\cite{song2020score} whose convergence is by now well-understood; see e.g.~\cite{chen2023sampling}.
\begin{example}[Projected SDEs with Overdampened Drift]
\label{ex:proj_SDE}
Consider a latent dimension ${\bar{d}}\in \mathbb{N}_+$, $\mu:\mathbb{R}^{\bar{d}}\to \mathbb{R}^{\bar{d}}$ be Lipschitz and the gradient of a strongly convex function; i.e.\ there is a $K>0$ such that $(\mu(x)-\mu(y))^{\top}(x-y)\le -K \|x-y\|^2$ for all $x,y\in \mathbb{R}^d$.  For any $x\in\mathbb{R}^{\bar{d}}$ let $Z_{\cdot}^x\eqdef (Z_t^x)_{t\ge 0}$ be the unique strong solution (which exists by~\cite[Theorem 8.2]{DaPratoFavBook_2008} since $\mu$ is Lipschitz)
\begin{equation}
\label{eq:overdampened}
    Z_t^x = x + \int_0^t \, \mu(Z_s^x)\,ds + \int_0^t\, W_s
\end{equation}
where $W_{\cdot}$ is a ${\bar{d}}$-dimensional Brownian motion.  
Let $f:\mathbb{R}^{\bar{d}}\to \mathbb{R}^d$ be a bounded $1$-Lipschitz function and consider the discrete-time Markov process $X_{\cdot}\eqdef (X_n)_{n=0}^{\infty}$ on $\mathbb{R}^d$ given by
\[
    X_n^x\eqdef f(Z_n^x)
.
\]
As shown in Proposition~\ref{prop:proj_SDE}, 
$X_{\cdot}$
satisfies Assumptions~\ref{ass:Comp_Support} and~\ref{ass:contractivity}.
The standard example of SDEs~\eqref{eq:overdampened} are
\textit{Langevin dynamics} for a strictly convex potential $U:\mathbb{R}^d\to \mathbb{R}$.  As shown in~\cite{Bolleyeta_JFA_2012}
\[
    \mu(x)=-\nabla U(x)/2
.
\]
\end{example}
\begin{example}[Projections of Diffusive Martingales]
\label{ex:martingales}
Let $d\in \mathbb{N}_+$.
Let $\sigma:\mathbb{R}^d\to P_d^+$ taking values in the cone $P_d^+$ of $d\times d$-dimensional positive definite matrices, be Lipschitz with the Fr\"{o}benius norm on $\mathbb{R}^{d\times d}$, and satisfy the uniform ellipticity condition: there exists a $\lambda>0$ such that for every $x\in \mathbb{R}^d$ holds
$
s_{min}(\sigma(x)\sigma(x)^{\top})\ge \lambda
$,
where $s_{min}(A)$ denotes the minimal singular values of a matrix $A$.  The martingale $Z_{\cdot}$ (see \cite[Proposition 6.15]{DaPratoFavBook_2008} for a proof of martingality) defined for each $t\ge 0$ by
$
Z_t = \int_0^t\, \sigma(Z_s)\,dW_s
$
where $W_{\cdot}\eqdef (W_t)_{t\ge 0}$ is a $d$-dimensional Brownian motion.
Let $f:\mathbb{R}^d\to \mathbb{R}^d$ be any $1$-Lipschitz bounded function.  By Proposition~\ref{prop:projection_martingale}, the data-generating Markov process $X_{\cdot}\eqdef (X_n)_{n=0}^{\infty}$, defined for each $n\in \mathbb{N}_+$ by
$
X_n\eqdef f(Z_n)
$
satisfies both Assumptions~\ref{ass:Comp_Support} and~\ref{ass:contractivity}.
\end{example}
We have presented the simplest cases here; which is readily generalizable.  By Lemma~\ref{lem:BoundedLipProjection} to any Markov process exponentially ergodic $Z_{\cdot}$, not necessarily solving the simple dynamics~\eqref{eq:overdampened}, automatically yields examples of data-generating processes satisfying both Assumptions~\ref{ass:Comp_Support} and~\ref{ass:contractivity}.  We list some examples of such processes here: McKean-Vlasov type with relatively general, i.e. it can have non-constant law-dependent drift and diffusion coefficients~\cite[Corollary 4.4]{Wang_McKeanVlassovReflectedExponential_SPA_2023} (possibly with reflections), several SDEs is driven by a pure-jump L\'{e}vy process~\cite[Theorem 3.1]{LuoRefinedCoupledExponentialLevyPureJump}.  When considering reflected SDEs, where the reflections constrain the process to remain in a bounded convex domain, we do not need $f$ to be bounded, as the processes themselves are already bounded.  Further examples of such can be constructed using compact Riemannian sub-manifolds of $\mathbb{R}^d$ with suitable curvature bounds~\cite{OllivierYann_2009}.
Additional examples include Markov processes whose kernels satisfy a log-Sobolev inequality (see Appendix~\ref{s:Examples__ss:LSI} for details).

\subsection{The Transformer Model}

\begin{figure}[!htp]
     \centering
     \begin{minipage}[t]{.3\textwidth}
         \centering
        \includegraphics[height=.2\textheight]{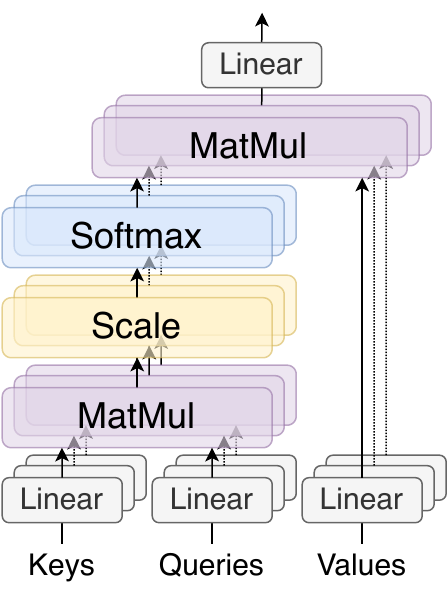}
        \label{fig:MHA}
     \end{minipage}%
     ~
     \begin{minipage}[t]{.3\textwidth}
         \centering
        \includegraphics[height=.2\textheight]{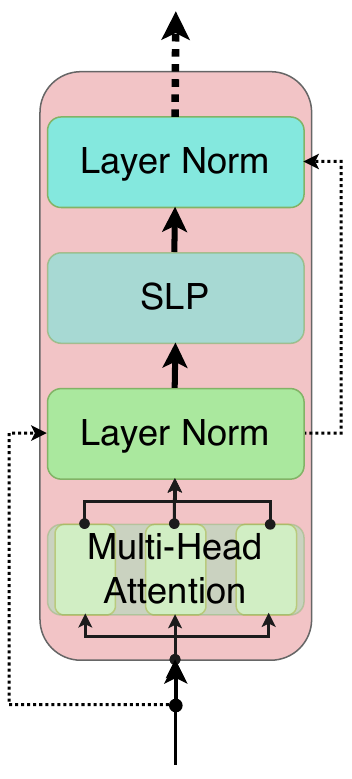}
        \label{fig:TransformerBlock}
     \end{minipage}%
     ~
     \begin{minipage}[t]{.3\textwidth}
         \centering
         \includegraphics[height=.2\textheight]{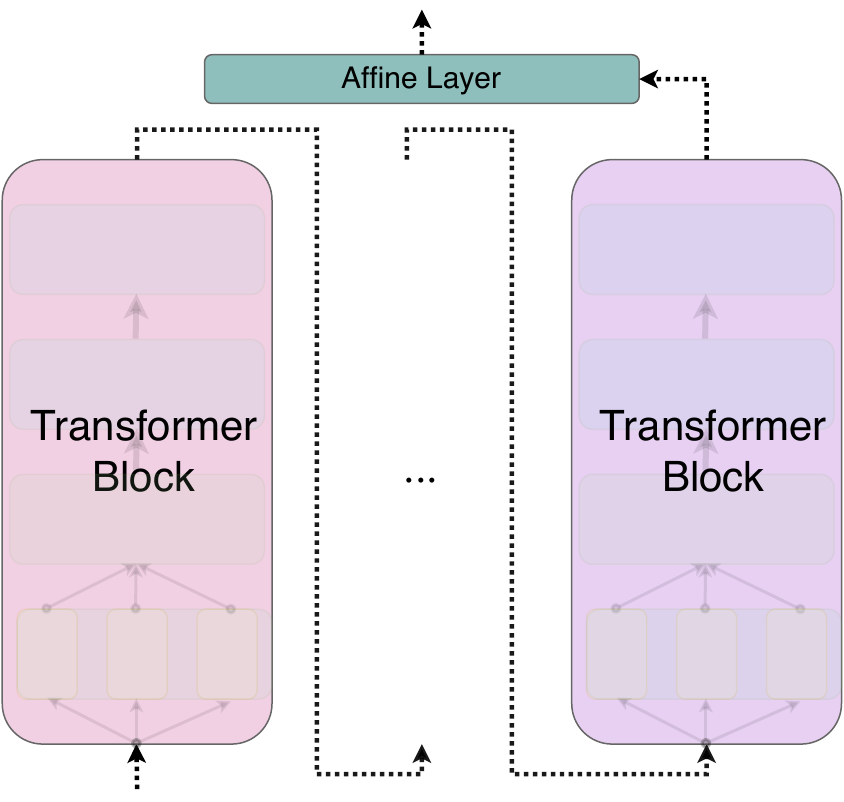}
    \label{fig:Transformer}
    \end{minipage}%
        \vspace{-.5em}
        \caption{\small {Multi-Head Att. (Def.~\ref{defn:MultiHead})}, {Transformer Block (Def.~\ref{defn:TransformerBlock})}, and {Transformer (Def.~\ref{defn:Transformer})}.} 
        \label{fig:TLDR_Transformers}
        
\end{figure}

The overall structure of transformers is summarized in Figure~\ref{fig:TLDR_Transformers}, and we give an in-depth definition of all components with their respective dimensions in \Cref{sec:TransformerDetails}, which is relevant for the details of the bound computation. On a high level, the most important aspects are:
    \paragraph{Multi-Head Attention [$\multihead$].}
    Consists of parallel application of the attention mechanism, described by the following steps. 
    \begin{enumerate}[label=(\roman*),nosep]
        \item Inputs are used three-fold, as keys, queries, and values, all are transformed by distinct linear transformations. 
        \item Keys and queries are multiplied, scaled, and transformed by a softmax application.
        \item This output is combined in a matrix multiplication with the values.
    \end{enumerate}
    \paragraph{Transformer Block [$\tblock$].}
    Here,
    \begin{enumerate}[label=(\roman*),nosep] 
    \item input features are mapped to contexts via \textit{multi-head attention mechanism},
    \item the output of the multi-head attention mechanism and the input features (via a skip connection) are normalized by a \textit{layer-norm [$\layernorm$]} (see \Cref{sec:TransformerDetails}),
    \item the normalized contextual features are transformed non-linearly by a \textit{single-layer perceptron [$\neuraln$]}, and 
    \item its outputs, together with the first set of normalized context (via another skip connection), are normalized by a final \textit{layer-norm} and returned by the transformer block. 
    \end{enumerate}
    \paragraph{Transformer [$\transformer$].}
    Iteratively feed input features through a series of transformer blocks before processing their outputs with a (fully connected affine layer). We denote a class of transformers of a fixed architecture by $\transformerclass$, with each parameter bounded to a predefined domain.

\subsection{Setting}

We consider smooth loss and target {functions} that are concentrated on a compact region, along with their derivatives. The growth rate of the {$C^s$-norm (see \Cref{defn:CsNorm} in \Cref{a:Notation})} of the loss function and its derivatives quantifies the degree of concentration.  
One easily verifies that any function in the Schwartz class satisfies this former of the following conditions, cf. \cite{treves16}.
\begin{definition}[Polynomial Growth of Derivatives]
\label{def:sGrowthRate}
Let $d,D{,M}\in \mathbb{R}$.  A smooth function $g:\mathbb{R}^{{M}d}\to \mathbb{R}^D$ is in   the class $C^{\infty}_{poly:C,r}([0,1]^{{M}d},\mathbb{R}^D)$ if $C,r\ge 0$ {are such that} $\|g\|_{C^s([0,1]^{{M}d})} \le C\, s^{r}$ for each $s\in \mathbb{N}_+$. {Here, $\|\cdot\|_{C^s}$ is the uniform Sobolev norm on the specified domain.}
\end{definition}
In Appendix~\ref{s:Functions_of_ControlledGrowth} we show that, in one dimension, any real analytic function whose power series expansion at $0$, has coefficients growing at an $\mathcal{O}((s+1)^r)$ rate belongs to $g\in C^{\infty}_{poly:C,r}([0,1],\mathbb{R})$.  One can easily extend this argument to multiple dimensions to obtain further examples.

We consider an \textit{realizable} PAC learning problem, determined by a smooth $1$-Lipschitz \textit{target function} $f^{\star}:\mathbb{R}^d\to \mathbb{R}^D$ which we would like to learn using a sequence of random observations $\big((X_t,Y_t)\big)_{t\in \mathbb{N}}$ as our training data.  That is, for each $t\in\mathbb{N}_+$
\begin{align*}
    Y_t \eqdef f^{\star}(X_t) 
\end{align*}
We aim to learn $f$ from a \textit{single path}.  The ability of a model to reliably recover the function $f:\mathbb{R}^d\to\mathbb{R}^D$ at time $t$, given the input $X_t$, is quantified by the \textit{$t$-future risk}
\begin{align*}
    \mathcal{R}_t(f)
    \eqdef 
    \mathbb{E}\big[
        \ell(f(X_t),f^{\star}(X_t))
    \big]
.
\end{align*}
The time-$t$ excess-risk $\mathcal{R}_t$, which is generally unobservable, is estimated by a single-path estimator known as the empirical risk computed using all the noisy samples to observed thus far
\begin{align*}
    \mathcal{R}^{(N)}(f)
    \eqdef 
    \frac1{N}\,\sum_{n=1}^N\,
            \ell(f(X_t),f^{\star}(X_t))
    .
\end{align*}
Our objective is to obtain a statistical learning guarantee on the quality of our estimate of the target function given by the \textit{time $t \,(\in \mathbb{N}_+$) generalization gap}
$
    \big|\mathcal{R}_t(f)-\mathcal{R}^{(N)}(f)\big|
    .
$  

We now summarize our setting and all parameters defining it, e.g.\ dimension, number of attention heads in the transformer, growth rate of the derivatives of the target and loss functions, etc.

\begin{setting}[Standing Assumptions]
\label{setting}
Consider a hypothesis class $\transformerclass$.
Fix $r_f,r_{\ell},C_f,C_{\ell} \ge 0$, as well as a
target function $f^*$ and loss function $\ell$ with 
\begin{align*}
f^{\star}\in C^{\infty}_{poly:C_f,r_f}(\mathbb{R}^{{M}d},\mathbb{R}^D)
\quad \text{ and } \quad
\ell\in C^{\infty}_{poly:C_{\ell},r_{\ell}}(\mathbb{R}^{2D},\mathbb{R});
\end{align*}
and suppose that \Cref{ass:Comp_Support} and either of \Cref{ass:contractivity,ass:MC} hold.
\end{setting}
 
\section{Main Results}
\label{s:Main_Result}

\subsection{Analytic Results: Bounds on the Higher-Order Derivatives of Transformer}
\label{ss:BoundonCs}

Our primary contribution is the computation of $C_{\ell, \transformerclass, K, s} \eqdef \sup_{\transformer \in \transformerclass} \Vert \ell(\transformer, f^*) \Vert_{C^s}$, which encodes the maximal size of the first $s$ partial derivatives of any transformer in the class $\transformerclass$.  
Thus, it describes the complexity of the class $\transformerclass$ (e.g.\ int terms of number of attention heads, depth, width, etc...), the size of the compact set $K$, and the smoothness of the loss function and target functions. 

We note that, any uniform generalization bound for smooth functions thus necessarily contains constants of the same order \textit{hidden within the big} $\mathcal{O}$.  See e.g. the entropy bound in~\citep[Theorem 2.71]{vaart2023empirical} which yields VC-dimension bounds via standard Dudley integral estimates in the i.i.d.\ case.  

Critically, when the function class is defined by function composition, i.e.\ as in deep learning, then these maximal partial derivatives tend to grow factorially in $s$.  This is a feature of the derivatives of composite functions in high dimensions as characterized by the multi-variate chain rule (i.e.\ the Fa\'{a} di Bruno formula~\cite{faa1855sullo,ConstantineSavits96}).  The combinatorics of these partial derivatives is encoded by the coefficients in the well-studied bell-polynomials~\cite{bell1934exponential,mihoubi2008bell,wang2009general} whose growth rate has been recently understood in~\cite{KO2022} and contains factors of the order of $\mathcal{O}(
\big( \frac{ 2 s}{e\ln s} ( 1 + o(1) ) \big)^s
)$.  

Remark that, in the feedforward case, i.e.\ when no layernorms or multihead attention are used, then the $s=1$ case is bounded above by the well-studied path-norms; see e.g.~\cite{bartlett2017spectrally,neyshabur2015norm}, which are simply the product of the weight matrices of in the network and serve as a simple upper-bound for the largest Lipschitz constant (i.e.\ $C^1$ norm) of the class $\transformerclass$.  These constants are included as very specific cases of our constant bounds. This is why we present two versions: a weaker but simpler bound, as well as a more accurate but detailed bound.

\begin{theorem}[$\transformerclass$-bound in terms of $\mathcal{O}$]\label{thrm:main_b}
In the case of a single transformer block $C_{\ell,\transformerclass,K,s}$ is of the order of
\begin{align*}
\resizebox{0.95\linewidth}{!}{$
    \mathcal{O}\Big(
        \underbrace{
            C^{\ell,f^{\star}}
            \vphantom{\biggl(\biggr)^{s^4}}
        }_{\text{Loss \& Target}}
        \,
        \underbrace{
            C^{\layernorm}_{K\tagA}({\cleq s})^s
            C^{\layernorm}_{K\tagC}({\cleq s})^{s^3}
            \vphantom{\biggl(\biggr)^{s^4}}
        }_{\text{Layernorms}}
        \underbrace{
            C^{\neuraln}_{K\tagB}({\cleq s})^{s^2}
            \vphantom{\biggl(\biggr)^{s^4}}
        }_{\text{Perceptron}}
        \,
            \underbrace{
                \biggl(
                    1 + C^{\multihead}_K(\cleq s)
                \biggr)^{s^4}
                \vphantom{\biggl(\biggr)^{s^4}}
             }_{\text{Multihead Attention}}
        \underbrace{
            D^{s^2}\,d^{2s^3}
            \vphantom{\biggl(\biggr)^{s^4}}
        }_{\text{dimensions}}
        \,
        \underbrace{
             c_s
             ^{s^s + s^3 + s^4}
             \vphantom{\biggl(\biggr)^{s^4}}
        }_{\text{Generic: $s$-th order Derivative}}
    \Big)
$}
\end{align*}
where the ``generic higher-order derivative constant'' is $c_s \eqdef\frac{ 2 s}{e\ln s} ( 1 + o(1) )$.
Further, 
\begin{gather*}
        C^{\ell,f^{\star}} 
        = \mathcal{O}\big(
            C_f^s
            \,
            s^{r_{\ell} + 2s^2} 
        \big),
    \qquad
        C^{\neuraln}_K (\cleq s)
        = \mathcal{O}\big(
                c^{\neuraln}
            +
                \ldim
                \Vert \sigma \Vert_{s}
                \tilde{c}_s^s (c^{\neuraln})^{s+1}
        \big),
    \\
        C^{\layernorm}_K (\cleq s)
        = \mathcal{O}\big(
            s^{(1+s)/2}
            c_s^s
        \big),
    \qquad
        C^{\multihead}_K (\cleq s)
        = \mathcal{O}\big(
             e^{-2s}
            {M^2} (2\idim\kdim\cdot c_s)^s
            (s \cdot c^{\multihead})^{2s + 2}
        \big).
\end{gather*}
Here 
$\tilde{c}_s\eqdef s^{1/2}({n}/{e})^sc_s^s$; 
$\idim$ is the input-dimension and $\kdim$ is the key-dimension of the multi-head attention $\multihead$ (see \Cref{defn:MultiHead} for details);
$\ldim$ is the width of the neural network $\neuraln$ (see \Cref{defn:TransformerBlock} for details); 
$c^{\neuraln}$ as well as $c^{\multihead}$ are parameter bounds on $\neuraln$ as well as $\multihead$, respectively (see \Cref{thm:TransformerBlockSobolevBoundLevel} for details); and 
$\Vert \sigma \Vert_{s}$ is the $C^s$-bound of the activation function used.
If no layer norms, SLP, or multi head attention mechanisms are included in the class, then their respective terms in our order estimate should be taken to be $1$.
\end{theorem}
\begin{proof}
    The result is a direct consequence of \Cref{thm:TransformerBlockSobolevBoundLevel,thrm:main__technicalversion}.
    The order of the bounds $C^{\layernorm}_K, C^{\neuraln}_K$, and $C^{\multihead}_K$ are given by
    \Cref{cor:MultiHead,cor:LayerNorm,cor:FeedForward}.
\end{proof}
%

\paragraph{Explicit bound computation.}
We further refined this result by deriving formulae that enable the precise calculation of these bounds.
In order to enhance the accuracy of these estimates, we distinguished not only between different levels of derivatives but also between various types of derivatives. An exemplary improvement of the bound by this distinction can be seen on the right-hand side of \Cref{fig:boundillustration}. 

By \textit{derivative type}, we are distinguishing the different multi-variate derivative operators, but regard those as equivalent that are identical after sorting. In a 2-dim setting, for instance, we would say the $(1, 2)$-derivative (once in the first component, twice in the second component) has the same type as the $(2, 1)$-derivative, but not the $(3, 0)$-derivative. The reason why we choose to consider an equivalence is that the size of the set $\calP$ in Fa\`a di Bruno's Formula (\Cref{thm:FaaDiBruno}) grows too big in high dimensions for numeric evaluation.

\begin{figure*}[th]
    \centering
    \begin{minipage}{.32\textwidth}
        \centering
        \includegraphics[width=\textwidth]{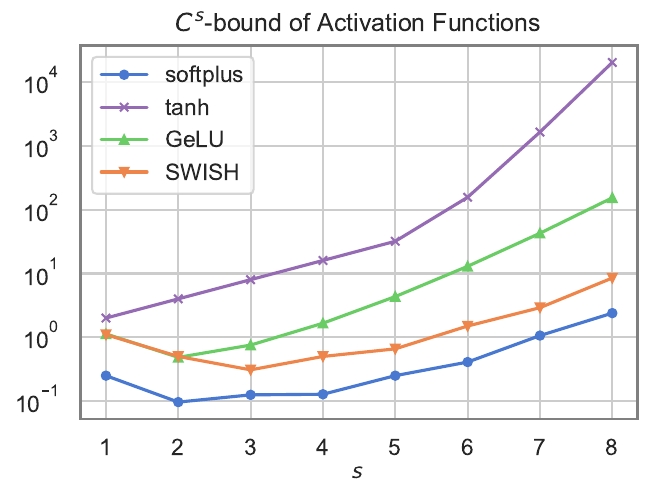}
    \end{minipage}
    \begin{minipage}{.32\textwidth}
        \centering
        \includegraphics[width=\textwidth]{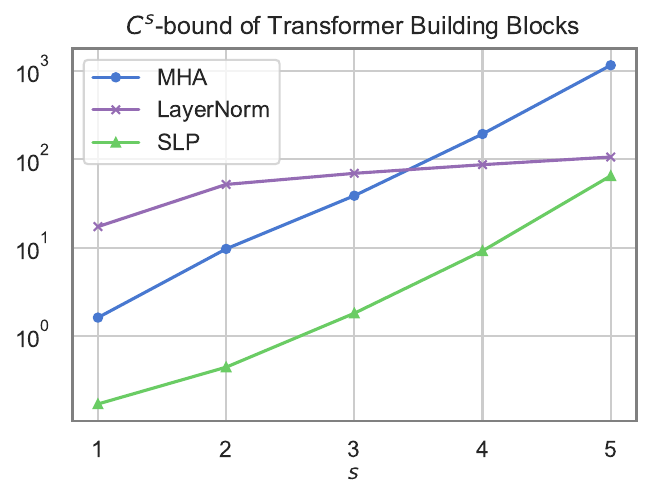}
    \end{minipage}
    \begin{minipage}{.32\textwidth}
        \centering
        \includegraphics[width=\textwidth]{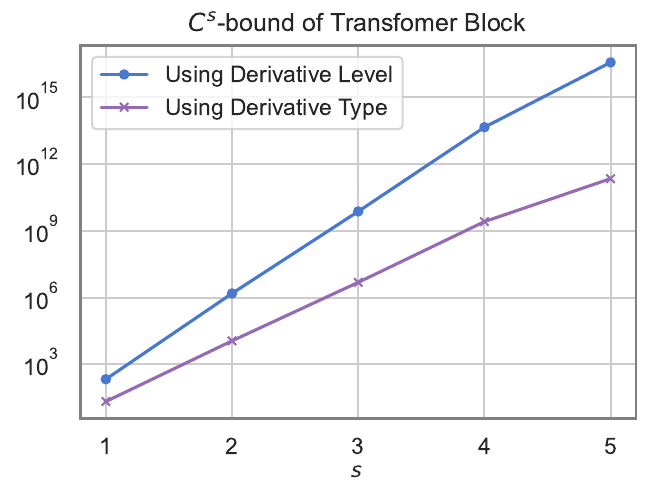}
    \end{minipage}
    \caption{{Effects of Transformer Components of~\ref{eq:FutureG}}: (left to right.) The first figure shows the $C^s$ bound of various activation functions according to results in Appendix \ref{sec:Activation}. The second illustrates $C^s$ bounds for Multi-Head Attention (\Cref{defn:MultiHead}), single-layer perceptrons, and the layer norm. 
    The third shows the $C^s$-bound of a transformer block (\Cref{defn:TransformerBlock}), distinguishing if the bound was computed level-specific (\Cref{cor:DerivativeBoundByLevel}) or type specific (\Cref{thm:DerivativeBoundByType}).
    The parameters used for the above plots are the base cases of \Cref{tab:Activations,tab:NeuralNet,tab:LayerNorm,tab:TBlock,tab:MultiHead}.
    }
    \label{fig:boundillustration}
\end{figure*}

Since these results are fairly technical and verbose, we relegate them to \Cref{sec:RegularityTransfomerBuildingBlocks}, see \Cref{thm:TransformerBlockSobolevBound} for the analogue result to \Cref{thm:TransformerBlockSobolevBoundLevel} and \Cref{lem:Attention,lem:LayerNorm,lem:FeedForward} for tighter bounds on $C^{\layernorm}_K, C^{\neuraln}_K$, and $C^{\multihead}_K$.
Additionally, we provide software tools to efficiently compute the bounds of a given transformer architecture.\footnote{The source code to compute derivative bounds is available at \url{https://anonymous.4open.science/r/transfomer-bounds-B476}.}

\subsection{Statistical Results: Bounds on Future Generalization for Non-I.i.d.\ Data}

Having characterized an upper bound for $C_{\ell,\transformerclass,K,s}$, we may now state our statistical guarantee, which is a version of~\eqref{eq:FutureG}.  This version provides insights on the future-generalization of transformers via 1) explicit constants and 2) explicit \textit{phase transition times} above which the convergence rate in~\eqref{eq:FutureG} accelerates by a polylogarithmic factor.
We express these times of convergence rate acceleration using the following \textit{convergence rate function}
\begin{equation}
\tag{rate}
\label{eq:rate_function__definition}
        \operatorname{rate}_s(N)
    \eqdef 
        \begin{cases}
           \frac{
               \log(
                    c\, N
                )^{{M}d-2s+s/({M}d)}
           }{
               c_2\,  N^{s/({M}d)}
           } 
            & \mbox{ if } {M}d>2s \quad {\color{gray}{\text{ (\textit{initial phases})}}}\\
           \frac{
               \log(
                    c\, N
                )
           }{
                c\, N^{1/2}
           } 
            & \mbox{ if } {M}d=2s \quad  {\color{gray}{\text{ (\textit{critical phase})}}}\\
            \frac{
               \log(
                    c\, N
                )^{{M}d/( 2s + 1 )}
           }{
                c\, N^{1/2}
           } 
            & \mbox{ if } 
            {M}d<2s  \quad {\color{gray}{\text{ (\textit{eventual phases})}}}\\       
        \end{cases}
\end{equation}
where $c\eqdef 1 - \kappa$, $c_2\eqdef c^{s/({M}d)}$, and $0< \kappa< 1$ are constants depending only on $X_{\cdot}$.

\medskip
Our main statistical guarantee provides an infinite number of generalization bounds, which all converge, each with different rates and constants.  The bounds with the fastest convergence have the largest constants and visa-versa.  This allows us to \textit{sort} this family of simultaneous bounds as a function of the sample size $N$; with the small constant bounds providing tighter generalization guarantees for small sample sizes $N$ while the rapidly converging bounds provide tighter guarantees for large sample sizes $N$. 

The small-sample bounds depend only on the first few derivatives of the transformers in our class, while the large-sample bounds rely on higher-order derivatives. The point at which one bound in our class surpasses another, becoming the tightest bound for a given sample size $N$, is characterized by a transition time, denoted $\tau_s$, as illustrated in Figure~\ref{fig:TransitionTimes}.  
In this way, every single derivative order computed in Theorem~\ref{thrm:main_b} is simultaneously utilized to obtain precise statistical guarantees optimized to the sample size $N$.

\begin{theorem}[Pathwise Generalization Bounds for Transformers]
\label{thrm:main}
In Setting~\ref{setting}, there exists $\kappa\in(0,1)$, depending only on $X_{\cdot}$, and $t_0\in \mathbb{N}_0$; such that for each $t_0\le N \le t \le \infty$ and $\delta \in (0,1]$ the following holds with probability at-least $1-\delta$
\begin{align*}
\resizebox{1\linewidth}{!}{$
        \underset{\mathcal{T}\in 
            \transformerclass
        }{\sup}
        \,
            \big|
                \mathcal{R}_{\max\{t,N\}}(\mathcal{T})-\mathcal{R}^{(N)}(\mathcal{T})
            \big|
    \lesssim
        \overset{\infty}{\underset{{s=0}}{\sum}}
        \,
            I_{N\in [\tau_{s},\tau_{s+1})}
            \,
            C_{\ell,\transformerclass,K,s}
        \biggl(
                \underbrace{
                    \kappa^t
                }_{\term{time}}
                \,
                {
                +\,
                \underbrace{
                    \operatorname{rate}_s(N)
                }_{\term{complexity}}
            +
                \underbrace{
                    \frac{
                        \sqrt{2\ln(1/\delta)}
                    }{
                        N^{1/2}
                    }
                }_{\term{hprob}}
                }
    \biggr)
$}
\end{align*}
with $I_\cdot$ as indicator function, $\operatorname{rate}_s(N)$ as in \eqref{eq:rate_function__definition}, the constant $
C_{\ell,\transformerclass,K,s} \eqdef \sup_{\transformer \in \transformerclass} \Vert \ell(\transformer, f^*) \Vert_{C^s},
$
and the \textit{transition times} $(\tau_s)_{s=0}^{\infty}$ are given iteratively by $\tau_0\eqdef 0$ and for each $s\in \mathbb{N}_+$
\begin{equation*}
\resizebox{1\hsize}{!}{$
    \tau_s
\eqdef 
    \inf\,
    \biggl\{ 
        t\ge \tau_{s-1}
    :\,
            C_{\ell,\transformerclass,K,s}
            ( \kappa^t + \operatorname{rate}_s(N)+
            \frac{\sqrt{\log(1/\delta)}}{\sqrt{N}}
            ) 
        \le 
            C_{\ell,\transformerclass,K,s-1}
            ( \kappa^t + \operatorname{rate}_{s-1}(N)+
            \frac{\sqrt{\log(1/\delta)}}{\sqrt{N}}
            )
    \biggr\}
.
$}
\end{equation*}
Furthermore, $c\eqdef 1-\kappa$, $c_2\eqdef c^{s/({M}d)}$, $\kappa^{\infty}\eqdef \lim_{t\to \infty}\, \kappa^t= 0$, and $\lesssim$ hides an absolute constant.
\end{theorem}


\begin{wrapfigure}[14]{r}{0.35\textwidth}
\vspace{-3em}
  \begin{center}
    \includegraphics[width=.9\linewidth]{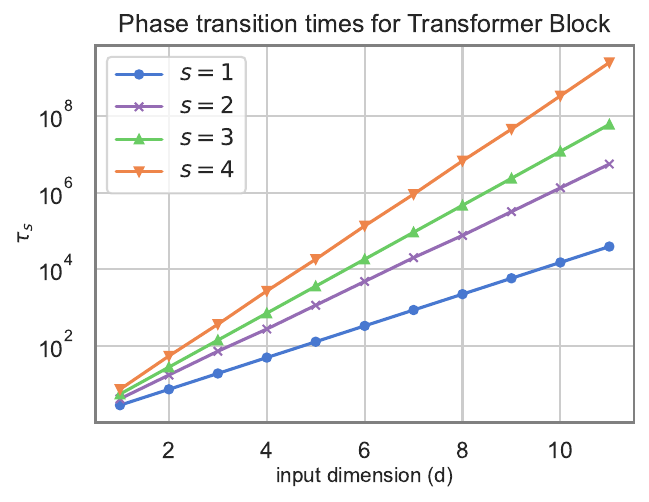}
  \end{center}
  \caption{\textbf{Transition times:} (y-axis) when the future-generalization bound accelerates by a polylogarithmic factor (in $N$) for a single transformer block in terms of the input dimension $d$ (x axis).  See~\Cref{ss:BoundonCs} for details on constants.}
  \label{fig:TransitionTimes}
\end{wrapfigure}
Theorem~\ref{thrm:main} implies the order estimate in~\eqref{eq:FutureG}.  This is because $C_{\ell,\transformerclass,K,s}$ is constant in $N$ and $\operatorname{rate}_s(N)<\operatorname{rate}_{s-1}(N)$; thus, for every $s>0$ the right-hand side our bound is eventually bounded by any $C_{\ell,\transformerclass,K,s}
        (
                    \kappa^t
                \,
            +
                        \sqrt{2\ln(1/\delta)}/N^{1/2}
                +\,
                    \operatorname{rate}_s(N)
    )$ for $N$ large enough. 
However, unlike the order estimate~\eqref{eq:FutureG}, Theorem~\ref{thrm:main} provides an explicit description of the actual size of the future-generalization gap in terms of three factors which we now interpret.

\paragraph{\Cref{time}: Non-Stationarity.} This term quantifies the rate at which the data-generating Markov process $X_{\cdot}$ becomes stationary.  This term only depends on the time $t$ and a constant $0<\kappa<1$ determined only by $X_{\cdot}$.  The limiting case is described using the notational convention $\kappa^{\infty}\eqdef \lim\limits_{t\to \infty}\,\kappa^t = 0$.  

\paragraph{\Cref{complexity}: Model Complexity Term (Phase Transitions).} This term captures the complexity of the transformer network in terms of the number of self-attention heads, depth, width, and the activation function used to define the class $\transformerclass$. 
Each constant $C_1\le \dots \le C_s\le \dots$ collects the higher-order sensitivities ($s^{th}$ order partial derivatives; where $s\in \mathbb{N}_+$) of the transformer model.  Each $0=\tau_0\le \tau_1\le\dots \le \tau_s\le $ indicates the times at which there is a phase-transition in the convergence rate of the generalization bound accelerates.  Once $t\ge \tau_s$, then the convergence rate of~\Cref{complexity} accelerates, roughly speaking, by a reciprocal log-factor of $1/\log(N)$.  
Observe that the rate function is asymptotically equal to the rate function from the central limit theorem, as $s$ tends to infinity; that is, $\lim_{s\to \infty}\, \operatorname{rate}_s(N) = 1/(c\sqrt{N})$.  
The rate~\eqref{eq:rate_function__definition} is the (optimal) rate at which the empirical measure generated by observations from a Markov process converges to its stationary distribution in $1$-Wasserstein distance~\cite{Kloeckner_2020CounterCurse_ESAIM,Riekert_ConcentrationOfMeasure2022}.  The polylogarithmic factor is removable if the data is i.i.d.~\cite{graf2007foundations,dereich2013constructive}.

\paragraph{\Cref{hprob}: Probabilistic Validity Term.} This term captures the cost of the bound being valid with probability at least $1-\delta$.  
The convergence rate of this term cannot be improved due to the central limit theorem. It is responsible for the overall convergence rate of our generalization bound being ``stuck'' at the optimal rate of $\mathcal{O}(1/\sqrt{N})$ from the central limit theorem; as the other two terms converge exponentially to $0$.

\section{Discussion: Implications of architecture choices.}

\begin{wrapfigure}[18]{r}{0.35\textwidth}
    \centering
    \includegraphics[width=.35\textwidth]{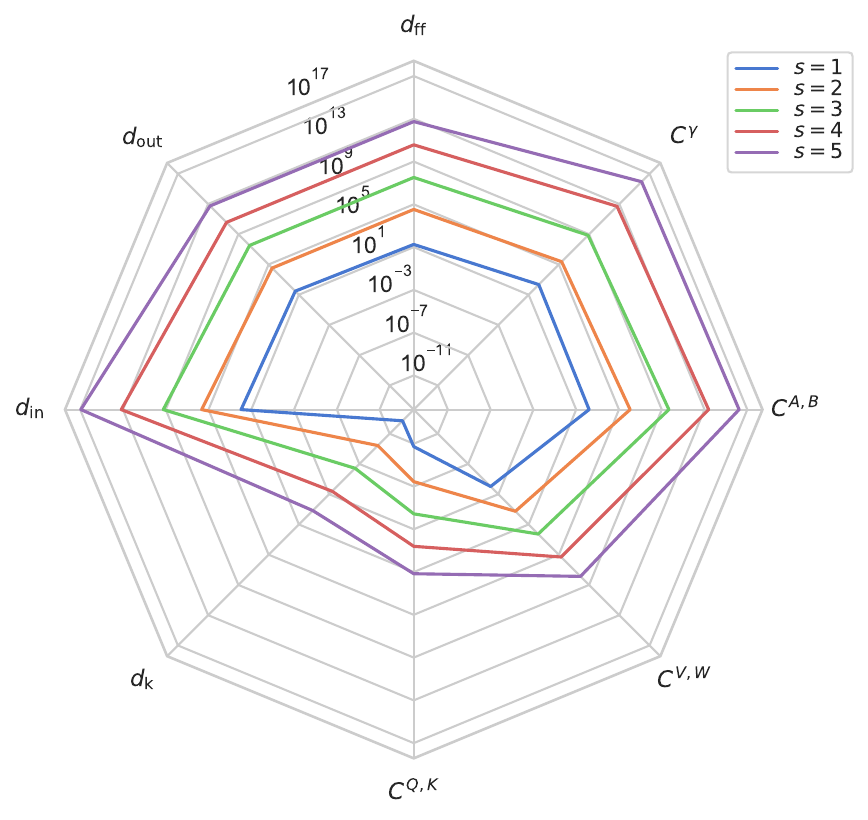}
    \caption{Absolute changes in $C^s$-bound for changes in architecture. Changes in dimensions ($d_\cdot$) are $\times 2$, while changes in parameter-bounds ($C^\cdot$) are $\times 10$, from the base parameters (see \Cref{tab:Activations,tab:NeuralNet,tab:LayerNorm,tab:TBlock,tab:MultiHead}).}
    \label{fig:Architecture}
\end{wrapfigure}

Figure~\ref{fig:boundillustration} illustrates the effect of various building blocks in the construction of a transformer (e.g.\ activation choice, multihead attention (MHA), layernorms) through their effect on the constants in our generalization bounds. 
While \Cref{tab:Activations,tab:NeuralNet,tab:LayerNorm,tab:TBlock,tab:MultiHead} contain more details, highlight here some key implications that architecture choices have on the bound:

    \paragraph{I) Choice of Activation Function:} We found (see \Cref{lem:Softplus,lem:tanh,lem:GELU,lem:SWISH}) that the $C^s$-bounds of activation function may vary substantially, framing \texttt{softplus} and \texttt{swish} as the more regular, and \texttt{tanh} resulting in the highest bound. This is also illustrated in \Cref{fig:boundillustration}. Note that the activation bound impacts the $\neuraln$-bound linearly and therefore effects the transformer-block bound of order $s^2$.
    
    \paragraph{II) Effects of Three Different Block-Types:}
    Considering the three components -- $\multihead, \layernorm, \neuraln$ -- that make up a transformer block, we observe that for low $s$ the regularization by $\layernorm$ has the highest bound, but becomes less relevant with the exponential increase of the $\multihead, \neuraln$-bounds for larger $s$. Exemplary values are shown in \Cref{fig:boundillustration}.

    \paragraph{III) Weight Size for MLP vs. Multi-Head Attention:} 
    As evident in \Cref{fig:Architecture} (and \Cref{tab:TBlock}), the parameter-bounds on $\neuraln$ (denoted by $C^{A,B}$) seem to have a more substantial impact on the bound than the parameter bounds of $\multihead$. For the latter, bounds on key- and query-matrices ($C^{K, Q}$) seem to have bigger impacts for lower $s$ than value- and aggregation-matrices ($C^{V, W}$) (see \Cref{defn:MultiHead} for details on notation), however show larger growth rates for larger $s$, as also shown in \Cref{tab:MultiHead}.
    The simple local-Lipschitz case of this result was obtained in~\cite{pmlr-v139-kim21i} (however, the complicated higher-order version of this result is completely new).
    
    \paragraph{IV) Effect of Dimensions (Key, Input, etc$\dots$):} Eventually, we can examine how various dimensions effect the bound. The input dimension ($\idim$) has a slightly higher impact than the output dimension ($\odim$). When it comes to choosing latent dimensions, scaling the hidden dimension of the $\neuraln$ ($\ldim$), has an effect similar to changes in the output dimension, and substantially higher comparing to the key-dimension $\kdim$ (see \Cref{defn:MultiHead,defn:TransformerBlock} in \Cref{sec:TransformerDetails} for details and notation). 

Consequentially, we show the effect on the phase-transition times $(\tau_t)_{t=0}^{\infty}$, defined in Theorem~\ref{thrm:main}, dictating when the bound accelerates by a polylogarithmic factor in $N$.

\acks{
AK acknowledges financial support from an NSERC Discovery Grant No.\ RGPIN-2023-04482 and their McMaster Startup Funds. RS is supported by Canada NSERC CGS-D Doctoral Grant. 
RS and AK acknowledge that resources used in preparing this research were provided, in part, by the Province of Ontario, the Government of Canada through CIFAR, and companies sponsoring the Vector Institute \href{https://vectorinstitute.ai/partnerships/current-partners/}{https://vectorinstitute.ai/partnerships/current-partners/}.
}




\bibliography{refs/main,
                refs/qfin-datasets, 
                refs/qfin-pm,
                refs/kalman-filters,
                refs/rebutal,
                refs/InContext}

\newpage
\appendix

\DoToC

\section{Notation}
\label{a:Notation}




In this section, we present the notation that will be employed throughout the appendix. This notation builds upon the framework established in the main body of the text, while incorporating additional levels of specificity. Given the technical nature of certain results discussed herein, a more detailed and precise formulation of the notation is necessary to ensure clarity and rigor in the statements that follow.

\begin{notation}[Multi-index Notation]\label{not:Multivariate}
We will fix the following multivariate notation.
\begin{itemize}[nosep]
    \item Multi-indices $\alpha \eqdef (\alpha_1, \ldots, \alpha_k) \in \N^k, k \in \N$ are denoted by Greek letters.
    \item The sum of entries is given by $\vert \alpha \vert \eqdef \sum_{i=1}^k \alpha_k$.
    \item Its faculty is defined by $\alpha!  \eqdef \prod_{i=1}^k \alpha_k!$,
    \item We denote the derivative w.r.t. $\alpha$ by $D^\alpha \eqdef \partial^{\vert \alpha \vert} / \partial x_1^{\alpha_1} \cdots \partial x_k^{\alpha_k}$ if $\vert \alpha \vert > 0$ else $D^\alpha$ is the identity operator.
    \item For a vector $x \in \R^k$, we write $x^\alpha \eqdef \prod_{i=1}^k x_k^{\alpha_k}$.
    \item We define the relation $\alpha \prec \beta$ for $\beta \in \N^k$ if one of the three following holds 
    \begin{enumerate}[label=(\roman*), nosep]
        \item $\vert \alpha \vert < \vert \beta \vert$;
        \item $\vert \alpha \vert = \vert \beta \vert$, and $\alpha_1 < \beta_1$; or 
        \item $\vert \alpha \vert = \vert \beta \vert$, and $\alpha_i = \beta_i$ for $i \in \myset{1}{j-1}$ and $\alpha_j < \beta_j$ for $j \in \myset{2}{k}$.
    \end{enumerate}
    \item Unit vectors $e_i \in \{0,1\}^k$ are defined by $(e_i)_j = 0$ for $i \neq j$ and  $(e_i)_i = 1$.
\end{itemize}
\end{notation}
\begin{definition}[$C^s$-norm] \label{defn:CsNorm}
For any $s>0$, the norm $\|\cdot\|_{C^s}$ of a smooth function $f:\R^d\to \R$ is defined by
\begin{align*}
    \|f\|_{C^s}
\eqdef 
        \max_{k=1,\dots,s-1}\,
        \max_{\alpha\in \{1,\dots,d\}^d}
        \,
            \Big\|
                \frac{\partial^k f}{
                    \partial x_{\alpha_1}\dots\partial x_{\alpha_k}
                }
            \Big\|_{\infty}
    +
        \max_{\alpha \in \{1,\dots,d\}^{s-1}}\,
            \operatorname{Lip}\big(
                \frac{\partial^{s-1}f}{\partial x_{\alpha_1}\dots\partial x_{\alpha_{s-1}}}
            \big).
\end{align*}
\end{definition}
We use the following notation to streamline the analytic challenges the tacking of $C^s$-norms.
\begin{notation}[Order operator for multi-indeces]\label{not:Order}
Define the order operator $\order$ for multi-indeces by
\begin{align*}
    \order: \N^k \longrightarrow \N^k, \quad \alpha_1, \ldots, \alpha_k \longmapsto \alpha_{\tau_\alpha(1)}, \ldots, \alpha_{\tau_\alpha(k)},
\end{align*}
where $\tau_\alpha: \myset{1}{k} \to \myset{1}{k}$ s.t. $\alpha_{\tau_\alpha(1)} \geq  \ldots \geq  \alpha_{\tau_\alpha(k)}$.  We write $\alpha \eqorder \beta$ if $\order(\alpha) = \order(\beta)$ for $\alpha, \beta \in \N^k$. Further, denote by $\Order^k_{n}$ the set $\{\order(\alpha) :  \alpha \in \N^k, \vert \alpha \vert = n \}$ and write $\Order^k_{\leq n} \eqdef \{\order(\alpha) :  \alpha \in \N^k, \vert \alpha \vert \leq n \}$. Eventually, define $N(\alpha) \eqdef \# \vert \{ \alpha' \in \N^k : \order(\alpha') = \alpha \} \vert$. 
\end{notation}

We will use the following notation to tabulate the sizes of a $C^s$-norm.
\begin{notation}[Derivatives]\label{not:DerivativeBounds} Let $k \in \N$, $K \in \R^k$ be a set, $f: K \to \R^m$ a function and $\alpha\in\Order^k_n$ an ordered multi-index. Then,
\begin{itemize}[nosep]
    \item the uniform bound of $\alpha$-like derivatives on $K$ is given by 
    \begin{align*}
        C^f_K(\alpha) \eqdef \max_{i \in \myset{1}{m}} \max_{\gamma \sim \alpha} \Vert D^\gamma f_i \Vert_K,
    \end{align*}
    \item we define the bound at / up to derivative level $n$ by
    \begin{align*}
        C^f_K(n) \eqdef \max_{\alpha \in \Order^k_{n}} C^f_K(\alpha), \quad\quad 
        C^f_K(\cleq n) \eqdef \max_{\alpha \in \Order^k_{\cleq n}} C^f_K(\alpha),
    \end{align*}
    \item we write $\Vert K \Vert \eqdef \sup_{x \in K} \Vert x \Vert$, and 
    \item the $\ell^{\infty}$-matrix norm of any $n\times m$ matrix $A \in \R^{n\times m}$ is abbreviated as $$C^A \eqdef \max_{i \in \myset{1}{n}, j \in \myset{1}{m}} \vert A_{i,j}\vert.$$
\end{itemize}
\end{notation}
When segmenting, truncating, or manipulating time series we will using the following notation.
\begin{notation}[Time Series Notation]
The following notation is when indexing paths of any time series.
\begin{itemize}[nosep]
    \item \emph{Realized Path up to time $t$} is denoted by $x_{\le t}\eqdef (x_s)_{s\in \mathbb{Z},\,s\le t}$.
    \item \emph{Segment of a Path} Given a sequence $x\in \mathbb{R}^{\mathbb{Z}}$ and integers $s\le t$, we denote $x_{[s:t]}\eqdef (x_i)_{i=s}^t$.
\end{itemize}
\end{notation}
Lastly, we recorded some additional notations that were required throughout our manuscript.
\begin{notation}[Miscellaneous] We define:
\begin{itemize}[nosep]
    \item \emph{$N$-Simplex}. For $N \in \N$ we write $$\Delta_N\eqdef \{u\in [0,1]^N:\,\sum_{i=1}^N\,u_i=1\}.$$
    \item \emph{Infinite powers}: For $c\in (0,1)$, we define $$c^{\infty}\eqdef \lim\limits_{t\to \infty}\, c^t = 0.$$
    \item \emph{Reshape operator}: For any $F_1,F_2\in \mathbb{N}_+$, the operator is given by $\reshape_{F_1\times F_2}$, mapping any vector $u\in \mathbb{R}^{F_1F_2}$ to the $F_1\times F_2$ matrix
    \begin{align*}
        \reshape_{F_1\times F_2}(x)_{i,j}\eqdef x_{(i-1)F_2 +j}
        .
    \end{align*}
    We denote the inverse of the map $\reshape_{F_1\times F_2}$ by $\vectorize_{F_1,F_2}:\mathbb{R}^{F_1\times F_2}\to\mathbb{R}^{F_1F_2}$.  
    \item \emph{Softmax operator}: For each $F\in \mathbb{N}_+$ and each $x\in \R^F$, 
    $$
            \operatorname{softmax}(x)\eqdef \smax(x) \eqdef (
                    {\exp(x_i)}
                /
                    {\textstyle\sum_{j = 0}^{F-1} \exp(x_j)}
            )_{i=0}^{F-1}.
    $$
\end{itemize}
\end{notation}

\section{{Examples of Data-Generating Processes Satisfying Assumptions~\ref{ass:Comp_Support} and~\ref{ass:contractivity}}}
\label{s:Examples}
This section provides several examples of stochastic (data-generating) processes which satisfy our assumptions and are outside the i.i.d.\ restrictions.  

\subsection{Projected Exponentially Ergodic Latent Processes}
\label{s:ProjectionConstruction}
\begin{proposition}[Lipschitz-Transformed SDEs with Overdampened Drift]
\label{prop:proj_SDE}
In the setting of Example~\ref{ex:proj_SDE}, $\{(P^n(x,\cdot))_{n=0}^{\infty}\}_{x\in [0,1]^d}$ satisfies both Assumptions~\ref{ass:Comp_Support} and~\ref{ass:contractivity}.
\end{proposition}
The proof of Proposition~\ref{prop:proj_SDE} uses the following lemma.
\begin{lemma}[Enforcing Boundedness via $1$-Lipschitz Maps Preserves Exponential Ergodicity]
\label{lem:BoundedLipProjection}
Let $\tilde{d},d\in \mathbb{N}_+$ and $Z_{\cdot}$ be a Markov process on $\mathbb{R}^{\tilde{d}}$ satisfying Assumption~\ref{ass:contractivity}.  Given any bounded Lipschitz function $f:\mathbb{R}^{\tilde{d}}\to \mathbb{R}^d$ the Markov process $X_{\cdot}\eqdef (X_n)_{n=0}^{\infty}$ in $\mathbb{R}^d$, defined for each $n$ by $X_n\eqdef f(Z_n)$, satisfies both Assumption~\ref{ass:Comp_Support} and~\ref{ass:contractivity}.
\end{lemma}
\begin{proof}[{Proof of Lemma~\ref{lem:BoundedLipProjection}}]
Since $f$ is bounded, then there exists some $r>0$ such that $f(\mathbb{R}^d)\subset B_r^d\eqdef \{u\in \mathbb{R}^d:\,\|u\|\le \}$.
For each $x\in \mathbb{N}_+$, let $P^n(x,\cdot)\eqdef \mathbb{P}(X_t\in \cdot|X_0=x)=\mathbb{P}(f(Z_t)\in \cdot|f(Z_0)=f(x))=f_{\#}P^n(x,\cdot)$ then the Kantorovich duality, see e.g.~\citep[Theorem 5.10]{VillaniBook_2009}, implies that $f_{\#}:\mathcal{P}_1(\mathbb{R}^d)\to \mathcal{P}_1(B_r^d)$ is $1$-Lipschitz; whence~\eqref{eq:satisfaction__pointversion} imples that: for each $x\in [0,1]^d$ and every $n\in \mathbb{N}$ we have
\begin{equation}
\label{eq:satisfaction__pointversion}
    \mathcal{W}_1\big(
        P^n(x,\cdot),P^n(y,\cdot)
    \big)
\le 
    \operatorname{Lip}(f)
    \mathcal{W}_1\big(
        \tilde{P}^n(x,\cdot),\tilde{P}^n(y,\cdot)
    \big)
\le 
    \kappa
    \,
    \|x-y\|
.
\end{equation}
Thus, Assumption~\ref{ass:contractivity} holds.  Finally, we note that Assumption~\ref{ass:Comp_Support} holds since each $P^n(x,\cdot)\in \mathcal{P}_1(B_r^d)$.
\end{proof}
\begin{proof}[{Proof of Proposition~\ref{prop:proj_SDE}}]
For any $\mu \in \mathcal{P}_1(\mathbb{R}^D)$ consider the unique strong solution (which exists by our Lipschitz assumption)
For the following SDE (which is a Markov process)
\begin{align*}
Z_t^{\mu} = Z_0^{\mu} + \int_0^t \, \mu(Z_s^{\mu})\,ds + \int_0^t\, W_s
\end{align*}
where $W_{\cdot}\eqdef (W_n)_{n=0}^{\infty}$ is a $d$-dimensional Brownian motion and $Z_0^{\mu}$ is distributed according to $\mu$.
For every $n\in \mathbb{N}_+$ let $\tilde{P}^n\mu\eqdef \mathbb{P}(Z_n^{\mu}\in \cdot)$ and, for each $x\in \mathbb{R}^d$, let $\tilde{P}^n(x,\cdot)\eqdef \tilde{P}^n\delta_x$.
Then~\cite[Theorem 1.1]{LuoWangExpConvergence_2016} implies that: for all $n\in \mathbb{N}$ and each $\mu,\nu\in\mathcal{P}_1(\mathbb{R}^d)$ we have
\begin{equation}
\label{eq:satisfaction}
    \mathcal{W}_1\big(
        \tilde{P}^n\mu,\tilde{P}^n\nu
    \big)
\le 
    \kappa
    \,
    \mathcal{W}_1(\mu,\nu)
\end{equation}
where $\kappa = \exp(-K)$; note that $\kappa \in (0,1)$ since $K>0$.  That is, $(\tilde{P}^n)_{n=0}^{\infty}$ satisfies Assumption~\ref{ass:contractivity} upon taking $\mu=\delta_x$ and $\nu=\delta_y$, for any given $x,y\in \mathbb{R}^d$, since
$\mathcal{W}_1(\delta_x,\delta_y)=\|x-y\|$
, see e.g.~\cite[page 99 point 5]{VillaniBook_2009} or note that the only coupling between $\delta_x$ and $\delta_y$ is the product measure $\delta_x\otimes\delta_y$).
\begin{align*}
    \mathcal{W}_1\big(
        \tilde{P}^n(x,\cdot),\tilde{P}^n(x,\cdot)
    \big)
\le 
    \kappa
    \,
    \|x-y\|
.
\end{align*}
Applying Lemma~\ref{lem:BoundedLipProjection} yields the conclusion.
\end{proof}

\begin{proposition}
\label{prop:projection_martingale}
Consider the setting of Example~\ref{ex:martingales}.  Then, the process $X_{\cdot}$ satisfies both Assumptions~\ref{ass:Comp_Support} and~\ref{ass:contractivity}.
\end{proposition}
\begin{proof}[{Proof of Proposition~\ref{prop:projection_martingale}}]
Under our assumptions $\sigma$ satisfies~\cite[Assumption (A8) (1) and (A8) (3)]{Wang_McKeanVlassovReflectedExponential_SPA_2023}.  Therefore, the stochastic process $Z_{\cdot}\eqdef (Z_t)_{t\ge 0}$ defined by
\begin{equation}
    Z_t \eqdef \int_0^t \sigma(Z_s)\,dW_s
\end{equation}
where $W_{\cdot}$ is a $d$-dimensional Brownian motion, satisfies the conditions of~\cite[Corollary 4.4]{Wang_McKeanVlassovReflectedExponential_SPA_2023} from which we deduce that $Z_{\cdot}$ satisfies Assumption~\ref{ass:contractivity}.  Applying Lemma~\ref{lem:BoundedLipProjection} yields the conclusion.
\end{proof}

\subsection{Markov Processes Satisfying a Log-Sobolev Inequalities}
\label{s:Examples__ss:LSI}

Our main result is equally valid under the assumption that the stationary distribution of the Markov chain and its kernels all satisfy a log-Sobolev inequality (LSI).  Since their introduction, LSIs have been heavily studied \cite{GrossLogSobolev,LedouxNourdinPEccati_SteinGAFA_OTIneq_2015,Zimmerman_JFALogSobolev,InglisPapageourgio_JFAGibbsLogSobolev_PTRF_2019,ChenChewiNilesWeed_MixtureDimFreLogSob_JFA_2021} and have found numerous applications in differential privacy \cite{minami2016differential,ye2022differentially}, optimization \cite{LeCunetal_JStatMech_EntropSGD_2019}, random matrix theory \cite{Winger_AnnMath_1955_I,Winger_AnnMath_1955_II}, optimal transport \cite{DoleraMaini_AHPCStat_LipContinuityProbOT}, since they typically imply~\cite{Gozlan_LSI_II_AHPC_DFRatesSnowflake,Gozlan_LSI_IC_VPDE__DFConcentration} and effectively characterizes \cite{Gozlan_LSI_III_AnnProb_2009} dimension-free rate for concentration of measure.
We define the \textit{entropy functional} $\mathcal{H}_{\mu}$ associated to any Borel probability measure $\mu$ on $\mathbb{R}^d$ acts on smooth functions $g:\mathbb{R}^d\to \mathbb{R}$ by
\begin{align*}
    \mathbb{H}_{\mu}(g) \eqdef \mathbb{E}_{X\sim \mu}
        \Big[
            g(X)
            \log\Big(
                \frac{g(X)}{\mathbb{E}_{Z\sim \mu}[g(Z)]}
            \Big)
        \Big]
.
\end{align*}
The entropy functional can be used to express the log-Sobolev inequalities.  
\begin{definition}[Log-Sobolev Inequality]
\label{defn:LSI}
A probability measure $\mu$ on $\mathbb{R}^d$ is said to satisfy a log-Sobolev inequality with constant $C>0$ ($\operatorname{LSI}_C$) if for every smooth function $g:\mathbb{R}^d\to\mathbb{R}$
\begin{align*}
\smash{
    \mathbb{H}_{\mu}(g^2)
    \le 
    C \ 
    \mathbb{E}_{X\sim \mu}[\|\nabla g(X)\|^2]
}
\end{align*}
\end{definition}
We require that the Markov process is time-homogeneous to admit a satisfactory measure.  
Further, we require that its Markov kernel and its stationary measure all satisfy $\operatorname{LSI}_C$.
\begin{assumption}[Satisfactions of the Log-Sobolev Inequality]
\label{ass:MC}
There exists a $C>0$ such that $\bar{\mu}$, $\mu_0$, and $P(x,\cdot)$ all satisfy $\operatorname{LSI}_C$, for each $x\in \mathcal{X}$.
\end{assumption}

Instead of the compact support Assumption~\ref{ass:Comp_Support} we may consider the following weaker condition.
\begin{assumption}[Exponential Moments]
\label{ass:Exp_Moment}
There exist $\lambda,\tilde{C}>0$ and $\gamma\in (0,1)$ such that: for each $x\in \mathcal{X}$ we have $\mathbb{E}_{X\sim P(x,\cdot)}[e^{\lambda|X|}]\le \gamma\, e^{\lambda|x|}+\tilde{C}$.
\end{assumption}
Note that, Assumption~\ref{ass:Comp_Support} implies~\ref{ass:Exp_Moment}, but not conversely.

Several examples of Markov processes satisfying LSI inequalities are given in~\cite{ledoux2006concentration} and Gaussian processes satisfy the Exponential Moments Assumption.  

\begin{proposition}[Log-Sobolev Conditions and Exponential Moments Imply Assumption~\ref{ass:contractivity}]
\label{prop:LSI_to_Contractivity}
If Assumptions~\ref{ass:Exp_Moment} and~\ref{ass:MC} hold then the process $X_{\cdot}$ satisfies Assumption~\ref{ass:contractivity}.
\end{proposition}

\begin{proof}[{Proof of Proposition~\ref{prop:LSI_to_Contractivity}}]

Under the log-Sobolev Assumption~\ref{ass:MC},~\cite[Theorem 1.3]{BobkovGotze_1999_JFA} can be applied to $\bar{\mu}$ and $P(x,\cdot)$ for each $x\in \mathcal{X}$, implying that the transport inequalities hold: for each $\nu\in\mathcal{P}(\mathcal{X})$ and each $\tilde{\mu}\in \{\bar{\mu},\mu_0\}\cup \{P(x,\cdot)\}_{x\in \mathcal{X}}$
\begin{equation}
\label{eq:term_II__transport_inequalities}
    \mathcal{W}_1(\tilde{\mu},\nu)^2
\le 
    2C^2\, \operatorname{KL}(\nu|\tilde{\mu})
\end{equation}
where we recall the definition of the Kullback–Leibler divergence $\operatorname{KL}(\nu|\mu)\eqdef \mathbb{E}_{X\sim \nu}[\log(\frac{d\nu}{d\mu}(X))]$.  Thus,~\eqref{eq:term_II__transport_inequalities} implies that the following exponential contractility property of the Markov kernel: there exists some $\kappa \in (0,1)$ such that for each $x,\tilde{x}\in \mathcal{X}$ and every $t\in \mathbb{N}_+$
\begin{equation}
    \mathcal{W}_1\big(P^t(x,\cdot),P^t(\tilde{x},\cdot)\big)
    \le 
    \kappa^t\, \|x-\tilde{x}\|
    .
\end{equation}
This completes the proof.
\end{proof}

\section{Transformer Definition Details}\label{sec:TransformerDetails}

For any $F\in \mathbb{N}_+$, we will consider a weighted (parametric) variant of the \textit{layer normalization} function of \cite{ba2016layer}, which permits a variable level of regularization.  Our weighted \textit{layer normalization} is defined by $\operatorname{LayerNorm}:\mathbb{R}^F\to \mathbb{R}^F$ defined for any $u\in \mathbb{R}^F$ by 
\begin{align*}
    \layernorm(u;\gamma,\beta,w)\eqdef  
        \gamma\, 
        \frac{(u-\mu_u^w)}{\sqrt{1 + (\sigma_u^w)^2}} + \beta 
\end{align*}
where $\mu_u^w\eqdef \sum_{i=1}^F\, \frac{w}{F} u_i$ and $(\sigma_u^w)^2 \eqdef \sum_{i=1}^F\, \frac{w}{F}\,\|u_i-\mu_u\|^2$, $\softplus \eqdef \ln (1 + \exp(\cdot))$, parameters $\beta\in \mathbb{R}^F$ and $\gamma \in \mathbb{R}$, and the normalization strength parameter $w\in [0,1]$ with $w=1$ being the default choice.   Here, we prohibit the layer norm from magnifying the size of its outputs when the layer-wise weighted variance $\sigma_u^w$ is small.\footnote{Note that this formulation of the layer norm avoids division by $0$ when the entries of $u$ are identical.}


\begin{definition}[Multi-Head Self-Attention]\label{defn:MultiHead}
Fix $\idim \in \N$. For {$x \in  \R^{M \times \idim}$}, $Q, K \in \R^{\kdim \times \idim}$, and $V \in \R^{\vdim \times \idim}$, where we have key-dimension $\kdim \in \N$ and value-dimension $\vdim \in \N$; we define
\begin{align*}
        \attention(
            x; Q, K, V 
        )
    \eqdef    
        \bigg(
            \sum_{j=0}^M
                \softmax\left(
                    \Big(
                        \frac{\langle
                            Qx_m, Kx_i
                        \rangle}{\sqrt{d_k}
                        }
                    \Big)_{i = 0}^{M}
                \right)_j
                Vx_j
        \bigg)_{m=1}^M
    \in \R^{M \times \vdim}.
\end{align*}
For $H \in \N$, set $Q \eqdef (Q^{(h)})_{h=1}^H, K \eqdef (K^{(h)})_{h=1}^H \subseteq \R^{\kdim \times \idim}$, $V \eqdef (V^{(h)})_{h=1}^H \subseteq \R^{\vdim \times \idim}$, and $W \eqdef (W^{(h)})_{h=1}^H \subseteq \R^{\idim \times \vdim}$. For $x \in  \R^{M \times \idim}$, we define
\begin{align*}
        \multihead(x; Q, K, V, W) 
    \eqdef 
        \bigg(
            \sum_{h = 1}^H
                W^{(h)} \attention(
                    x; Q^{(h)}, K^{(h)}, V^{(h)}
                )_m 
        \bigg)_{m = 1}^M \in \R^{M \times \idim}.
\end{align*}
\end{definition}

Each transformer block takes a set of inputs and intersperses normalization via layer norms, contextual comparisons via multi-head attention mechanisms, and non-linear transformations via a {single layer perceptron (SLP)}.  
%
We also allow the transformer block to extend or contract the length of the generated sequence.  

\begin{definition}[Transformer Block]
\label{defn:TransformerBlock}
Fix a non-affine activation function $\sigma\in C^{\infty}(\R)$.  
Fix a dimensional multi-index $d=(\idim, \kdim, \vdim, \ldim, \odim) \in \N^{5}$, a sequence length $M \in \N_+$, and a number of self-attention heads $H\in \mathbb{N}_+$.  
A transformer block is a permutation equivariant map $\tblock:\mathbb{R}^{M\times \idim}\to \mathbb{R}^{M\times \odim}$ represented for each $x\in \mathbb{R}^{M\times \idim}$
\begin{align}\label{eq:TB}
\begin{aligned}
        \tblock (x) 
    & \eqdef  
        \Big(
        \layernorm\big(
        B^{(1)} x'_m + B^{(2)} \big(\sigma\bullet (A\, x'_m + a)\big) ;\, \gamma_2,\beta_2,w_2
        \big)
        \Big)_{m=1}^M
    \\
        x'
    & \eqdef 
        \Big(
    \layernorm\big(
        x_m + \multihead(x; Q, K, V, W)_m;\,\gamma_1, \beta_1,w_1
        \big)
        \Big)_{m=1}^M
\end{aligned}
\end{align}
for $\gamma_1,\gamma_2\in \mathbb{R}$, 
$w_1,w_2\in [0,1]$, 
$\beta_1\in \mathbb{R}^{\idim}$, 
$\beta_2\in \mathbb{R}^{\odim}$, 
$A\in \mathbb{R}^{\ldim \times \idim}$, $a\in \mathbb{R}^{\ldim}$, 
$B^{(1)} \in \mathbb{R}^{\odim \times \idim}$, $B^{(2)} \in \mathbb{R}^{\odim \times \ldim}$, and $Q, K, V, W$ as in \Cref{defn:MultiHead}. {Above, we write $\bullet$ for a pointwise application.}

The class of transformer blocks with representation~\eqref{eq:TB} and bounds on $\gamma_1, \gamma_2, \beta_1, \beta_2, a, A,$ ${B^{(1)}}, {B^{(2)}},$ ${Q}, {K}, {V}, {W}$ is denoted by $\tblockclass$.
\end{definition}



A transformer concatenates several transformer blocks before passing their outputs to an affine layer and ultimately outputting its prediction.

\begin{definition}[Transformers]
\label{defn:Transformer}
Fix depth $L\in \mathbb{N}_+$, memory $M\in \mathbb{N}$, width $W\in \mathbb{N}^5_+$, number of heads $H\in \mathbb{N}_+$, and input-output dimensions $D,d\in\mathbb{N}_+$.  
A transformer (network) is a map $\transformer: \mathbb{R}^{M\times D}\to \mathbb{R}^d$ with representation
\begin{equation}
\label{eq:Transformer_Nets}
        \transformer(x)
    =
        A \big(
            \vectorize_{1 + M, \odim^{L}}
            \circ 
            \tblock_L
            \circ 
            \dots 
            \circ 
            \tblock_1
            (x)
        \big)
    +
        b
\end{equation}
where multi-indices 
$d^{l}=(\idim^{l},\kdim^{l},\vdim^{l},\ldim^{l},\odim^{l}) \leq W$ 
are such that 
$\idim^1=D$, 
$\idim^{l+1}=\odim^l$ 
for each $l=1, \dots,L-1$, 
and where 
$H \eqdef (H^l)_{l=1}^L$ are the number of self-attention heads,
$C' \eqdef (C^l)_{l=1}^L$ the parameter bounds,
and for $l=1,\dots,L$ 
we have $\tblock_l \in \tblockclass_l$, where $\tblockclass_l$ is a transformer block class with $\idim = d^{l}, M=M,$ and $H=H^{l}$.
Furthermore, $A \in \mathbb{R}^{d \times M\odim^{L}}$ and $b\in \mathbb{R}^d$.

The set of transformer networks with representation~\eqref{eq:Transformer_Nets} and bounds on $A, b$ is denoted by $\transformerclass$.
\end{definition}

\section{Elucidation of Constants in Theorem~\ref{thrm:main}}
\label{s:dettables}
The aim of this section is to elucidate the magnitude of the constants appearing in Theorem~\ref{thrm:main}.  We aim of to make each of these concrete by numerically estimating them, which we report in a series of tables.  Importantly, we see how subtle choices of the activation function used to define the transformer model can have dramatic consequences on the size of these constants, which could otherwise be hidden in big $\mathcal{O}$ notation.

Interestingly, in \Cref{tab:Activations,tab:NeuralNet}, we see that the $\operatorname{softplus}$ activation function produces significantly tighter bounds than the $\tanh$ activation function through much smaller constants, and the $\operatorname{GeLU}$ and $\operatorname{SWISH}$ activation functions are a relatively comparable second-place.

The bounds depicted in \Cref{tab:NeuralNet} exhibits a notable trait of independence from both input dimension and the compactum they are defined on.  Notably, the selection of latent dimensionality demonstrates a relatively minor influence in contrast to the pronounced impact of parameter bounds. This suggests that while adjusting the latent dimension may have some effect, the primary driver of the derivative bound lies within the constraints imposed on the parameters. Despite the seemingly conservative nature of the chosen parameter-bounds, it is important to acknowledge their alignment with the parameter ranges observed in trained transformer-models. 

Note that the latter can be observed as well for Multi-Head attention (\Cref{tab:MultiHead}), however, we see that here the input dimension (composed of $\idim$ and $M$) is of greater importance with respect to the derivative bound. 

\begin{table}[H]
    \centering
	\ra{1.3}
    \caption{$C^s$-bounds of activation functions based on numerical maximization of analytic derivatives in Appendix \ref{sec:Activation}.}
        \small
\begin{tabular}{rrrrr}
\toprule
Bound & $\operatorname{softplus}$ & $\operatorname{GeLU}$ & $\tanh$ & $\operatorname{Swish}$
\\    
\midrule 
$C^{1}$ & \texttt{\footnotesize  0.25} & \texttt{\footnotesize    1.12} & \texttt{\footnotesize        4.00} & \texttt{\footnotesize 1.10}  \\
$C^{2}$ & \texttt{\footnotesize  0.10} & \texttt{\footnotesize    0.48} & \texttt{\footnotesize        8.00} & \texttt{\footnotesize 0.50}  \\
$C^{3}$ & \texttt{\footnotesize  0.12} & \texttt{\footnotesize    0.75} & \texttt{\footnotesize       16.00} & \texttt{\footnotesize 0.31}  \\
$C^{4}$ & \texttt{\footnotesize  0.13} & \texttt{\footnotesize    1.66} & \texttt{\footnotesize       32.00} & \texttt{\footnotesize 0.50}  \\
$C^{5}$ & \texttt{\footnotesize  0.25} & \texttt{\footnotesize    4.34} & \texttt{\footnotesize      156.65} & \texttt{\footnotesize 0.66}  \\
$C^{6}$ & \texttt{\footnotesize  0.41} & \texttt{\footnotesize   12.95} & \texttt{\footnotesize     1651.32} & \texttt{\footnotesize 1.50}  \\
$C^{7}$ & \texttt{\footnotesize  1.06} & \texttt{\footnotesize   42.77} & \texttt{\footnotesize    20405.43} & \texttt{\footnotesize 2.91}  \\
$C^{8}$ & \texttt{\footnotesize  2.39} & \texttt{\footnotesize  153.76} & \texttt{\footnotesize   292561.95} & \texttt{\footnotesize 8.50}  \\
$C^{9}$ & \texttt{\footnotesize  7.75} & \texttt{\footnotesize  594.17} & \texttt{\footnotesize  4769038.09} & \texttt{\footnotesize21.76}  \\
$C^{10}$& \texttt{\footnotesize 22.25} & \texttt{\footnotesize 2445.69} & \texttt{\footnotesize 87148321.71} & \texttt{\footnotesize77.50}  \\
\bottomrule
\end{tabular} 
\label{tab:Activations}
\end{table}

The bound of the layer-norm (see \Cref{tab:LayerNorm}) seems to be particularly effected by the domain it is defined on, which can be problematic if it appears in later layers. An immediate solution is the usage of its parameter $\gamma$, a more drastic approach would be applications in combination with an upstream sigmoid activation.

Eventually, as also shown in \Cref{fig:boundillustration}, we included in \Cref{tab:MultiHead,tab:TBlock} a comparison of using type-specific bounds (see \Cref{thm:DerivativeBoundByType}) or level-specific bounds (\Cref{thm:DerivativeBoundByLevel}) in the computation of the constants. This effect seems to become more evident with higher number of function compositions.

\begin{table}[H]
\centering
\small
\caption{Derivative Bounds of the Perceptron Layer by derivative level according to \Cref{lem:FeedForward}.}
\label{tab:NeuralNet}
\begin{tabular}{ccc|ccccc}
\toprule
\multicolumn{3}{c|}{Parameters} & \multicolumn{5}{c}{Derivative Level} \\
 $\sigma$ & $\ldim$ & $C^{\{A, B^{(1)}, B^{(2)}\}}$ & 1 & 2 & 3 & 4 & 5 \\
\midrule
softmax & 64 & 1.0 & \texttt{\footnotesize 	17.00} & \texttt{\footnotesize 	50.47} & \texttt{\footnotesize 	236.94} & \texttt{\footnotesize 	1.34E+03} & \texttt{\footnotesize 	1.33E+04} \\ \hline
tanh & &  & \texttt{\footnotesize 	129.00} & \texttt{\footnotesize 	1.02E+03} & \texttt{\footnotesize 	8.96E+03} & \texttt{\footnotesize 	9.83E+04} & \texttt{\footnotesize 	1.34E+06} \\
GeLU & &  & \texttt{\footnotesize 	73.25} & \texttt{\footnotesize 	237.41} & \texttt{\footnotesize 	1.25E+03} & \texttt{\footnotesize 	1.16E+04} & \texttt{\footnotesize 	1.81E+05} \\
SWISH& &  & \texttt{\footnotesize 	71.39} & \texttt{\footnotesize 	236.78} & \texttt{\footnotesize 	870.75} & \texttt{\footnotesize 	5.33E+03} & \texttt{\footnotesize 	4.13E+04} \\ \hline
 &16 &      & \texttt{\footnotesize 	5.00} & \texttt{\footnotesize 	12.62} & \texttt{\footnotesize 	59.23} & \texttt{\footnotesize 	334.15} & \texttt{\footnotesize 	3.31E+03} \\
 &32 &      & \texttt{\footnotesize 	9.00} & \texttt{\footnotesize 	25.24} & \texttt{\footnotesize 	118.47} & \texttt{\footnotesize 	668.30} & \texttt{\footnotesize 	6.63E+03} \\
 &128 &     & \texttt{\footnotesize 	33.00} & \texttt{\footnotesize 	100.95} & \texttt{\footnotesize 	473.87} & \texttt{\footnotesize 	2.67E+03} & \texttt{\footnotesize 	2.65E+04} \\
 &256 &     & \texttt{\footnotesize 	65.00} & \texttt{\footnotesize 	201.90} & \texttt{\footnotesize 	947.75} & \texttt{\footnotesize 	5.35E+03} & \texttt{\footnotesize 	5.30E+04} \\ \hline
 & & 0.01   & \texttt{\footnotesize 	17.00} & \texttt{\footnotesize 	44.32} & \texttt{\footnotesize 	180.94} & \texttt{\footnotesize 	919.87} & \texttt{\footnotesize 	6.52E+03} \\
 & & 0.1    & \texttt{\footnotesize 	17.00} & \texttt{\footnotesize 	44.38} & \texttt{\footnotesize 	181.00} & \texttt{\footnotesize 	919.92} & \texttt{\footnotesize 	6.52E+03} \\
 & & 10.0     & \texttt{\footnotesize 	17.00} & \texttt{\footnotesize 	660.15} & \texttt{\footnotesize 	5.62E+04} & \texttt{\footnotesize 	4.17E+06} & \texttt{\footnotesize 	6.73E+08} \\
 & & 100.0    & \texttt{\footnotesize 	17.00} & \texttt{\footnotesize 	6.16E+04} & \texttt{\footnotesize 	5.60E+07} & \texttt{\footnotesize 	4.17E+10} & \texttt{\footnotesize 	6.73E+13} \\
\bottomrule
\end{tabular}
\end{table}

\begin{table}[htbp]
\small
\centering
\caption{Layer Norm}
\label{tab:LayerNorm}
\begin{tabular}{ccc|ccccc}
\toprule
\multicolumn{3}{c|}{Parameters} & \multicolumn{5}{c}{Derivative Level} \\
 $k$ & $\Vert K \Vert$ & $\gamma$ & 1 & 2 & 3 & 4 & 5 \\
\midrule
5 & 10.0 & 0.1  & \texttt{\footnotesize 	18.67} & \texttt{\footnotesize 	28.56} & \texttt{\footnotesize 	104.49} & \texttt{\footnotesize 	1.49E+03} & \texttt{\footnotesize 	4.93E+03} \\ \hline
3 &      &      & \texttt{\footnotesize 	18.67} & \texttt{\footnotesize 	28.56} & \texttt{\footnotesize 	104.49} & \texttt{\footnotesize 	945.21} & \texttt{\footnotesize 	4.93E+03} \\ 
10 &      &      & \texttt{\footnotesize 	18.67} & \texttt{\footnotesize 	28.56} & \texttt{\footnotesize 	104.49} & \texttt{\footnotesize 	1.49E+03} & \texttt{\footnotesize 	4.93E+03} \\ 
   20 &      &      & \texttt{\footnotesize 	18.67} & \texttt{\footnotesize 	28.56} & \texttt{\footnotesize 	104.49} & \texttt{\footnotesize 	1.49E+03} & \texttt{\footnotesize 	4.93E+03} \\ \hline
   & 0.1&      & \texttt{\footnotesize 	0.17} & \texttt{\footnotesize 	3.61} & \texttt{\footnotesize 	5.37} & \texttt{\footnotesize 	7.05} & \texttt{\footnotesize 	8.87} \\
   & 1.0&      & \texttt{\footnotesize 	1.73} & \texttt{\footnotesize 	5.20} & \texttt{\footnotesize 	6.95} & \texttt{\footnotesize 	8.71} & \texttt{\footnotesize 	10.64} \\
   &100.0 &      & \texttt{\footnotesize 	321.71} & \texttt{\footnotesize 	7.68E+03} & \texttt{\footnotesize 	7.88E+05} & \texttt{\footnotesize 	1.42E+08} & \texttt{\footnotesize 	4.39E+09} \\
   &1000.0 &      & \texttt{\footnotesize 	321.71} & \texttt{\footnotesize 	7.68E+03} & \texttt{\footnotesize 	7.88E+05} & \texttt{\footnotesize 	1.42E+08} & \texttt{\footnotesize 	4.39E+09} \\\hline
      &       & 0.01    & \texttt{\footnotesize 	1.73} & \texttt{\footnotesize 	2.07} & \texttt{\footnotesize 	2.24} & \texttt{\footnotesize 	2.42} & \texttt{\footnotesize 	2.59} \\
   &       &  1.0    & \texttt{\footnotesize 	321.71} & \texttt{\footnotesize 	7.75E+03} & \texttt{\footnotesize 	7.91E+05} & \texttt{\footnotesize 	1.42E+08} & \texttt{\footnotesize 	4.44E+09} \\
\bottomrule
\end{tabular}
\end{table}

\begin{table}[p]
\footnotesize
\centering
\setlength{\tabcolsep}{0.5em}
\caption{Derivative Bounds of Transformer Block by derivative level according to \Cref{thm:TransformerBlockSobolevBoundLevel}.}
\label{tab:TBlock}
\begin{tabular}{cccc|ccccc}
\toprule
\multicolumn{4}{c|}{Parameters} & \multicolumn{5}{c}{Derivative Level} \\
 $\idim$ & {\tiny $C^{\{K, Q, V, W \}}$} & {\tiny $C^{\{A, B^{(1,2)}\}}$} & $\gamma$ & 1 & 2 & 3 & 4 & 5 \\
\midrule
5&0.01&0.001& 0.01& \texttt{\footnotesize 	21.15} & \texttt{\footnotesize 	1.13E+04} & \texttt{\footnotesize 	4.81E+06} & \texttt{\footnotesize 	2.59E+09} & \texttt{\footnotesize 	2.22E+11} \\ \hline
\multicolumn{4}{c|}{ --- using derivative level --- } & \texttt{\footnotesize 212.70} & \texttt{\footnotesize1.53E+06} & \texttt{\footnotesize7.47E+09} & \texttt{\footnotesize4.55E+13} & \texttt{\footnotesize3.75E+16} \\ \hline
10&&& & \texttt{\footnotesize 	111.32} & \texttt{\footnotesize 	4.51E+05} & \texttt{\footnotesize 	1.71E+09} & \texttt{\footnotesize 	1.45E+13} & \texttt{\footnotesize 	8.70E+16} \\
20&&& & \texttt{\footnotesize 	1.29E+03} & \texttt{\footnotesize 	1.25E+08} & \texttt{\footnotesize 	3.47E+13} & \texttt{\footnotesize 	1.85E+19} & \texttt{\footnotesize 	2.20E+24} \\\hline

&0.001&& & \texttt{\footnotesize 	21.15} & \texttt{\footnotesize 	1.13E+04} & \texttt{\footnotesize 	4.81E+06} & \texttt{\footnotesize 	2.59E+09} & \texttt{\footnotesize 	2.22E+11} \\
&0.1&& & \texttt{\footnotesize 	21.16} & \texttt{\footnotesize 	1.13E+04} & \texttt{\footnotesize 	4.83E+06} & \texttt{\footnotesize 	2.61E+09} & \texttt{\footnotesize 	2.32E+11} \\
&1.0&& & \texttt{\footnotesize 	22.30} & \texttt{\footnotesize 	4.64E+04} & \texttt{\footnotesize 	6.96E+08} & \texttt{\footnotesize 	1.70E+13} & \texttt{\footnotesize 	1.95E+17} \\\hline

&&0.0001& & \texttt{\footnotesize 	5.05} & \texttt{\footnotesize 	126.27} & \texttt{\footnotesize 	1.12E+04} & \texttt{\footnotesize 	4.94E+06} & \texttt{\footnotesize 	6.87E+08} \\
&&0.01& & \texttt{\footnotesize 	182.17} & \texttt{\footnotesize 	1.12E+06} & \texttt{\footnotesize 	4.66E+09} & \texttt{\footnotesize 	2.43E+13} & \texttt{\footnotesize 	1.71E+16} \\
&&0.1& & \texttt{\footnotesize 	1.79E+03} & \texttt{\footnotesize 	1.12E+08} & \texttt{\footnotesize 	4.65E+12} & \texttt{\footnotesize 	2.42E+17} & \texttt{\footnotesize 	1.70E+21} \\\hline

&&&0.0001 & \texttt{\footnotesize 	0.21} & \texttt{\footnotesize 	108.21} & \texttt{\footnotesize 	4.44E+04} & \texttt{\footnotesize 	2.27E+07} & \texttt{\footnotesize 	1.58E+09} \\
&&&0.001 & \texttt{\footnotesize 	2.09} & \texttt{\footnotesize 	1.09E+03} & \texttt{\footnotesize 	4.46E+05} & \texttt{\footnotesize 	2.29E+08} & \texttt{\footnotesize 	1.60E+10} \\
&&&0.1 & \texttt{\footnotesize 	240.09} & \texttt{\footnotesize 	2.45E+05} & \texttt{\footnotesize 	7.31E+08} & \texttt{\footnotesize 	4.96E+12} & \texttt{\footnotesize 	8.79E+15} \\ 
\bottomrule
\end{tabular}
\end{table}

\begin{sidewaystable}[p]
\centering
\caption{Derivative Bounds of Multi-Head Attention by derivative level according to \Cref{cor:MultiHead}.}
\small
\begin{tabular}{cccccccc|ccccc}
\toprule
\multicolumn{8}{c|}{Parameters} & \multicolumn{5}{c}{Derivative Level} \\
 $\idim$ & $M$ & $\kdim$ & $C^K$ & $C^Q$ & $C^V$ & $C^W$ & $\Vert K \Vert$
 & 1 & 2 & 3 & 4 & 5 \\
\midrule
5& 1 & 3 & 0.1 & 0.1 & 0.1 & 0.1 & 1.0  & \texttt{\footnotesize 	7.67} & \texttt{\footnotesize 	46.08} & \texttt{\footnotesize 	184.82} & \texttt{\footnotesize 	931.81} & \texttt{\footnotesize 	5.73E+03} \\ \hline
\multicolumn{8}{c|}{ --- using derivative level ---}  & \texttt{\footnotesize 	7.67} & \texttt{\footnotesize 	46.15} & \texttt{\footnotesize 	 186.90} & \texttt{\footnotesize 1.01E+03} & \texttt{\footnotesize 	7.84E+03} \\ \hline
10 & & &     &     &     &     &       & \texttt{\footnotesize 	7.82E+03} & \texttt{\footnotesize 	4.87E+04} & \texttt{\footnotesize 	8.14E+05} & \texttt{\footnotesize 	2.95E+08} & \texttt{\footnotesize 	2.38E+11} \\
20&  &   &     &     &     &     &      & \texttt{\footnotesize 	1.57E+04} & \texttt{\footnotesize 	1.09E+05} & \texttt{\footnotesize 	1.03E+07} & \texttt{\footnotesize 	9.25E+09} & \texttt{\footnotesize 	1.50E+13} \\ \hline 
 & 5 &   &     &     &     &     &      & \texttt{\footnotesize 	3.90E+03} & \texttt{\footnotesize 	2.37E+04} & \texttt{\footnotesize 	1.34E+05} & \texttt{\footnotesize 	9.96E+06} & \texttt{\footnotesize 	3.83E+09} \\
 &20 &   &     &     &     &     &      & \texttt{\footnotesize 	4.71E+06} & \texttt{\footnotesize 	2.85E+07} & \texttt{\footnotesize 	7.11E+08} & \texttt{\footnotesize 	2.66E+12} & \texttt{\footnotesize 	1.40E+16} \\\hline
 &   & 3 &     &     &     &     &      & \texttt{\footnotesize 	7.67} & \texttt{\footnotesize 	45.99} & \texttt{\footnotesize 	183.98} & \texttt{\footnotesize 	920.53} & \texttt{\footnotesize 	5.53E+03} \\
 &   &10 &     &     &     &     &      & \texttt{\footnotesize 	7.70} & \texttt{\footnotesize 	46.41} & \texttt{\footnotesize 	190.35} & \texttt{\footnotesize 	1.07E+03} & \texttt{\footnotesize 	1.03E+04} \\
 &   &20 &     &     &     &     &      & \texttt{\footnotesize 	7.75} & \texttt{\footnotesize 	47.52} & \texttt{\footnotesize 	226.02} & \texttt{\footnotesize 	2.87E+03} & \texttt{\footnotesize 	1.23E+05} \\\hline
 &   &   & 0.01 &     &     &     &      & \texttt{\footnotesize 	7.65} & \texttt{\footnotesize 	45.91} & \texttt{\footnotesize 	183.61} & \texttt{\footnotesize 	918.01} & \texttt{\footnotesize 	5.51E+03} \\
 &   &   & 1.0 &     &     &     &      & \texttt{\footnotesize 	7.90} & \texttt{\footnotesize 	54.45} & \texttt{\footnotesize 	740.55} & \texttt{\footnotesize 	6.51E+04} & \texttt{\footnotesize 	9.33E+06} \\
 &   &   &     & 0.01&     &     &      & \texttt{\footnotesize 	7.65} & \texttt{\footnotesize 	45.91} & \texttt{\footnotesize 	183.61} & \texttt{\footnotesize 	918.01} & \texttt{\footnotesize 	5.51E+03} \\
 &   &   &     & 1.0 &     &     &      & \texttt{\footnotesize 	7.90} & \texttt{\footnotesize 	54.45} & \texttt{\footnotesize 	740.55} & \texttt{\footnotesize 	6.51E+04} & \texttt{\footnotesize 	9.33E+06} \\
 &   &   &     &     & 0.01&     &      & \texttt{\footnotesize 	0.77} & \texttt{\footnotesize 	4.61} & \texttt{\footnotesize 	18.48} & \texttt{\footnotesize 	93.18} & \texttt{\footnotesize 	573.32} \\
 &   &   &     &     & 1.0 &     &      & \texttt{\footnotesize 	76.75} & \texttt{\footnotesize 	460.80} & \texttt{\footnotesize 	1.85E+03} & \texttt{\footnotesize 	9.32E+03} & \texttt{\footnotesize 	5.73E+04} \\
 &   &   &     &     &     & 0.01&      & \texttt{\footnotesize 	0.77} & \texttt{\footnotesize 	4.61} & \texttt{\footnotesize 	18.48} & \texttt{\footnotesize 	93.18} & \texttt{\footnotesize 	573.32} \\
 &   &   &     &     &     & 1.0 &      & \texttt{\footnotesize 	76.75} & \texttt{\footnotesize 	460.80} & \texttt{\footnotesize 	1.85E+03} & \texttt{\footnotesize 	9.32E+03} & \texttt{\footnotesize 	5.73E+04} \\
 &   &   &0.001&0.001&0.001&0.001&      & \texttt{\footnotesize 	0.00} & \texttt{\footnotesize 	0.00} & \texttt{\footnotesize 	0.02} & \texttt{\footnotesize 	0.09} & \texttt{\footnotesize 	0.55} \\
 &   &   &0.01 &0.01 &0.01 & 0.01&      & \texttt{\footnotesize 	0.08} & \texttt{\footnotesize 	0.46} & \texttt{\footnotesize 	1.84} & \texttt{\footnotesize 	9.18} & \texttt{\footnotesize 	55.08} \\
 &   &   &1.0  &1.0  &1.0  & 1.0 &      & \texttt{\footnotesize 	1.02E+03} & \texttt{\footnotesize 	8.06E+04} & \texttt{\footnotesize 	4.95E+07} & \texttt{\footnotesize 	5.70E+10} & \texttt{\footnotesize 	8.04E+13} \\\hline
 &   &   &     &     &     &     &0.1   & \texttt{\footnotesize 	0.77} & \texttt{\footnotesize 	32.14} & \texttt{\footnotesize 	142.34} & \texttt{\footnotesize 	752.79} & \texttt{\footnotesize 	4.68E+03} \\
 &   &   &     &     &     &     &10.0  & \texttt{\footnotesize 	79.00} & \texttt{\footnotesize 	259.65} & \texttt{\footnotesize 	5.54E+03} & \texttt{\footnotesize 	5.73E+05} & \texttt{\footnotesize 	8.04E+07} \\
 &   &   &     &     &     &     &100.0 & \texttt{\footnotesize 	1.02E+03} & \texttt{\footnotesize 	7.66E+04} & \texttt{\footnotesize 	4.88E+07} & \texttt{\footnotesize 	5.63E+10} & \texttt{\footnotesize 	7.90E+13} \\
\bottomrule
\end{tabular}
\label{tab:MultiHead}
\end{sidewaystable}

\clearpage

\section{Supporting Technical Results on the \texorpdfstring{$C^s$}{Cs}-Norms of Smooth Functions}
\label{s:Technical_Results}

This section contains many of the technical tools on which we build our analysis.  Most results concern smooth functions, especially their derivatives and those of compositions thereof.  However, the first set of results concerns the integral probability metric $d_s$.

\subsection{Integral Probability Metrics and Restriction to Compact Sets}
\label{ss:IMPs}

Fix $d\in \mathbb{N}_+$ and a non-empty compact subset $K\subseteq \mathbb{R}^d$.  
Observe that any Borel probability measure $\mu$ on $K$ can be canonically extended to a compactly supported Borel probability measure $\mu^+$ on all of $\mathbb{R}^d$ via
\begin{align*}
    \mu^+(B)\eqdef \mu(B\cap K),
\end{align*}
for any Borel subset $B$ of $\mathbb{R}^d$; noting only that $B\cap K$ is Borel.

Let $\mathcal{P}(K)$ denote the set of Borel probability measures on $K$.
Suppose that $K$ is a regular compact set, i.e.\ the closure of its interior is itself.  As usual, see~\cite{evans2022partial}, for any $s\in \mathbb{N}_+$, we denote the set of functions from the interior of $K$ to $\mathbb{R}$ with $s$ continuous partial derivatives thereon and with a continuous extension to $K$ by $C^s(K)$.  This space, is a Banach space when equipped with the (semi-)norm
\begin{align*}
\resizebox{1\linewidth}{!}{$
    \|f\|_{s:K}
\eqdef 
        \max_{k=1,\dots,s-1}\,
        \max_{\alpha\in \{1,\dots,d\}^k}
        \,
            \sup_{u\in K}\,
            \Big\|
                \frac{\partial^k f}{
                    \partial x_{\alpha_1}\dots\partial x_{\alpha_k}
                }(u)
            \Big\|
    +
        \max_{\alpha \in \{1,\dots,d\}^{s-1}}\,
            \operatorname{Lip}\biggl(
                \frac{\partial^{s-1}f}{\partial x_{\alpha_1}\dots\partial x_{\alpha_{s-1}}}
            \biggr)
.
$}
\end{align*}
We may define an associated integral probability metric $d_{s:K}$ on $\mathcal{P}(K)$ via
\begin{align*}
    d_{s:K}(\mu,\nu) 
    \eqdef 
    \sup_{f\in C^s(K)}\,
        \|
                \mathbb{E}_{X\sim \mu}[f(X)]
            -
                \mathbb{E}_{X\sim \nu}[f(X)]
        \|
\end{align*}
for any $\mu,\nu\in \mathcal{P}(K)$.  The main purpose of this technical subsection is simply to reassure ourselves, and the reader, that quantities $d_{s:K}(\mu,\nu)$ and $d_s(\mu^+,\nu^+)$ are equal for any $\mu,\nu \in \mathcal{P}(K)$.  Therefore, we may use them interchangeably.

\begin{lemma}[Consistency of Smooth IMP Extension - Beyond Regular Compact Sets]
\label{lem:res_ext}
Fix $d,s\in \mathbb{N}_+$ and let $K$ be a non-empty regular compact subset of $\mathbb{R}^d$.  For any $\mu,\nu\in \mathcal{P}(K)$ the following holds
\begin{align*}
    d_{s:K}(\mu,\nu)=d_s(\mu^+,\nu^+)
.
\end{align*}
\end{lemma}
\begin{proof}
Let $\operatorname{int}(K)$ denote the interior of $K$,
By the Whitney extension theorem, as formulated in \citep[Theorem A]{fefferman2005sharp}, for any $f\in C^s(K)$ there exists a $C^s$-extension $F:\mathbb{R}^d\to \mathbb{R}$ of $f|_{\operatorname{int}(K)}$ to all of $\mathbb{R}^d$; i.e.\ $F|_{\operatorname{int}(K)} =f $ and $\in C^s(\mathbb{R}^d)$.  
Since any continuous function is uniformly continuous on a compact set, $\operatorname{int}(K)$ is dense in $K$, and since uniformly continuous functions are uniquely determined by their values on compact sets, then $f$ coincides with $F$ on all of $K$ (not only on $\operatorname{int}(K)$).

For any $\mu\in \mathcal{P}(K)$, by definition of $\mu^+$ we have that
\begin{align*}
        \mathbb{E}_{X\sim \mu^+}[F(X)] 
    = 
        \mathbb{E}_{X\sim \mu^+}[F(X)I_{X\in K}]
    =
        \mathbb{E}_{X\sim \mu^+}[f(X)I_{X\in K}]
    =
        \mathbb{E}_{X\sim \mu}[f(X)]
.
\end{align*}
Therefore, for any $\mu,\nu \in \mathcal{P}(K)$ we conclude that and each $f\in C^s(K)$ there exists some $F\in C^s(\mathbb{R}^d)$ such that
\begin{align*}
        \mathbb{E}_{Y\sim \mu}[f(Y)] 
            - 
        \mathbb{E}_{Y\sim \nu}[f(Y)]
    =
        \mathbb{E}_{X\sim \mu^+}[F(X)] 
            - 
        \mathbb{E}_{X\sim \nu^+}[F(X)] 
.
\end{align*}
Consequentially, $d_{s:K}(\mu,\nu)\leq d_s(\mu^+,\nu^+)$.  Conversely, since the restriction of any $g\in C^s(\mathbb{R}^d)$ to $K$ belongs to $C^s(K)$ then the reverse inequality holds; namely, $d_{s:K}(\mu,\nu)\geq d_s(\mu^+,\nu^+)$.  
\end{proof}

By Lemma~\ref{lem:res_ext} we henceforth may interpret any such $\mu$ as its extension $\mu^+$, without loss of generality.

\subsection{{Examples  of Functions in The Classes \texorpdfstring{$C^s_{poly:C,r}([0,1]^d,\mathbb{R})$ and $C^s_{exp:C,r}([0,1]^d,\mathbb{R})$}{Cs-Poly and Cs-Exp}}}
\label{s:Functions_of_ControlledGrowth}

In several learning theory papers, especially in the kernel ridge regression literature e.g.~\cite{simon2023eigenlearning,barzilai2023generalization,tsigler2023benign,simon2023eigenlearning,cheng2024characterizing,cheng2024characterizingV2}, one often quantifies the \textit{learnability} of a target function in terms of some sort of decay/growth rates of its coefficients in an appropriate expansion; e.g.\ the decay of its coefficients in an eigenbasis associated to a kernel.  These decay/growth rates are often equivalent to the smoothness of a function%
\footnote{See e.g.~\citep[page 120-121]{atkinson2012spherical} for an example between the decay rate of the Laplacian eigenspectrum characterize the smoothness of the functions in the RKHS of radially symmetric kernels.}.  
Therefore, in a like spirit, we unpack the meaning of the smoothness condition in Assumption~\ref{def:sGrowthRate} which impacts the learning rates in Theorem~\ref{thrm:main} by giving examples of functions in the classes $C^s_{poly:C,r}([0,1]^d,\mathbb{R})$ and $C^s_{exp:C,r}([0,1]^d,\mathbb{R})$.

For brevity and transparency in our illustration, we consider the one-dimensional case.  In particular, this shows that the class is far from being void.
\begin{proposition}[{Functions with Polynomially/Exponentially Growing $C^s$-Norms on $[0,1]$}]
\label{prop:Examples_Functions_inPoly}
Fix $d\in \mathbb{N}_+$ and let $K$ be a non-empty regular compact subset of $\mathbb{R}^d$.  If $f:\mathbb{R}\to \mathbb{R}$ is real-analytic with power-series expansion at $0$ given by 
\begin{align*}
        f(x) 
    = 
        \sum_{i=0}^{\infty}\, 
            \frac{\beta_i\, x^i}{i!}
,
\end{align*}
and if there are $C,r>0$ such that
\begin{enumerate}
    \item[(i)] \textbf{Polynomial Growth:} $|\beta_i|\le Ce^{i\,r} \,\,(\forall i \in \mathbb{N})$, 
    then $f\in C^{\infty}_{poly:C,r}([0,1],\mathbb{R})$; or
    \item[(ii)] \textbf{Exponential Growth:} $|\beta_i|\le C(1+i)^r \,\,(\forall i \in \mathbb{N})$, 
    then $f\in C^{\infty}_{poly:C,r}([0,1],\mathbb{R})$.
\end{enumerate}
\end{proposition}
\begin{proof}
Since $f$ is real-analytic we may consider its Maclaurin-Taylor series expansion which, coincides with $\sum_{i=0}^{\infty}\, \frac{\beta_i\, x^i}{i!}$; meaning that for each $i\in \mathbb{N}$ we have $\beta_i=\partial^i \, f(0)$.
Therefore, standard analytic estimates and manipulations of the Maclaurin-Taylor series---see e.g.~\citep[page 173]{rudin1964principles}---yield
\begin{equation}
\label{eq:MLT_Estimates}
\begin{aligned}
    \max_{0\le x \le 1}
    \,
        \Big|
            \sum_{i=0}^{\infty}\, 
                \frac{\beta_i\, x^i}{i!}
            -
            f(x)
        \Big|
    \le &
        \frac{1}{(s+1)!}\,
        \sup_{0 \le x \le 1}
            \,
            \Big|
                \big(
                    \sum_{i=0}^{\infty}\, 
                    \frac{\beta_i\, x^i}{i!}
                \big)^{s+1}
                -
                \partial f^{s+1}(x)
            \Big|
\\
    \le &
        \frac{1}{(s+1)!}\,
         \beta_s(s+1)!
.
\end{aligned}
\end{equation}
If (i) holds, then the right-hand side of~\eqref{eq:MLT_Estimates} is bounded from above by $C\, (s+1)^r$ and $f\in C^{\infty}_{poly:C,r}([0,1],\mathbb{R})$.  If instead (ii) holds, then the right-hand side of~\eqref{eq:MLT_Estimates} is bounded from above by $C\, e^{s\,r}$ implying $f\in C^{\infty}_{exp:C,r}([0,1],\mathbb{R})$.
\end{proof}

\section{Proof of \texorpdfstring{Theorem~\ref{thrm:main}}{The Main Result}}
\label{s:Proof_of_MainResult}


Before entering into the formal proof of Theorem~\ref{thrm:main}, we overview its structure by overlooking its key steps and ideas on a high-level.
The first step in deriving our generalization bounds is to quantify the regularity of the transformer model as a function of its depth, number of attention heads, and norm of its weight matrices.  
By \textit{regularity}, we mean the number and size of the continuous partial derivatives admitted by the transformer. To quantify the size of the partial derivatives of the transformer we first remark that it is \textit{smooth}; that is, it admits continuous partial derivatives of all orders (see Theorem~\ref{prop:TransformerIsSmooth}).

We will uniformly bound the generalization capabilities of the class of transformers $\mathcal{T}\in \transformerclass$ by instead uniformly bounding the generalization of any $C^s$ functions on $\mathbb{R}^{{M}d}$ with $C^s$-norm at most equal to the largest $C^s$-norm in the class $\transformerclass$.  That is, we control the right-hand side of
\begin{equation}
\label{eq:risk_bound__decomposition}
        {
       \textstyle 
       \sup_{\mathcal{T}\in \transformerclass}\,
        \,
        \big|
            \mathcal{R}_t(\mathcal{T})-\mathcal{R}^{(N)}(\mathcal{T})
        \big|
    \le 
        \sup_{\hat{f}\in \mathcal{C}^s_R(\mathbb{R}^{{M}d})}
        \,
        \big|
            \mathcal{R}_t(\hat{f})-\mathcal{R}^{(N)}(\hat{f})
        \big|
            }
\end{equation}
where 
$R = C_{\ell, \transformerclass, K, s}$ as defined in \Cref{thrm:main},
describes the higher-order fluctuations of the ``difference'' between the target function $f^{\star}$ and any transformer $\transformer \in \transformerclass$, as quantified by the loss function $\ell$.  
Our first step is thus to bound $R$ by upper-bounding maximal size of the $s^{th}$ partial derivatives of all transformers $\transformer \in \transformerclass$. 
Explicit bounds are computed in Theorem~\ref{thm:TransformerBlockSobolevBound}, and their order estimates (as a function of $s$) are given in Theorem~\ref{thm:TransformerBlockSobolevBoundLevel}.  
Combing these estimates with the maximal $s^{th}$ partial derivatives of the loss and target function, via a Fa\'{a} di Bruno-type formula (in Theorem~\ref{lem:mdFDBestimate} or Lemma~\ref{thm:DerivativeBoundByLevel}), which is a multivariant higher-order chain rule, yields our estimate for $R$ in~\eqref{eq:risk_bound__decomposition}.

Now that we have bounded $R$, appearing in the supremum term in~\eqref{eq:risk_bound__decomposition}, it remains to translate this into a generalization bound.  We can do this by relating it to the so-called \textit{smooth Wasserstein distance} $d_s$
between the distribution of the Markov chain at time $\mu_t$ and its empirical distribution $\smash{\mu^{(N)}\eqdef 1/{N}\sum_{n=1}^N\, \delta_{X_n}}$ obtained by collecting samples up to time $N$.  The \textit{smooth Wasserstein distance} $d_s$, studied by \cite{Kloeckner_2020CounterCurse_ESAIM,Riekert_ConcentrationOfMeasure2022,hou2022instance}, is the integral probability metric (IPM)-type distance quantifying the distance between any two Borel probability measures $\mu,\nu$ on $\R^{{M}d}$ as the maximal distance which they can produced when tested on any function in $C^s_1(\R^{{M}d})$
\begin{align*}
    d_s(\mu,\nu) 
    \eqdef 
        \sup_{g\in \mathcal{C}^s_1(\mathbb{R}^{{M}d})}\, 
        \mathbb{E}_{X\sim \mu}[g(X)]
        -
        \mathbb{E}_{Y\sim \nu}[g(Y)]
.
\end{align*}
The right-hand side (RHS) of~\eqref{eq:risk_bound__decomposition} can be expressed as $R$ times the $d_s$ distance between the (true) distribution $\mu_t$ of the process $X_{\cdot}$ at time $t$ and the (empirical) distribution $\mu^{(N)}$ collected from samples
\begin{equation}
\label{eq:risk_bound__decomposition2}
        \mbox{RHS}\, 
        \eqref{eq:risk_bound__decomposition}
    \le 
        \sup_{\hat{f}\in C^s_R(\mathbb{R}^{{M}d})}\,\|\ell(\hat{f},f^{\star})\|_{C^s}
        \,
        d_s(\mu_t,\mu^{(N)})
.
\end{equation}
The $d_s$ distance between the process $X_{\cdot}$ at time $t$, i.e.~$\mu_t$, and the running empirical distribution $\mu^{(N)}$ can be accomplished in two steps.  First, we \textit{fast-forward time} and bound the $d_s$-distance between $\mu^{(N)}$ and the \textit{stationary distribution} $\mu_{\infty}$ of the data-generating Markov chain $X_{\cdot}$ (at time $t=\infty$).  We then \textit{rewind time} and bound the $d_s$-distance between the stationary distribution $\mu_{\infty}$ and the distribution $\mu_t$ of the Markov process up to time $t$; by setting up the i.i.d.\ concentration of measure results of~\cite{kloeckner2019effective,Riekert_ConcentrationOfMeasure2022}.  This last step is possible since our assumptions on $X_{\cdot}$ essentially guarantee that it has a finite (approximate) mixing time.

\subsection{The Formal Proof}

The proof of Theorem~\ref{thrm:main} will be largely broken down into two steps.  First, we derive our concentration of measure result for the empirical mean compared to the true mean general of an arbitrary $C^s$ function applied to a random input, where the $C^s$-norm of the $C^s$ function is at most $R\ge 0$ (in Subsection~\ref{s:Proofs__ss:ConcentrationProof}).  

Next, (in Subsection~\ref{s:Proofs__ss:ConcentrationProof}), we use the Fa\`{a} di Bruno-type results in Section~\ref{s:Compositions_and__FaadiBruno} to bound the maximal $C^s$ norm over the relevant class of transformer networks.  We do this by first individually bounding each of the $C^s$-norms of its constituent pieces, namely the multi-head attention layers, the SLP blocks with smooth activation functions, and then ultimately, we bound the $C^s$-norms of the composition of transformer blocks using the earlier Fa\`{a} di Bruno-type results.  

\begin{figure}[H]
     \centering
        \centerline{\includegraphics[width=.5\linewidth]{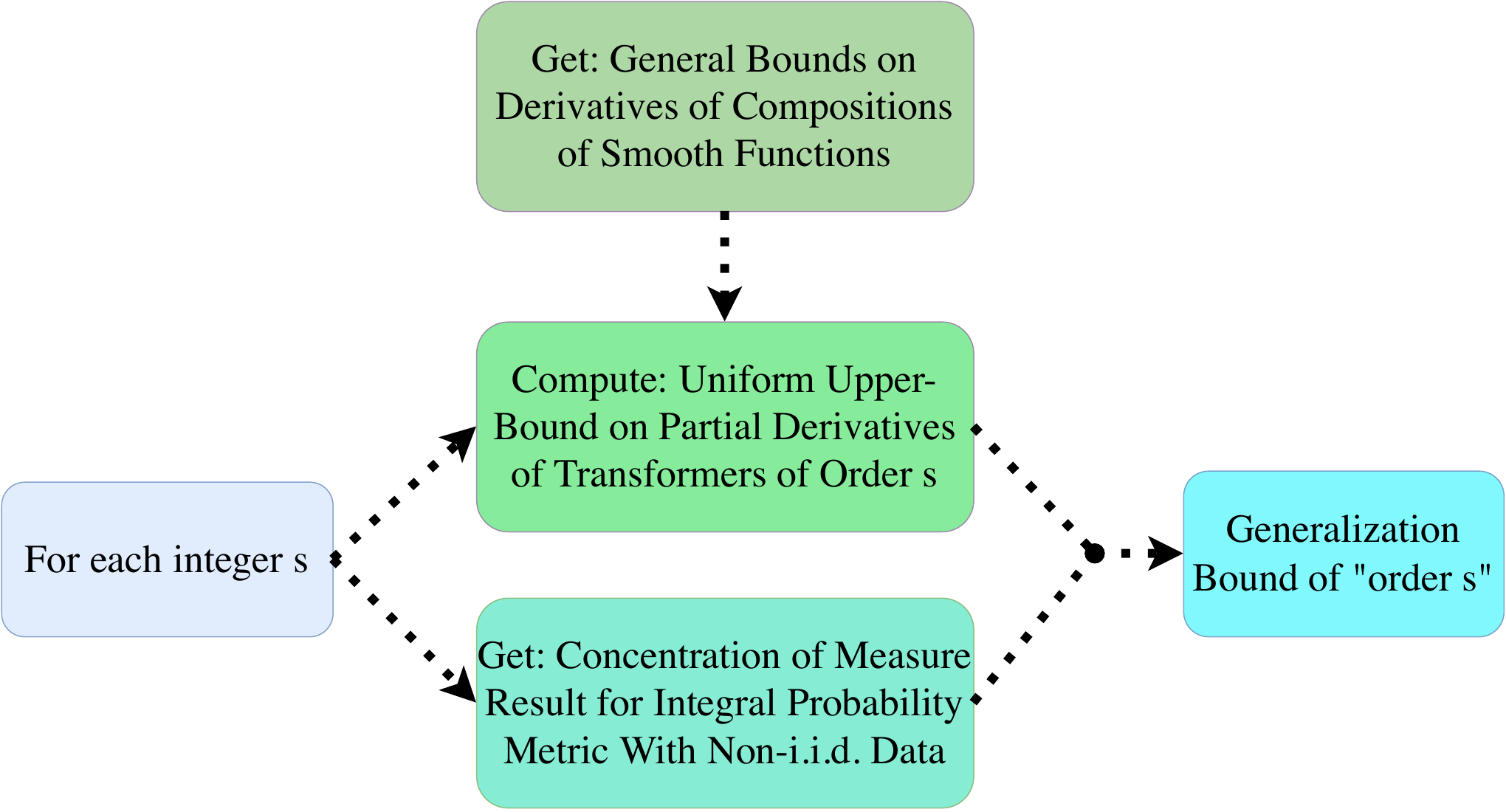}}
        \caption{Workflow of the proof technique used to derive Theorem~\ref{thrm:main}.}
        \label{fig:worklow}
\end{figure}

Our main result (Theorem~\ref{thrm:main}) is then then obtained upon merging these two sets of estimates.  The workflow which we use can be applied to derive generalization bounds for other machine learning, and is summarized in Figure~\ref{fig:worklow}.

\subsection{Step 0 - Bounds on the \texorpdfstring{$C^s$}{Cs} Regularity of Multivariate Composite Functions}
\label{s:Compositions_and__FaadiBruno}

In this section, we will derive a bound for the Sobolev norm of multivariate composite functions.

\subsubsection{Multivariate Fa\`a di Bruno formula revisited}
We begin by establishing notation and stating the multivariate Fa\`a di Bruno formula from \cite{ConstantineSavits96}.

\begin{theorem}[Multivariate Fa\`a di Bruno Formula, \cite{ConstantineSavits96}]\label{thm:FaaDiBruno}
    Let $n, m, k \in \N$, $\alpha \in \N^k$ with $\vert \alpha \vert = n$, and define 
    \begin{align*}
        h(x_1, \ldots, x_k) \eqdef f^{(1)}(g^{(1)}(x_1, \ldots, x_k), \ldots, g^{(m)}(x_1, \ldots, x_k)).
    \end{align*}
    Then, using the multivariate notation from \Cref{not:Multivariate},
    \begin{align*}
        D^\alpha h(x) 
        =
        \sum_{1 \leq |\beta| \leq n} (D^\beta f)(g(x))
        \sum_{\eta, \zeta \in \calP(\alpha,\beta)} \alpha! 
        \prod_{j=1}^{n} \frac{
            [D^{\zeta^{(j)}} g(x)]^{\eta^{(j)}}
        }{
            \eta^{(j)}!(\zeta^{(j)}!)^{|\eta^{(j)}|}
        }.
    \end{align*}
    where 
    \begin{multline*}
        \calP(\alpha, \beta) = \Big\{ \eta \eqdef (\eta^{(1)}, \ldots, \eta^{(n)}) \in (\N^m)^n, \zeta\eqdef (\zeta^{(1)}, \ldots, \zeta^{(n)}) \in (\N^k)^n : \\
         \exists j \leq m: \eta^{(i)} = 0, \zeta^{(i)} = 0 \text{ for } i < j, |\eta^{(i)}| > 0 \text{ for } i \geq j, \\
         \textstyle 0 \prec \zeta^{(j)} \prec \ldots \prec \zeta^{(n)}, \sum_{i=1}^{n} \eta^{(i)} = \beta \text{ and } \sum_{i=1}^{n} |\eta^{(i)}|\zeta^{(i)} = \alpha \Big\} .
    \end{multline*}
\end{theorem}
\begin{proof}
    See \cite{ConstantineSavits96}.
\end{proof}

\subsubsection{Universal Bounds}
\label{ss:UnivBounds}

\begin{theorem}\label{thm:DerivativeBoundByLevel}
In the notation of \Cref{thm:FaaDiBruno}, we have for a compact set $K \subseteq \R^k$ and an multi-index $\alpha \in \N^k$, $\vert \alpha \vert = n$,
\begin{align*}
        C^h_K(\alpha)
    \leq
        \max_{n' \in \myset{1}{n}}
        C^g_{g[K]}({n'})
        C^f_K({\cleq n})^{n'}
        \sum_{1 \leq |\beta| \leq n} 
        \sum_{\eta, \zeta \in \calP(\alpha,\beta)} \alpha! 
        \prod_{j=1}^{n} \frac{
            1
        }{
            \eta^{(j)}!(\zeta^{(j)}!)^{|\eta^{(j)}|}
        } 
\end{align*}
where $C^h_K(\cdot)$, $C^f_{g[K]}(\cdot)$, $C^g_K(\cdot)$ are defined as in \Cref{not:DerivativeBounds}. 
\end{theorem}
\begin{proof}
    Using \Cref{thm:FaaDiBruno},
    \begin{align*}
        C^h_K(\alpha)
    &\leq
    \sum_{1 \leq |\beta| \leq n} \Vert D^\beta f\Vert_{g[K]}
    \sum_{\eta, \zeta \in \calP(\alpha,\beta)} \alpha! 
    \prod_{j=1}^{n} \frac{
        \prod_{i=1}^{m}
        \Vert (D^{\zeta^{(j)}} g)_i \Vert_K^{\eta^{(j)}_i}
    }{
        \eta^{(j)}!(\zeta^{(j)}!)^{|\eta^{(j)}|}
    } 
    \\
    &\leq
    \sum_{1 \leq |\beta| \leq n} 
    C^g_{g[K]}({|\beta|})
    \sum_{\eta, \zeta \in \calP(\alpha,\beta)} \alpha! 
    \prod_{j=1}^{n} \frac{
        \prod_{i=1}^{m}
        C^f_K({\cleq n})^{\eta^{(j)}_i}
    }{
        \eta^{(j)}!(\zeta^{(j)}!)^{|\eta^{(j)}|}
    } 
    \\
    &\leq
    \sum_{1 \leq |\beta| \leq n} 
    C^g_{g[K]}({|\beta|})
    C^f_K({\cleq n})^{|\beta|}
    \sum_{\eta, \zeta \in \calP(\alpha,\beta)} \alpha! 
    \prod_{j=1}^{n} \frac{
        1
    }{
        \eta^{(j)}!(\zeta^{(j)}!)^{|\eta^{(j)}|}
    } 
    \\
    &\leq
    \max_{n' \in \myset{1}{n}}
    C^g_{g[K]}({n'})
    C^f_K({\cleq n})^{n'}
    \sum_{1 \leq |\beta| \leq n} 
    \sum_{\eta, \zeta \in \calP(\alpha,\beta)} \alpha! 
    \prod_{j=1}^{n} \frac{
        1
    }{
        \eta^{(j)}!(\zeta^{(j)}!)^{|\eta^{(j)}|}
    }.
    \end{align*}
\end{proof}

Next, we refine the strategy used in~\cite{hou2023instance} to convert our uniform risk-bound to a concentration of measure problem. Once done, the remainder of the proof will be to obtain bounds on the rate at which this measure concentrates.





\begin{lemma} 
\label{lem:mdFDBestimate}
For $\alpha\in\{1, \cdots, k\}^n$, it satisfies that 
\begin{align}
        \sum_{1 \leq |\beta| \leq n} 
        \sum_{\eta, \zeta \in \calP(\alpha,\beta)} \alpha! 
        \prod_{j=1}^{n} \frac{1}
        {
            \eta^{(j)}!(\zeta^{(j)}!)^{|\eta^{(j)}|}
        } 
   =  \Big[ \frac{ 2m |\alpha|}{e\ln |\alpha|} ( 1 + o(1) ) \Big]^{|\alpha|}   
   \notag 
\end{align} 
where $\calP(\alpha, \beta)$ is as defined in Theorem~\ref{thm:FaaDiBruno}. 
\end{lemma} 
\begin{proof} 
Consider functions  
\begin{align} 
 & g^{(i)}(x) = g^{(i)}(x_1, \cdots, x_d) \eqdef \exp\big( \sum_{j=1}^d x_j \big) : \, \mathbb{R}^d \to \mathbb{R}, \quad 
 i = 1, ..., 2m , \notag  \\ 
 &  f( g^{(1)}, \cdots , g^{(2m)} ) \eqdef \exp\big( \sum_{i=1}^{2m} g^{(i)} \big) : \,  
\mathbb{R}^{2m} \to \mathbb{R} , \notag 
\end{align} 
Since 
\begin{align} 
 \frac{\partial }{\partial g^{(i)}} f(g^{(1)}, \cdots , g^{(2m)}) = f(g^{(1)}, \cdots , g^{(2m)}) , \notag 
\end{align} 
it follows that 
\begin{align} 
&  (D^\beta f)(g^{(1)}(x), \cdots, g^{(2m)}(x)) = f (g^{(1)}(x), \cdots, g^{(2m)}(x)) , \quad 
\forall \, \beta \in \{1, \cdots, 2m \}^n . \notag  
\end{align}

Since 
 \begin{align} 
& \frac{\partial }{\partial x^j }  f(g^{(1)}(x_1, \cdots, x_k), \cdots , g^{(2m)}(x_1, \cdots, x_k))   \notag \\ 
 =& \sum_{i=1}^{2m} \frac{\partial }{\partial g^{(i)}}   f(g^{(1)}(x_1, \cdots, x_k), \cdots , g^{(2m)}(x_1, \cdots, x_k)) 
 \frac{\partial g^{(i)}(x_1, \cdots, x_k)}{\partial x_j}    \notag \\ 
  =& \sum_{i=1}^{2m}    f(g^{(1)}(x_1, \cdots, x_k), \cdots , g^{(2m)}(x_1, \cdots, x_k)) 
  g^{(i)}(x_1, \cdots, x_k) \notag   \\ 
= & \sum_{i=1}^{2m} \frac{\partial }{\partial g^{(i)}}   f(g^{(1)}(x), \cdots , g^{(2m)}(x)) 
 \frac{\partial g^{(i)}(x)}{\partial x_j}    \notag \\ 
  =&  f(g^{(1)}(x), \cdots , g^{(2m)}(x))  
  \sum_{i=1}^{2m} g^{(i)}(x)  \notag 
\end{align} 
and 
\begin{align} 
 &  \frac{\partial g^{(i)}(x_1, \cdots, x_k)}{\partial x_j} = g^{(i)}(x_1, \cdots, x_k) ,  \notag 
\end{align} 
we can show by the Fa\`{a} di Bruno formula that 
\allowdisplaybreaks
\begin{align} 
 & D^\alpha f(g^{(1)}(x), \cdots, g^{(2m)}(x)) \notag \\ 
    &= D^\alpha  \exp\big( \sum_{i=1}^{2m} g^{(i)}(x) \big) \notag \\
    &=  \sum \frac{|\alpha|!}{\gamma_1! (1!)^{\gamma_1} \cdots \gamma_{|\alpha|}! (|\alpha|!)^{\gamma_{|\alpha|}}} \big( D^{\gamma_1+\cdots + \gamma_{|\alpha|}}  \exp \big) \big( \sum_{i=1}^{2m} g^{(i)}(x) \big)
  \prod_{j=1}^{|\alpha|} \Big[ m^j (  \sum_{i=1}^{2m} g^{(i)}(x) ) \Big]^{\gamma_j} ,  
  \notag 
\end{align} 
where the summation on the right side of the last equality is over all $|\alpha|$-tuples $(\gamma_1, \cdots, \gamma_{|\alpha|})\geq 0$ such that $1\cdot \gamma_1 + 2\cdot \gamma_2 + \cdots + |\alpha|\cdot \gamma_{|\alpha|}=|\alpha|$.

By the multivariate Fa\`a di Bruno formula. For each $n=1$, ..., $s-1$ fixed, and for each $\alpha\in \{1, \cdots, k\}^n$, we have 
\begin{align} 
& D^\alpha f(g^{(1)}(x), \cdots, g^{(2m)}(x)) \notag \\ 
= &
        \sum_{1 \leq |\beta| \leq n} (D^\beta f)(g^{(1)}(x), \cdots, g^{(2m)}(x))
        \sum_{\eta, \zeta \in \calP(\alpha,\beta)} \alpha! 
        \prod_{j=1}^{n} \frac{
            [D^{\zeta^{(j)}} (g^{(1)}(x), \cdots, g^{(2m)}(x))]^{\eta^{(j)}}
        }{
            \eta^{(j)}!(\zeta^{(j)}!)^{|\eta^{(j)}|}
        } .  \notag 
\end{align} 
Taking $x = (x_1,\cdots, x_k) = 0$, we have 
\begin{align} 
 D^\alpha f(g^{(1)}(0), \cdots, g^{(2m)}(0)) 
  &=  \sum \frac{|\alpha|!}{\gamma_1! (1!)^{\gamma_1} \cdots \gamma_{|\alpha|}! (|\alpha|!)^{\gamma_{|\alpha|}}} 
  \exp(2m)
  \prod_{j=1}^{|\alpha|} ( 2 m )^{\gamma_j} 
  \notag \\ 
     (m^\beta f)(g^{(1)}(0), \cdots, g^{(2m)}(0)) &= f (g^{(1)}(0), \cdots, g^{(2m)}(0)) = \exp(2m) , \notag \\ 
   D^{\zeta^{(j)}} (g^{(1)}(x), \cdots, g^{(2m)}(x)) &= (1, \cdots, 1) . \notag 
\end{align} 
Substituting the above derivatives into the Fa\`a di Bruno formula, we obtain 
\begin{align} 
  \sum_{1 \leq |\beta| \leq n}  
        \sum_{\eta, \zeta \in \calP(\alpha,\beta)} \alpha! 
        \prod_{j=1}^{k} \frac{ 1 }
        {
            \eta^{(j)}!(\zeta^{(j)}!)^{|\eta^{(j)}|}
        } 
    =&  \sum \frac{|\alpha|!}{\gamma_1! (1!)^{\gamma_1} \cdots \gamma_{|\alpha|}! (|\alpha|!)^{\gamma_{|\alpha|}}} 
  \prod_{j=1}^{|\alpha|} ( 2 m )^{\gamma_j}    \notag \\ 
  \leq & 
  (2m)^{|\alpha|} 
  \sum \frac{|\alpha|!}{\gamma_1! (1!)^{\gamma_1} \cdots \gamma_{|\alpha|}! (|\alpha|!)^{\gamma_{|\alpha|}}} 
  \notag \\ 
  =& 
 (2m)^{|\alpha|}  \Big( \frac{|\alpha|}{e \ln |\alpha| } \Big)^{|\alpha|} (1+o(1))^{|\alpha|} , \notag 
\end{align} 
where the last equality follows from \cite[Theorem 2.1]{KO2022}, 
\begin{align} 
\sum \frac{|\alpha|!}{\gamma_1! (1!)^{\gamma_1} \cdots \gamma_{|\alpha|}! (|\alpha|!)^{\gamma_{|\alpha|}}} 
= 
\Big( \frac{|\alpha|}{e \ln |\alpha| } \Big)^{|\alpha|} (1+o(1))^{|\alpha|} 
.
  \notag 
\end{align} 
\end{proof}

\begin{corollary}[Level Specific $C^s$-Norm Bounds for Transformer Blocks]\label{cor:DerivativeBoundByLevel}
In the notation of \Cref{thm:DerivativeBoundByLevel}, it holds for $n \in \N, {n > 1}$ that
\begin{align*}
        C^h_K(\cleq n)
    \leq
        \max_{n' \in \myset{1}{n}}
        C^g_{g[K]}({n'})
        C^f_K({\cleq n})^{n'}
        \Big[ \frac{ 2m n}{e\ln n} ( 1 + o(1) ) \Big]^{n}.
\end{align*}
and if $C^f_K({\cleq n}) \geq 1$, 
\begin{align*}
        C^h_K(\cleq n)
    \leq
        C^g_{g[K]}({\cleq n})
        C^f_K({\cleq n})^{n}
        \Big[ \frac{ 2m n}{e\ln n} ( 1 + o(1) ) \Big]^{n}.
\end{align*}
\end{corollary}
\begin{proof}
    Follows directly from \Cref{thm:DerivativeBoundByLevel} and \Cref{lem:mdFDBestimate}.
\end{proof}

\subsubsection{Bounds in Derivative Type}
\label{ss:BoDerType}

The goal of this section is to bound the derivative of composite functions by grouping with respect to $\eqorder$, defined in \Cref{not:Order}.
\begin{theorem}\label{thm:DerivativeBoundByType}
In the notation of \Cref{thm:FaaDiBruno}, we have for a compact set $K \subseteq \R^k$ and an ordered multi-index $\alpha \in \Order^k_{n}$
\begin{align*}
C^h_K(\alpha)
    &\leq
    \alpha! \sum_{\beta \in \Order^m_{\leq n}}  N(\beta) C^f_{g[K]}(\beta)
    \sum_{\eta, \zeta \in \calP'(\alpha,\beta)} 
    \prod_{j=1}^{n} \frac{
        C^g_K(\order(\zeta^{(j)}))^{\vert \eta^{(j)} \vert}
    }{
        \eta^{(j)}!(\zeta^{(j)}!)^{|\eta^{(j)}|}
    },
\end{align*}
where $C^h_K(\cdot)$, $C^f_{g[K]}(\cdot)$, $C^g_K(\cdot)$ are defined as in \Cref{not:DerivativeBounds}; and 
\begin{multline*}
    \calP'(\alpha, \beta) = \Big\{ \eta \eqdef (\eta^{(1)}, \ldots, \eta^{(n)}) \in (\N^m)^n, \zeta\eqdef (\zeta^{(1)}, \ldots, \zeta^{(n)}) \in (\N^k)^n : \\
    \exists j \leq m: \eta^{(i)} = 0, \zeta^{(i)} = 0 \text{ for } i < j, |\eta^{(i)}| > 0 \text{ for } i \geq j, \\
    \textstyle 0 < \zeta^{(j)} \vartriangleleft \ldots \vartriangleleft \zeta^{(n)}, \sum_{i=1}^{n} \eta^{(i)} = \beta \text{ and } \sum_{i=1}^{n} |\eta^{(i)}|\zeta^{(i)} = \alpha \Big\} ,
\end{multline*}
where $\alpha \vartriangleleft \beta$ for $\alpha, \beta \in \N^k$ if $\vert \alpha \vert \leq \vert \beta \vert$ and  $\alpha \neq \beta$.
\end{theorem}
\begin{proof}
We have for $\alpha \in \Order^k_{n}$
\begin{align*}
    C^h_K(\alpha)
    &\leq
    \max_{\gamma \eqorder \alpha}
    \sum_{1 \leq |\beta| \leq n} \Vert D^\beta f\Vert_{g[K]}
    \sum_{\eta, \zeta \in \calP(\gamma,\beta)} \alpha! 
    \prod_{j=1}^{n} \frac{
        \prod_{i=1}^{m}
        \Vert (D^{\zeta^{(j)}} g)_i \Vert_K^{\eta^{(j)}_i}
    }{
        \eta^{(j)}!(\zeta^{(j)}!)^{|\eta^{(j)}|}
    } \\
    &\leq
    \sum_{1 \leq |\beta| \leq n} C^f_{g[K]}(\order (\beta))
    \max_{\gamma \eqorder \alpha}
    \sum_{\eta, \zeta \in \calP(\gamma,\beta)} \alpha! 
    \prod_{j=1}^{n} \frac{
        C^g_K(\order(\zeta^{(j)}))^{\vert \eta^{(j)} \vert}
    }{
        \eta^{(j)}!(\zeta^{(j)}!)^{|\eta^{(j)}|}
    }.
\end{align*}
Then 
\begin{align*}
    \mySet{\eta, (\order(\zeta^{(1)}), ..., \order(\zeta^{(n)}))}{ (\eta, \zeta) \in \calP'(\alpha, \beta)}
\end{align*}
is invariant in $\alpha$ with respect to $\eqorder$ and thus
\begin{align*}
    C^h_K(\alpha)
    &\leq
    \sum_{1 \leq |\beta| \leq n} C^f_{g[K]}(\order (\beta))
    \sum_{\eta, \zeta \in \calP'(\alpha,\beta)} \alpha! 
    \prod_{j=1}^{n} \frac{
        C^g_K(\order(\zeta^{(j)}))^{\vert \eta^{(j)} \vert}
    }{
        \eta^{(j)}!(\zeta^{(j)}!)^{|\eta^{(j)}|}
    }.
\end{align*}
Further, notice that
\begin{align*}
    \mySet{((\vert \eta^{(1)}\vert, \eta^{(1)}!), \ldots, (\vert \eta^{(n)}\vert, \eta^{(n)}!)), \zeta}{ (\eta, \zeta) \in \calP'(\alpha, \beta)}
\end{align*}
is invariant in $\beta$ with respect to $\eqorder$ and the assertion follows.
\end{proof}

\begin{corollary}\label{cor:DerivativeBoundAffine}
    In the notation of \Cref{thm:DerivativeBoundByType}, if $f$ is affine-linear, 
    \begin{align*}
        C^h_K(\alpha)
        &\leq
        m \alpha! C^f_{g[K]}(e_1)
        C^g_K(\alpha),
    \end{align*}
    where $C^f_{g[K]}(e_1)$ is the maximum weight of the matrix representing $f$. 
\end{corollary}
\begin{proof}
    \Cref{thm:DerivativeBoundByType} yields
    \begin{align*}
        C^h_K(\alpha)
        &\leq
        m \alpha! C^f_{g[K]}(e_1)
        \sum_{\eta, \zeta \in \calP'(\alpha,\beta)} 
        \prod_{j=1}^{n} \frac{
            C^g_K(\order(\zeta^{(j)}))^{\vert \eta^{(j)} \vert}
        }{
            \eta^{(j)}!(\zeta^{(j)}!)^{|\eta^{(j)}|}
        },
    \end{align*}
    and since $\calP'(\alpha, e_1) = \{(0, ..., 0, e_1), (0, ..., 0, \alpha) \}$ the result follows.
\end{proof}

\subsection{Step 1 - Concentration of Measure\texorpdfstring{ - Bounding the Right-Hand Side of~\eqref{eq:risk_bound__decomposition}}{}}
\label{s:Proofs__ss:ConcentrationProof}

We are now ready to derive our main concentration of measure results used to derive our risk-bound.  This corresponds to bounding term~\eqref{eq:risk_bound__decomposition} by controlling the integral probability term $d_s(\mu_t,\mu^{(N)})$ in~\eqref{eq:risk_bound__decomposition2}, with high probability, where the randomness is due to the randomness of the empirical measure $\mu^{(N)}$.

{We state the next bound in the case where the input space if $\mathbb{R}^d$.  Note that the results hold for any other input dimension, such as $Md$, simply by relabeling $d\gets Md$.  Thus, it applies to the \textit{finite-dimensional} Markovian lifts $X_{\cdot}^M$ of data-generating processes $X_{\cdot}$, where $M\in \mathbb{N}_+$, by relabeling.
Therefore, for notational minimality, we chose to label the input dimension $d$ and not $dM$.}

\begin{proposition}[Excess Risk-Bound]
\label{prop:Concentation} 
Under Assumption~\ref{ass:Comp_Support} and either~\ref{ass:contractivity} or~\ref{ass:MC}, 
let $f^{\star}:\mathbb{R}^d\to \mathbb{R}^D$, $\ell:\mathbb{R}^{2D}\to \mathbb{R}$, and $R,r>0$ be a such that the composite map $\mathbb{R}^d\ni x\mapsto \ell(f^{\star}(x),f(x))$ belongs to $\mathcal{C}_{R}^s(\mathbb{R}^d)$ for all $f\in C^2_r(\mathbb{R}^d)$.
Then, there exists some $\kappa\in(0,1)$ depending only on the Markov chain $X_{\cdot}$ and some $t_0\in \mathbb{N}_0$ such that for each $t_0\le N < t \le \infty$, each ``rate-to-constant-tradeoff parameter'' $s\in \mathbb{N}_+$, and every ``confidence level'' $\delta \in (0,1]$ the following
\begin{align*}
        \sup_{g\in \mathcal{C}^s_R(\mathbb{R}^d)}
        \,
        \frac{
            \big|
                \mathcal{R}_{\max\{t,N\}}(g)-\mathcal{R}^{(N)}(g)
            \big|
        }{R}
    \lesssim
                \kappa^t\,
            +
                \frac{
                    \sqrt{2\ln(1/\delta)}
                }{
                    N^{1/2}
                }
                +\,
                \begin{cases}
                   \frac{
                       \log\big(
                            c\, N
                        \big)^{d/(2s + 1)}
                   }{
                        c\, N^{1/2}
                   } 
                    & \mbox{ if } 
                    d<2s\\
                   \frac{
                       \log\big(
                            c\, N
                        \big)
                   }{
                        c\, N^{1/2}
                   } 
                    & \mbox{ if } d=2s\\
                   \frac{
                       \log\big(
                            c\, N
                        \big)^{d-2s+(s/d)}
                   }{
                       c_2\,  N^{s/d}
                   } 
                    & \mbox{ if } d=2s\\
                \end{cases}
\end{align*}
holds with probability at least $1-\delta$; where $0<\kappa <1$, and we use the notation $\kappa^{\infty}\eqdef 0$.
\end{proposition}

\begin{proof}[{Proof of Proposition~\ref{prop:Concentation}}]
By hypothesis, $\tilde{f}\in C^s_r(\mathbb{R}^d)$ the induced map
\begin{equation}
\label{eq:composite_to_concentrate}
\begin{aligned}
f:  \mathbb{R}^{d}& \to \mathbb{R} \\
x & \mapsto f(x)\eqdef \ell\big(f^{\star}(x),\tilde{f}(x)\big)
\end{aligned}
\end{equation}
belongs to $C_R^s(\mathbb{R}^d)$.

\paragraph{Conversion to a Concentration of Measure Problem.}
Denote the empirical (random) measure associated with the samples $\{(X_n,Y_n)\}_{n=1}^N$ by $\mu^{(N)}\eqdef \frac1{N}\sum_{n=1}^N\, \delta_{(X_n,Y_n)}$.
Note that the generalization bound is $0$ for any constant function; therefore, we consider the bound over $C_R^s(\mathbb{R}^d)\setminus \operatorname{Lip}_0$ where $\operatorname{Lip}_0$ denotes the set of constant functions from $\mathbb{R}^d$ to $\mathbb{R}$.  Note the bijection between $C_R^s(\mathbb{R}^d)\setminus \operatorname{Lip}_0$ and $C_1^s(\mathbb{R}^d)\setminus \operatorname{Lip}_0$ given by $f\mapsto \frac1{\max\{1,\|f\|_{C^s(\mathbb{R}^d)}\}} f$.  Therefore, we compute
\allowdisplaybreaks
\begin{align}
\nonumber
        \big|
            \mathcal{R}_{
            t
            }(f)-\mathcal{R}^{(N)}(g)
        \big|
    \le &
        \sup_{g\in \mathcal{C}^s_R(\mathbb{R}^d)}
        \,
        \big|
            \mathcal{R}_{
            t
            }(f)-\mathcal{R}^{(N)}(g)
        \big|
\\
\nonumber
    \le &
        R\,
        \sup_{g\in \mathcal{C}^s_1(\mathbb{R}^d)}
        \,
        \big|
            \mathcal{R}_{
            t
            }(g)-\mathcal{R}^{(N)}(g)
        \big|
    \\
\nonumber
    \le &
        R\, d_{C^s}(\mu_{\max\{t,N\}},\mu^{(N)})
    \\
    \label{eq:TB_Bounded}
    \le &
        R\, \big(
            \underbrace{
                d_{C^s}(\mu_{
                t
                },\bar{\mu})
            }_{\term{term_II}}
        +
            \underbrace{
                d_{C^s}(\bar{\mu},\mu^{(N)})
            }_{\term{term_I}}
    \big)
.
\end{align}
Next, we bound terms $\text{(I)}$ and $\text{(II)}$.  

\paragraph{Bounding Term~\eqref{term_II}.}
If Assumption~\ref{ass:contractivity} holds then: for every $t \in \mathbb{N}_+$ each $x,\tilde{x}\in \mathcal{X}$ we have
\begin{align*}
        \mathcal{W}_1\big(P^t(x,\cdot),P^t(\tilde{x},\cdot)\big)
    \le 
        \kappa^t\, \mathcal{W}_1(\delta_x,\delta_{\tilde{x}})
    =
    \kappa^t\, \|x-\tilde{x}\|
    .
\end{align*}
If, instead, we operate under the log-Sobolev Assumption~\ref{ass:MC}, then~\cite[Theorem 1.3]{BobkovGotze_1999_JFA} can be applied to $\bar{\mu}$ and $P(x,\cdot)$ for each $x\in \mathcal{X}$, implying that the transport inequalities hold: for each $\nu\in\mathcal{P}(\mathcal{X})$ and each $\tilde{\mu}\in \{\bar{\mu},\mu_0\}\cup \{P(x,\cdot)\}_{x\in \mathcal{X}}$
\begin{equation}
\label{eq:term_II__transport_inequalities_B}
    \mathcal{W}_1(\tilde{\mu},\nu)^2
\le 
    2C^2\, \operatorname{KL}(\nu|\tilde{\mu})
\end{equation}
where we recall the definition of the Kullback–Leibler divergence $\operatorname{KL}(\nu|\mu)\eqdef \mathbb{E}_{X\sim \nu}[\log(\frac{d\nu}{d\mu}(X))]$.  Thus,~\eqref{eq:term_II__transport_inequalities_B} implies that the following exponential contractility property of the Markov kernel: there exists some $\kappa \in (0,1)$ such that for each $x,\tilde{x}\in \mathcal{X}$ and every $t\in \mathbb{N}_+$
\begin{equation}
\label{eq:exp_contract}
    \mathcal{W}_1\big(P^t(x,\cdot),P^t(\tilde{x},\cdot)\big)
    \le 
    \kappa^t\, \|x-\tilde{x}\|
    .
\end{equation}
Furthermore,~\eqref{eq:exp_contract} implies that 
the conditions for~\citep[Theorem 1.5]{Riekert_ConcentrationOfMeasure2022} are met; whence, for every $\varepsilon\ge 0$ and each $N\in \mathbb{N}$ the following holds with probability at-least $
1-\exp\big(\frac{-N\,\varepsilon^2(1-\kappa)^2}{2C^2}\big)
$
\begin{equation}
\label{eq:term_II__concentration}
        \eqref{term_II}
    =
        d_s(\bar{\mu},\mu^{(N)})
    \le 
        \mathbb{E}\big[d_s(\bar{\mu},\mu^{(N)})\big]
        +
        \varepsilon
,
\end{equation}
for some $C>0$.
Upon setting $\varepsilon \eqdef 
\frac{C\sqrt{2\ln(1/\delta)}}{\sqrt{N\,(1-\kappa^2)}}$,~\eqref{eq:term_II__concentration} implies that: for every $N\in \mathbb{N}$ and each $\delta \in (0,1]$ the following holds with probability at-least $1-\delta$
\begin{equation}
\label{eq:term_II__concentration___mod}
        \eqref{term_II}
    =
        d_s(\bar{\mu},\mu^{(N)})
    \le 
        \underbrace{
            \mathbb{E}\big[d_s(\bar{\mu},\mu^{(N)})\big]
        }_{\term{term_III}}
        +
        \frac{C\sqrt{2\ln(1/\delta)}}{\sqrt{N\,(1-\kappa^2)}}
.
\end{equation}
It remains to bound the expectation term~\eqref{term_III} in~\eqref{eq:term_II__concentration___mod} to bound term~\eqref{term_II}.

Under the exponential moment assumption~\ref{ass:Exp_Moment}, we have that
\begin{equation}
\label{eq:exp_moment_repackaged}
\mathbb{E}_{X\sim P(x,\cdot)}[e^{\beta|X|}-1]\le \gamma\, (e^{\beta|x|}-1) + (C-1+\gamma)
.
\end{equation}
Therefore \citep[Proposition 1.3]{Riekert_ConcentrationOfMeasure2022}, implies that $\sup_{t\in \mathbb{N}_0}\, \mathbb{E}[e^{\beta|X_t|}-1]<\infty$.  Whence, \citep[Assumption 2]{Riekert_ConcentrationOfMeasure2022} holds with Young function $\Phi(x)=\frac1{\max\{1,\sup_{t\in \mathbb{N}_+\}}\mathbb{E}[e^{\beta|X_t}-1]}\,(e^{\beta|X_t|}-1)$; namely, $\sup_{t\in \mathbb{N}_0}\, \mathbb{E}[\Phi(|X_t|)]\le 1$.  Consequentially,~\citep[Theorem 1.1]{Riekert_ConcentrationOfMeasure2022} applies from which we conclude that there is some $t_0\in \mathbb{N}_+$ such that for all $N\ge t_0$
\begin{equation}
\label{eq:Bound_on_Expectation}
    \eqref{term_III}
    =
    \mathbb{E}\big[d_s(\bar{\mu},\mu^{(N)})\big]
    \lesssim
    \,
    \log\big(
        (1-\kappa)\, N
    \big)^{s}
    \,
        \begin{cases}
           \frac{
               \log\big(
                    (1-\kappa)\, N
                \big)^{d/(2s) + 1}
           }{
                (1-\kappa)^{1/2}\, N^{1/2}
           } 
            & \mbox{ if } 1=d<2s\\
           \frac{
               \log\big(
                    (1-\kappa)\, N
                \big)
           }{
                (1-\kappa)^{1/2}\,N^{1/2}
           } 
            & \mbox{ if } d=2s\\
           \frac{
               \log\big(
                    (1-\kappa)\, N
                \big)^{d-2s+s/d}
           }{
                (1-\kappa)^{s/d}\, N^{s/d}
           } 
            & \mbox{ if } d=2s\\
        \end{cases}
.
\end{equation}

Combining the order estimate of~\eqref{term_III} in~\eqref{eq:Bound_on_Expectation} with the estimate in~\eqref{eq:term_II__concentration___mod} implies that: for every $N\ge t_0$ and each $\delta \in (0,1]$ we have
\begin{equation}
\label{eq:term_II__concentration___done}
        \eqref{term_II}
    =
        d_s(\bar{\mu},\mu^{(N)})
    \lesssim
        \,
        \frac{
            \sqrt{2\ln(1/\delta)}
        }{
            N^{1/2}
        }
        +\,
            \begin{cases}
               \frac{
                   \log\big(
                        c\, N
                    \big)^{d/(2s) + 1}
               }{
                    c\, N^{1/2}
               } 
                & \mbox{ if } 1=d<2s\\
               \frac{
                   \log\big(
                        c\, N
                    \big)
               }{
                    c\, N^{1/2}
               } 
                & \mbox{ if } d=2s\\
               \frac{
                   \log\big(
                        c\, N
                    \big)^{d-2s+s/d}
               }{
                   c_2\,  N^{s/d}
               } 
                & \mbox{ if } d=2s\\
            \end{cases}
\end{equation}
where $c\eqdef (1-\kappa), c_2\eqdef c^{s/d}\in (0,1)$, and $\lesssim$ suppresses the absolute constant $\max\{1,C\}>0$.

\paragraph{Bounding Term~\eqref{term_I}.}
Next, we bound~\eqref{term_I} by computing
\allowdisplaybreaks
\begin{align}
\nonumber
    \eqref{term_I} = 
        d_{C^s}(\mu_{
        t
        },\bar{\mu})
    \eqdef &
        \sup_{g\in \mathcal{C}^s_1(\mathbb{R}^d)}\, \mu_{
        t
        }[g]-\bar{\mu}[g]
\\
    \le & 
    \label{eq:inclusion_CsLip}
        \sup_{g\in \operatorname{Lip}_1(\mathbb{R}^d)}\, \mu_{
        t
        }[g]-\bar{\mu}[g]
\\
    = & 
    \label{eq:KRDuality}
        \mathcal{W}_1(\mu_{
        t
        },\bar{\mu})
\\
    = & 
        \mathcal{W}_1(P^t\mu_0,\bar{\mu})
\\
    = & 
    \label{eq:stationarity}
        \mathcal{W}_1(P^t\mu_0,P^t\bar{\mu})
\\
    \le & 
    \label{eq:exp_contract2}
        \kappa^t\, 
        \mathcal{W}_1(\mu_0,\bar{\mu})
    \eqdef 
        \kappa^t\,C
\end{align}
where~\eqref{eq:inclusion_CsLip} held by definition of the MMD $d_{C^s}$ and by the inclusion of $\mathcal{C}_1^s(\mathbb{R}^d)\subset \operatorname{Lip}_1(\mathbb{R}^d)$,~\eqref{eq:KRDuality} held by Kantorovich duality (see \citep[Theorem 5.10]{VillaniBook_2009}),~\eqref{eq:stationarity} held since $\bar{\mu}$ is the stationary probability measure for the Markov chain $X_{\cdot}$, it is invariant to the action of the Markov kernel, and~\eqref{eq:exp_contract2} followed from \citep[Corollary 21]{olivera2019density} since we deduced the exponential contractility property~\eqref{eq:exp_contract} of the Markov kernel.  Note that $C\eqdef \mathcal{W}_1(\mu_0,\bar{\mu})$ is a constant depending only on the initial and stationary distributions of the Markov chain.

\paragraph{Conclusion.}
Incorporating the estimates for~\eqref{term_I} and~\eqref{term_II} into the right-hand side of~\eqref{eq:TB_Bounded} implies that: for every $t,N\ge t_0$, $s\in \mathbb{N}_+$, and each $\delta\in (0,1]$ the following holds
\begin{align*}
\resizebox{1\linewidth}{!}{$
\begin{aligned}
        \sup_{g\in \mathcal{C}^s_R(\mathbb{R}^d)}
        \,
        \frac{
            \big|
                \mathcal{R}_{\max\{t,N\}}(g)-\mathcal{R}^{(N)}(g)
            \big|
        }{R}
    \lesssim &
                I_{t<\infty}\,\kappa^t\,
            +
                \frac{
                    \sqrt{2\ln(1/\delta)}
                }{
                    N^{1/2}
                }
                +\,
                \begin{cases}
                   \frac{
                       \log\big(
                            c\, N
                        \big)^{d/(2s + 1)}
                   }{
                        c\, N^{1/2}
                   } 
                    & \mbox{ if } 1=d<2s\\
                   \frac{
                       \log\big(
                            c\, N
                        \big)
                   }{
                        c\, N^{1/2}
                   } 
                    & \mbox{ if } d=2s\\
                   \frac{
                       \log\big(
                            c\, N
                        \big)^{d-2s+(s/d)}
                   }{
                       c_2\,  N^{s/d}
                   } 
                    & \mbox{ if } d=2s\\
                \end{cases}
\end{aligned}
$}
\end{align*}
with probability at-least $1-\delta$; where $c\eqdef (1-\kappa)$ and $\kappa\in (0,1)$; where $I_{t<\infty}k^{\infty}\eqdef 0$ if $t=\infty$.
\end{proof}

\subsection{Step 2 (A) - Bounding the \texorpdfstring{$C^s$}{Cs} Regularity of Transformer Building Blocks}
\label{sec:RegularityTransfomerBuildingBlocks}

We begin by the following simple remark, that if the activation function used to defined the transformer is smooth, then so must the entire transformer model.
\begin{proposition}[Transformers with Smooth Activation Functions are Smooth]
\label{prop:TransformerIsSmooth}
    Fix $\transformerclass$, as in \Cref{defn:Transformer}, then every transformer $\transformer \in \transformerclass$ is smooth.
\end{proposition}

\begin{proof}[{Theorem~\ref{prop:TransformerIsSmooth}}]
    The smoothness of $\attention$ follows directly from the smoothness of $\softmax$, which immediately implies smoothness of $\multihead$ since the operators used for its definition are smooth. Furthermore, the $\layernorm$ is smooth due to its smooth and lower-bounded denominator and the activation function $\sigma$ is smooth by definition, therefore we conclude that $\tblock \in \tblockclass$ is smooth for every $\tblockclass$ as in \Cref{defn:TransformerBlock} and we obtain smoothness of $\transformer \in \transformerclass$ as a consequence.
\end{proof}

\subsubsection{The Softmax Function}
\begin{lemma}[Representation of higher-order $\softmax$ derivatives]\label{lem:SoftmaxDerivativeRepresentation}
For $F \in \N$ and
\begin{align*}
    \smax:
            \R^F 
        \to     
            \R^F,
    \quad 
            x 
        \mapsto 
            \left(
                    {\exp(x_i)}
                /
                    {\textstyle\sum_{j = 0}^{F-1} \exp(x_j)}
            \right)_{i=1}^{F}.
\end{align*}
there exists for any multi-index $\alpha \in \N^F$ and $m \in \myset{1}{F}$ indicators $(a_{i, j}^k)_{i, j \in I(\alpha)}^{k \in \myset{1}{\vert \alpha \vert!}} \subseteq \{0, 1\}$ such that 
\begin{align}
    \smax^{(\alpha)} (x_m) = \sum_{k = 1}^{\vert \alpha \vert!} \smax(x_{m}) \prod_{i,j \in I(\alpha)} (a^k_{i,j}  - \smax(x_{j})), \label{id:SoftmaxDerivativeShape}
\end{align}
where $I(\alpha)\eqdef \{(i, j): i = 1, \ldots, F, j = 1, \ldots, \alpha_i\}$.
\end{lemma}
\begin{proof}
    For $\vert \alpha \vert = 0$, we have $n \in \myset{1}{F}$ s.t. $\alpha_n = 1$, therfore 
    \begin{align*}
        \smax^{(\alpha)} (x_m) = \frac{\partial \smax}{\partial x_{n}} (x_{m}) = \smax(x_{m})\left( \delta_{m n} - \smax(x_{n})\right),
    \end{align*}
    which is of the form \eqref{id:SoftmaxDerivativeShape}. Now, let $\alpha \in \N^F$ arbitrary, therefore, by defining $\alpha' \in \N^F$ by $\alpha'_i \eqdef \alpha_i$ for $i \neq n$ and $\alpha'_n \eqdef \alpha_n - 1$ (w.l.o.g. $\alpha_n > 0$). We have 
    \begin{align*}
        \smax^{(\alpha)} (x_m) 
        & = \frac{\partial \smax^{(\alpha')}}{\partial x_{n}} (x_{m}) \\
        & = \frac{\partial}{\partial x_{n}} \sum_{k = 1}^{\vert \alpha' \vert!} \smax(x_{m}) \prod_{i,j \in I(\alpha')}(a^{\prime k}_{i,j}  - \smax(x_{j})).
    \end{align*}
    Since for any $k$
    \begin{align*}
        & \frac{\partial}{\partial x_{n}} \smax(x_{m}) \prod_{i,j \in I(\alpha')}(a^{\prime k}_{i,j}  - \smax(x_{j})) \\
        & = \smax(x_{m})\left( \delta_{m n} - \smax(x_{n})\right)  \prod_{i,j \in I(\alpha')} (a^{\prime k}_{i,j}  - \smax(x_{j})) \\ 
        & \quad +  \smax(x_{m}) \sum_{i',j' \in I(\alpha')}  
         - \smax(x_{j'})\left(\delta_{j',n} - \smax(x_{n})\right) 
        \prod_{\substack{i,j \in I(\alpha) \\ (i,j) \neq (i',j')}}(a^{\prime k}_{i,j}  - \smax(x_{j})),
    \end{align*}
    we can define $(a_{i, j}^k)_{i, j \in I(\alpha)}^{k \in \myset{1}{\vert \alpha' \vert + 1}} \subseteq \{0, 1\}$ such that
    \begin{align*}
        \frac{\partial}{\partial x_{n}} \smax(x_{m}) \prod_{i,j \in I(\alpha')}(a^{\prime k}_{i,j}  - \smax(x_{j}))
        = \sum_{k=1}^{\vert \alpha \vert}\smax(x_{m}) \prod_{i,j \in I(\alpha)} (a^{k}_{i,j}  - \smax(x_{j})). 
    \end{align*}
    Since $\vert \alpha \vert ! = \vert \alpha \vert \cdot \vert \alpha' \vert !$, this concludes the proof.
\end{proof}

\begin{lemma}[Bound of higher-order $\softmax$ derivatives]
    With \Cref{not:DerivativeBounds}, it holds for any set $K \in \R^k, k\in \N$ and any $\alpha \in \Order^k_{< \infty}$ that
    \begin{align*}
        C^{\smax}(\alpha) \leq \vert \alpha \vert!.
    \end{align*}
\end{lemma}
\begin{proof}
    This is a direct consequence of the representation in \Cref{lem:SoftmaxDerivativeRepresentation} together with $\Vert \smax \Vert = 1$.
\end{proof}

\subsubsection{The Multi-Head Self-Attention Mechanism}

\begin{lemma}[Bound of Dot product]\label{lem:Dotp}
    In the notation of \Cref{defn:MultiHead} and for $m \in \myset{1}{M}$
    \begin{align*}
        \dotp_m(\,\cdot\,; Q, K): 
                \R^{M\idim} 
            \longrightarrow 
                \R^{M}
            , \quad
                x
            \longmapsto
                \langle
                    Qx_m, Kx_j
                \rangle_{j = 0}^{M} 
    \end{align*}
    we have using \Cref{not:DerivativeBounds}
    \begin{enumerate}
    \item $ C^{\dotp_m}_K(e_1) \leq  2 \idim \kdim \Vert K \Vert C^{Q}C^{K}$, where $C^Q \eqdef \max_{i, i' \in \myset{1}{\kdim}\times \myset{1}{\idim}} \vert Q_{i, i'}\vert$, $C^K$ analogously, and $\Vert K \Vert \eqdef \max_{x \in K} \Vert x \Vert$. Additionally,
    \item  $ C^{\dotp_m}_K( \alpha ) \leq  2 \kdim  C^{Q}C^{K}$, for $\vert \alpha \vert = 2$, and 
    \item $C^{\dotp_m}_K(\alpha) = 0$ for $\vert \alpha \vert > 2$.
    \end{enumerate}
    Since all bounds are not dependent on $m$ we write $C^{\dotp}$ short for $C^{\dotp_m}$.
\end{lemma}
\begin{proof}%
    \begin{enumerate}[wide]%
    \item  Let $l = (l_1, l_2) \in \myset{1}{M} \times \myset{1}{\idim}$. 
    Assume $l_1 = m$. If $j \neq m$, then
        \begin{align*}
                D^{e_l}\dotp_m(x; Q, K)_j 
            =   
                D^{e_l}
                \sum_{i=1}^{\kdim}
                 (Kx_j)_i
                \sum_{i'=1}^{\idim}
                    Q_{i, i'}(x_m)_{i'}
            =   
                \sum_{i=1}^{\kdim}
                \bigg(\sum_{i'=1}^{\idim}
                    K_{i, i'}(x_j)_{i'}
                    \bigg)
                Q_{i, l_2},
        \end{align*}
    implying
    \begin{align}\label{id:DotPSimple}
            \Vert D^{e_l}\dotp_m(x; Q, K) \Vert 
        \leq 
            \Vert K \Vert
            \sum_{i=1}^{\kdim}
            Q_{i, l_2}
            \sum_{i'=1}^{\idim}
                K_{i, i'}
        \leq 
            \idim \kdim
            \Vert K \Vert
            C^{Q}C^{K}
        .
    \end{align}
    If $j = m$, 
        \begin{align*}
                D^{e_l}\dotp_m(x; Q, K)_j 
            &=   
                D^{e_l}
                \sum_{i=1}^{\kdim}
                \bigg(\sum_{i'=1}^{\idim}
                    K_{i, i'}(x_m)_{i'}
                    \bigg)
                \bigg(\sum_{i'=1}^{\idim}
                    Q_{i, i'}(x_m)_{i'}
                    \bigg) \\ 
            &=   
                \sum_{i=1}^{\kdim}
                    \bigg(
                    K_{i, l_2}
                \sum_{i'=1}^{\idim}
                    Q_{i, i'}(x_m)_{i'}
                +
                    Q_{i, l_2}
                \sum_{i'=1}^{\idim}
                    K_{i, i'}(x_m)_{i'}
                    \bigg)
        \end{align*}
   therefore implying     
    \begin{align*}
            \Vert D^{e_l}\dotp_m(x; Q, K) \Vert 
        \leq 
            2 \idim \kdim
            \Vert K \Vert
            C^{Q}C^{K}
        .
    \end{align*}
    If $l_1 \neq m$ then for $j \neq l_1$, $D^{e_l}\dotp_m(x; Q, K)_j = 0$, for $j = l_1$
    \begin{align*}
                D^{e_l}\dotp_m(x; Q, K)_j 
            =   
                D^{e_l}
                \sum_{i=1}^{\kdim}
                (Qx_m)_i
                \sum_{i'=1}^{\idim}
                K_{i, i'}(x_j)_{i'}
            =   
                \sum_{i=1}^{\kdim}
                \bigg(\sum_{i'=1}^{\idim}
                    Q_{i, i'}(x_m)_{i'}
                    \bigg)
                K_{i, l_2},
        \end{align*}
        and we obtain \eqref{id:DotPSimple} analogously.
    \item If $l_1 = m$ and $j \neq m$
    \begin{align*}
        D^{e_l}\bigg(
            (x_m)_{l_2}
                \sum_{i=1}^{\kdim}
                \bigg(\sum_{i'=1}^{\idim}
                    k_{i, i'}(x_j)_{i'}
                    \bigg)
                q_{i, l_2}
            \bigg) = 0,
    \end{align*}
    implying $\Vert D^{2e_l}\dotp_m(x; Q, K) \Vert \leq 0$, what analogously holds for $l_1 \neq m$. However, for $l_1 = m$ and  $j = m$
    \begin{align*}
        D^{e_l}\bigg(
                \sum_{i=1}^{\kdim}
                    K_{i, l_2}
                \sum_{i'=1}^{\idim}
                    Q_{i, i'}(x_m)_{i'}
                +
                    Q_{i, l_2}
                \sum_{i'=1}^{\idim}
                    K_{i, i'}(x_m)_{i'} \bigg)
         = 
                \sum_{i=1}^{\kdim}
                    K_{i, l_2}
                    Q_{i, l_2}
                +
                    Q_{i, l_2}
                    K_{i, l_2}
    \end{align*}
    we have
    \begin{align*}
            \Vert D^{2e_l}\dotp_m(x; Q, K) \Vert 
        \leq 
            2 
            \kdim 
            C^{Q}C^{K}
        .
    \end{align*}
    \item  Let $l' = (l'_1, l'_2) \in \myset{1}{M} \times \myset{1}{\idim}$. Assume $l_1 = m$, $j \neq m$. If $l'_1 \neq j$, $D^{e_l + e_{l'}}\dotp_m(x; Q, K)_j = 0$. For $l'_1 = j$ follows $D^{e_l + e_{l'}}\dotp_m(x; Q, K)_j =  \sum_{i = 1}^{\kdim} K_{i, l'_2}Q_{i, l_2}$. If $l_1 = m$, $j \neq m$, we have  $D^{e_l + e_{l'}}\dotp_m(x; Q, K)_j = 0$ in the case that $l'_1 \neq m$, and for $l'_1 \neq m$ we obtain
    \begin{align*}
        D^{e_l + e_{l'}}\dotp_m(x; Q, K)_j = 
                \sum_{i=1}^{\kdim}
                    K_{i, l_2}
                    Q_{i, l'_2}
                +
                    Q_{i, l_2}
                    K_{i, l'_2}.
    \end{align*}
    This means, we can use the bound 
    \begin{align*}
            \Vert D^{e_l + e_{l'}}\dotp_m(x; Q, K) \Vert 
        \leq 
            2 
            \kdim 
            C^{Q}C^{K}
        .
    \end{align*}
\end{enumerate}
\end{proof}
\begin{lemma}[Bound of Self-Attention for Derivative Type]\label{lem:Attention}
    Using the notation of  \Cref{not:DerivativeBounds}, \Cref{defn:MultiHead} and \Cref{lem:Dotp}, it holds that
   \begin{align*}
    C_K^{
        \attention
    }(\alpha)
    \leq    
    \idim
    M
    C^V
    \left(
        \Vert K \Vert
        C_K^{
            \smax \circ \dotp
        }(\alpha)
    +
        \sum_{l=1}^{M\idim}
        \alpha_l
        C_K^{
            \smax \circ \dotp
        }(\alpha - e_l)
    \right)
\end{align*} 
where 
\begin{align}
C^{\smax \circ \dotp}_K(\alpha)
    &\leq
    \alpha! \sum_{\beta \in \Order^{M}_{\leq n}}  N(\beta) C^{\smax}_{\dotp[K]}(\beta)
    \sum_{\eta, \zeta \in \calP'(\alpha,\beta)} 
    \prod_{j=1}^{n} \frac{
        C^{\dotp}_K(\order(\zeta^{(j)}))^{\vert \eta^{(j)} \vert}
    }{
        \eta^{(j)}!(\zeta^{(j)}!)^{|\eta^{(j)}|}
    }. \label{id:SmaxDpCompositeBound}
\end{align}
\end{lemma}
\begin{proof}
    Fix $\alpha \in \N^k$, and note that
    \begin{align*}
        \Vert
            D^{\alpha}
            \attention(
                x; Q, K, V 
            ) 
        \Vert
        \leq
        \max_{m\in \myset{1}{M}}
        \max_{i\in \myset{0}{\vdim}}
        \Vert
            D^{\alpha}
            \attention(
                x; Q, K, V 
            )_{m,i}
        \Vert
    \end{align*}
    and
    \begin{align*}
         \Vert
        D^{\alpha}
        \attention(&
            x; Q, K, V 
        )_{m,i}
        \Vert \\
        &\leq
        \sum_{j=1}^M
        \sum_{i'=0}^{\idim}
        \Vert
            D^{\alpha}
            \smax \circ \dotp(x; Q, K)_j
            V_{i, i'}(x_j)_{i'}
        \Vert \\
        &\leq
        \idim M
        \max_{j\in \myset{1}{M}}
        \max_{i'\in \myset{0}{\idim}}
        \Vert
            D^{\alpha}
            \smax \circ \dotp(x; Q, K)_j
            V_{i, i'}(x_j)_{i'}
        \Vert.
    \end{align*}
    Due to the extended Leibnitz rule \cite{Hardy06}, we have
    \begin{multline*}
        \Vert
            D^{\alpha}
            \smax \circ \dotp(x; Q, K)_j
            V_{i, i'}(x_j)_{i'}
        \Vert
        \\ \leq
        \Vert
            D^{\alpha}
            \smax \circ \dotp(x; Q, K)_j
            V_{i, i'}(x_j)_{i'}
        \Vert
        +
        \sum_{l=1}^{M\idim}
        V_{i, i'}
        \alpha_l
        \Vert
            D^{\alpha - e_l}
            \smax \circ \dotp(x; Q, K)_j
        \Vert.
    \end{multline*}
    \Cref{id:SmaxDpCompositeBound} follows directly from \Cref{thm:DerivativeBoundByType}.
\end{proof}

\begin{corollary}[Bound of Self-Attention for Derivative Level]\label{cor:Attention}
    Using the setting of \Cref{lem:Attention}, for $n \in \N$,
   \begin{align}
    C_K^{
        \attention
    }(n)
    \leq    
    \idim
    M
    C^V
    C_K^{
        \smax \circ \dotp
    }(\cleq n)
    \left(
        \Vert K \Vert
    +
        n \idim M
    \right)
    \label{id:AttenstionBoundByLevel}
\end{align} 
where 
\begin{align}
C^{\smax \circ \dotp}_K(\cleq n)
    &\leq
        C^{\smax}_{\dotp[K]}(\cleq n)
        C^{\dotp}_K({\cleq n})^{n}
        \Big[ \frac{ 2 nM}{e\ln n} ( 1 + o(1) ) \Big]^{n}.
    \label{id:SmaxDpCompositeBoundII}
\end{align}
\end{corollary}
\begin{proof}
    \Cref{id:AttenstionBoundByLevel} follows directly from \Cref{lem:Attention}; and \eqref{id:SmaxDpCompositeBoundII} is a consequence of \Cref{cor:DerivativeBoundByLevel}.
\end{proof}

\begin{corollary}[Bound of Multi-head Self-Attention]\label{cor:MultiHead}
    In the notation of \Cref{defn:MultiHead}, \Cref{thm:DerivativeBoundByType} and \Cref{lem:Dotp} it holds that 
    \begin{align*}
        C^{\multihead}_K(\alpha)
        &\leq
        \alpha! 
        \vdim 
        C^{W}
        C^{\attention}_K(\alpha)
    \end{align*}
    where 
    \begin{align*}
        C^{\attention}_K \eqdef  \max_{h \in \myset{1}{H}}C^{\attention(\, \cdot\,; Q^{(h)}, K^{(h)}, V^{(h)})}_K , \quad C^{W} \eqdef \max_{h \in \myset{1}{H}} W^{(h)}.
    \end{align*}
In particular, we have the following order estimate
\begin{align*}
        C^{\multihead}_K(\le n)
    \in
        \mathcal{O}\biggl(
            M^2
            \Vert K \Vert
            \|W\|
            \|V\|
                (c_{\idim,\kdim}\Vert K \Vert \|Q\|\|K\|)^{n}
            \,\,
                n^2
                 \Big(\frac{n}{e}\Big)^{2n}
                C_n^n
        \biggr)
    .
\end{align*}
\end{corollary}
\begin{proof}
    From \Cref{cor:DerivativeBoundAffine} and \Cref{lem:Attention} we directly obtain
   \begin{equation}
    \label{eq:precise_bound_MHG}
\begin{aligned}
        C^{\multihead}_K(\alpha)
    \leq & 
         n! 
            \vdim 
            C^{W}
            \idim
        M
        C^V
        n!
            (2 \idim \kdim \Vert K \Vert C^{Q}C^{K})^{n}
\\
    & \times 
        \left(
            \Vert K \Vert
        +
            n \idim M
        \right)
            \Big[ \frac{ 2 nM}{e\ln n} ( 1 + o(1) ) \Big]^{n}
    .
\end{aligned}
\end{equation}
   Applying Stirling's approximation, we have that 
   \begin{equation}
    \label{eq:orderbound}
        C^{\multihead}_K(\alpha)
    \in
        \mathcal{O}\biggl(
            M^2
            \Vert K \Vert
            \|W\|
            \|V\|
                (c_{\idim,\kdim}\Vert K \Vert \|Q\|\|K\|)^{n}
            \,\,
                n^2
                 \Big(\frac{n}{e}\Big)^{2n}
                C_n^n
        \biggr)
,
   \end{equation}
where $C_n\eqdef \frac{ 2 n M}{e\ln n} ( 1 + o(1) )$ and $c_{\idim,\kdim}\eqdef 2 \idim \kdim$.
\end{proof}

\subsubsection{The Activation Functions}\label{sec:Activation}
\begin{lemma}[Derivatives of $\softplus$]\label{lem:Softplus}
    For 
    \begin{align*}
        \softplus: \R \to \R, \quad x \mapsto \ln (1 + \exp(\cdot))
    \end{align*}
    it holds 
    \begin{align*}
        \softplus^{(1)}(x) = \sigmoid(x) \eqdef 1 / (1 + \exp(-x))
    \end{align*}
    and for $n \in \N$
    \begin{align*}
        \softplus^{(n + 1)}(x) = \sigmoid^{(n)}(x) = \sum_{k=0}^n(-1)^{n+k}k!S_{n,k}\sigmoid(x)(1-\sigmoid(x))^k,
    \end{align*}
    where $S_{n,k}$ are the Stirling numbers of the second kind, $S_{n,k} \eqdef \frac{1}{k!} \sum_{j=0}^{k} (-1)^{k-j} \binom{k}{j} j^n$.
\end{lemma}
\begin{proof}
We start with Fa\`a di Bruno's formula,
\begin{align*}
    \frac{d^n}{dt^n}\sigmoid(x)=\frac{d^n}{dx^n}\frac{1}{f(x)}=\sum_{k=0}^n(-1)^kk!f^{-(k+1)}(x)B_{n,k}(f(x)),
\end{align*}
where \(f(x) \eqdef 1 + \exp(-x)\) and \(B_{n,k}(f(x))\) denotes the Bell polynomials evaluated on $f(x)$. Next, we know the \(k\)-th derivative of \(f(x)\) is given by
\begin{align*}
   \frac{d^k}{dt^k}f(x)=(1-k)_k+ke^{-x}.
\end{align*}
Now, using the definition of the Bell polynomials \(B_{n,k}(f(t))\), we have
\begin{align*}
    B_{n,k}(f(x))=(-1)^nS_{n,k}e^{-kx},
\end{align*}
where \(S_{n,k}\) represents the Stirling numbers of the second kind. Substituting the expression for \(B_{n,k}(f(x))\) into the derivative of \(\sigmoid(x)\), we obtain
\begin{align*}
    \frac{d^n}{dx^n}\sigmoid(x)=\sum_{k=0}^n(-1)^{n+k}k!S_{n,k}\sigmoid(x)(1-\sigmoid(x))^k.
\end{align*}
\end{proof}

\begin{corollary}\label{cor:Softplus}
In the setting of \Cref{lem:Softplus}, for $n \in \N$,
    \begin{align*}
        C^{\softplus}(n) \leq         \sum_{k=0}^n 
        \frac{k^k k!S_{n,k}}{(k+1)^{k+1}}.
    \end{align*}
\end{corollary}
\begin{proof}
    For $k \in \N$ and $x \in [0,1]$, we have
    \begin{align*}
        f^k(x) \eqdef x(1-x)^k, \quad 
        (f^k)'(x) = (1-(k+1)x)(1-x)^{k-1};
    \end{align*}
    which amounts to $(f^k)'(x) = 0$ at $1/(k+1)$, i.e. 
    \begin{align*}
        \max_{x \in [0,1]}f(x) = \frac{1}{k+1}\left(
            \frac{k}{k+1}
        \right)^k
        = 
            \frac{k^k}{(k+1)^{k+1}}
    .
    \end{align*}
\end{proof}

\begin{lemma}[Derivatives of $\operatorname{GELU}$]
\label{lem:GELU}
For 
    \begin{align*}
        \operatorname{GELU}: \R \to \R, \quad x \mapsto x \Phi(x) , 
    \end{align*}
it holds 
\begin{align}
       & \operatorname{GELU}'(x) = \Phi(x) + x\phi(x), \notag \\ 
        & \operatorname{GELU}^{(n)}(x) = n \varphi^{(n-2)}(x) + x \phi^{(n-1)}(x), \quad n\geq 2 ,  
        \label{GELUn}
    \end{align}
with $\phi(x)=\frac{1}{\sqrt{2\pi}}e^{-x^2/2}$, $\Phi(x)=\int_{-\infty}^x \phi(u) du$. 
The $n$-th derivative of $\phi(x)$ is given by 
\begin{align*}
 \frac{d^n}{d x^n} \phi(x) = \frac{1}{\sqrt{2\pi}} e^{-x^2/2} \Big[ \sum_{k=0}^{\lfloor n / 2 \rfloor} \binom{n}{2k} 2^{k} 
  \frac{\Gamma(\frac{2k+1}{2})}{\Gamma(\frac{1}{2})} x^{n-2k} \Big] . \notag  
\end{align*}
\end{lemma}

\begin{proof} 
By induction we can show that \eqref{GELUn} holds. 
And the representation of the $n$-th derivative of $\phi(x)$ follows from~\cite{Oliveira2012RepresentationOT}. 
\end{proof} 

\begin{corollary}\label{cor:GELU}
    For $n \geq 2$, it holds in the setting of \Cref{lem:GELU}
    \begin{align*}
        C^{\operatorname{GELU}}(n) \leq \frac{1}{\sqrt{2\pi}\Gamma(\frac{1}{2})} \left( n a_{n-2} b_{n-2} + a_{n-1} c_{n-1}
        \right),
    \end{align*}   
    where 
    \begin{align*}
        a_n \eqdef \sum_{k=0}^{\lfloor n / 2 \rfloor} \binom{n}{2k} 2^{k} 
        \Gamma\left(\frac{2k+1}{2}\right), \quad 
        b_n \eqdef \max_{x \in \R} e^{-\frac{x^2}{2}} \sum_{k=0}^{\lfloor n / 2 \rfloor} x^{n-k} , \quad 
        c_n \eqdef \max_{x \in \R} e^{-\frac{x^2}{2}} x \sum_{k=0}^{\lfloor n / 2 \rfloor} x^{n-k}.
    \end{align*}
\end{corollary}

\begin{lemma}[Derivatives of $\tanh$]\label{lem:tanh}

For $\tanh: \mathbb{R} \mapsto \mathbb{R}$, the $n$-th derivatives have the representation  
\begin{align} 
& \frac{d^n}{d x^n} \tanh{x} = C_n(\tanh{x}) ,\notag \\ 
& C_n(z) = (-2)^n (z+1) \sum_{k=0}^n \frac{k!}{2^k} \binom{n}{k} (z-1)^k , \quad n\geq 1 . \notag 
\end{align} 
\end{lemma}

\begin{proof} 
See~\cite{Boyadzhiev2007}. 
\end{proof} 
\begin{corollary}\label{cor:tanh}
    In the setting of \Cref{lem:tanh}, for $n \in \N$, $C^{\tanh}(n) = \max_{z \in [-1, 1]} C_n(z)$.
\end{corollary}

\begin{lemma}[Derivatives of $\operatorname{SWISH}$]\label{lem:SWISH}
For 
    \begin{align*}
        \operatorname{SWISH}: \R \to \R, \quad x \mapsto \frac{x}{1+e^{-x}} 
    \end{align*}
    it holds for $n\geq 1$ 
    \begin{align} 
    \frac{d^n}{d x^n} \operatorname{SWISH}(x) 
    = n \sum_{k=1}^{n} (-1)^{k-1} (k-1)! S_{n, k} \sigmoid^k(x) 
    + x \sum_{k=1}^{n+1} (-1)^{k-1} (k-1)! S_{n+1, k} \sigmoid^k(x) , \notag  
    \end{align} 
where $S_{n,k}$ are Stirling numbers of the second kind, i.e., 
\begin{align} 
 S_{n,k} = \frac{1}{k!} \sum_{j=0}^k (-1)^j \binom{k}{j} (k-j)^n . \notag
\end{align} 
\end{lemma}

\begin{proof} 
By induction, we can show that 
\begin{align} 
\frac{d^n}{d x^n} \operatorname{SWISH}(x) = n \sigmoid^{(n-1)}(x) + x \sigmoid^{(n)}(x) , \quad n \geq 1 . \notag 
\end{align}  
By~\cite[Theorem 2]{MINAI1993845}, the derivatives of the sigmod function can be represented as  
\begin{align}
\sigmoid^{(n)}(x) = \sum_{k=1}^{n+1} (-1)^{k-1} (k-1)! S_{n+1, k} \sigmoid^k(x) , \ n\geq 1 . \notag 
\end{align} 
Combining the above two equations, we obtain the general form of the $n$-th derivative of the SWISH function.  
\end{proof}

\subsubsection{The Layer Norm}
\begin{lemma}[Bound of the Layer Norm for Derivative Type]\label{lem:LayerNorm}
    Fix $k \in \N$, $\beta\in \mathbb{R}^k$, $\gamma \in \mathbb{R}$, and $w\in [0,1]$. For the layer norm, given by 
    \begin{gather*}
        \layernorm: \R^k \rightarrow \R^k,\quad  x \mapsto  
        \gamma\, 
        f(x)
        g\circ\Sigma(x) + \beta; 
    \\
        f: \R^k \rightarrow \R^k,\quad  x \mapsto  x-M(x);
    \qquad
        g: \R \rightarrow \R, \quad  u \mapsto  \frac{1}{\sqrt{1 + u}};
    \\
        M: \R^k \rightarrow \R, \quad  x \mapsto  \frac{w}{k} \sum_{i=1}^k\,  x_i;
    \qquad
        \Sigma: \R^k \rightarrow \R, \quad  x \mapsto  \frac{w}{k}\, \sum_{i=1}^k\,  (x_i-M(x))^2;
    \end{gather*}
    holds for a compact symmetric set $K$ (using \Cref{not:DerivativeBounds})
    \begin{multline*}
        C^{\layernorm}_K(\alpha) \leq  
        \alpha! \gamma
        \sum_{m = 1}^{n=1} \frac{(2m+1)!! }{2^{2m}}
        \bigg(
        \sum_{\substack{\alpha' \leq \alpha\\ \vert \alpha' \vert = n-1}} 
            \sum_{\eta, \zeta \in \calP'(\alpha', m)} 
            \prod_{j=1}^{n} \frac{
                C^\Sigma_K(\order(\zeta^{(j)}))^{\vert \eta^{(j)} \vert}
            }{
                \eta^{(j)}!(\zeta^{(j)}!)^{|\eta^{(j)}|}
            }
    \\ 
    +
            \sum_{\eta, \zeta \in \calP'(\alpha, m)} 
            \prod_{j=1}^{n} \frac{
                C^\Sigma_K(\order(\zeta^{(j)}))^{\vert \eta^{(j)} \vert}
            }{
                \eta^{(j)}!(\zeta^{(j)}!)^{|\eta^{(j)}|}
            }
        \bigg),
    \end{multline*}
    where $C^\Sigma_K(\alpha) = 2w \Vert K \Vert$ for $\vert \alpha \vert = 1$, $C^\Sigma_K(\alpha) = 2w$ for $\vert \alpha \vert = 2$, and $C^\Sigma_K(\alpha) = 0$ otherwise.
\end{lemma}
\begin{proof}
Note that 
\begin{align*} 
    g^{(n)}(x) = (-1)^n \frac{(2n+1)!}{n!2^{2n}} (1+x)^{-\frac{1}{2}-n},
\end{align*}
implying $C^g_K(n) \leq (2n+1)!! 2^{-2n}$, $!!$ denoting the double factorial. We have further $C^{f}_K(\alpha) \leq \mathbbm{1}_{\vert \alpha \vert = 1}$ and a direct computation yields
\begin{align*}
    C^\Sigma_K(\alpha) \leq \begin{cases}
        2w\Vert K \Vert 
        &
        \text{ for } \vert \alpha \vert = 1
        \\
        2w
        &
        \text{ for } \vert \alpha \vert = 2
        \\
        0
        &
        \text{ else. } 
    \end{cases}
\end{align*}
By \Cref{thm:DerivativeBoundByType}, 
\begin{align*}
    C^{g\circ \Sigma}_K(\alpha)
    &\leq
    \alpha! \sum_{m = 1}^n C^g_{\Sigma[K]}(m)
    \sum_{\eta, \zeta \in \calP'(\alpha, m)} 
    \prod_{j=1}^{n} \frac{
        C^\Sigma_K(\order(\zeta^{(j)}))^{\vert \eta^{(j)} \vert}
    }{
        \eta^{(j)}!(\zeta^{(j)}!)^{|\eta^{(j)}|}
    }.
\end{align*}
According to the general multivariate Leibnitz rule, it holds that 
\begin{align*}
D^{\alpha}(f \cdot (g \circ \Sigma)) 
= \sum_{\beta \leq \alpha}  \frac{\alpha!}{\beta!(\alpha - \beta)!} D^{\beta}f \cdot D^{\alpha - \beta}(g\circ \Sigma) 
\end{align*}
which implies
\begin{align*}
        C^{\layernorm}_K (\alpha)
    \leq  
        C^{g\circ \Sigma}_K(\alpha) \Vert K \Vert
        +
        \sum_{\beta \leq \alpha, \vert \beta \vert = 1}  \frac{\alpha!}{(\alpha - \beta)!} 
            C^{g\circ \Sigma}_K(\alpha - \beta). 
\end{align*}
\end{proof}

\begin{corollary}[Bound of the Layer Norm for Derivative Level]\label{cor:LayerNorm}
    In the setting of \Cref{lem:LayerNorm}, it holds that
    \begin{align*}
        C^{\layernorm}_K(\cleq n) \leq 2w \Vert K \Vert(2n+1)!! 2^{-2n} (\Vert K \Vert + kn) 
        \Big[ \frac{ 2 n}{e\ln n} ( 1 + o(1) ) \Big]^{n}.
    \end{align*}
Furthermore, we have the asymptotic estimate 
\begin{align*}
        C^{g\circ \Sigma}_K(\cleq n)
    \in 
        \mathcal{O}\Big(
            w \|K\|
            \,
            n^{1/2}
            \big(
                    \frac{
                        n^{5/2}
                    }{
                        e^{3/4}\ln(n)
                    }
                \,
                (
                    1 + o(1)
                )
            \big)^n
        \Big)
.
\end{align*}
\end{corollary}
\begin{proof}
    Analogue to the proof of \Cref{lem:LayerNorm},
\begin{align*}
        C^{\layernorm}_K (\cleq n)
    \leq  
        \Vert K \Vert C^{g\circ \Sigma}_K(\cleq n) 
        +
        k  n 
            C^{g\circ \Sigma}_K(\cleq n-1) 
    \leq  
        (\Vert K \Vert + kn) C^{g\circ \Sigma}_K(\cleq n), 
\end{align*}
where we can use \Cref{cor:DerivativeBoundByLevel} to bound
\begin{align}
\nonumber
    C^{g\circ \Sigma}_K(\cleq n)
    &\leq
    C^g_{\Sigma[K]}(\cleq n)
        C^\Sigma_K({\cleq n})^{n}
        \Big[ \frac{ 2 n}{e\ln n} ( 1 + o(1) ) \Big]^{n}
    \\
\label{eq:reducemeplease}
    &\leq 2w \Vert K \Vert(2n+1)!! 2^{-2n}
        \Big[ \frac{ 2 n}{e\ln n} ( 1 + o(1) ) \Big]^{n}.
\end{align}
Since $2n+1$ is odd, for each $n\in \mathbb{N}_+$, then sterling approximation \textit{for double factorial} yields the asymptotic
\begin{equation}
\label{eq:stirling_redux}
        (2n+1)!! 
    \in 
        \mathcal{O}\Big(
            \sqrt{2 n}\left(\frac{n}{e}\right)^{n/2}
        \Big)
.
\end{equation}
Merging~\eqref{eq:stirling_redux} with the right-hand side of~\eqref{eq:reducemeplease} yields
\begin{align*}
        C^{g\circ \Sigma}_K(\cleq n)
    \in 
        \mathcal{O}\Big(
            w \|K\|
            \,
            n^{1/2}
            \big(
                    \frac{
                        n^{5/2}
                    }{
                        e^{3/4}\ln(n)
                    }
                \,
                (
                    1 + o(1)
                )
            \big)^n
        \Big)
.
\end{align*}
\end{proof}

\subsubsection{The Multilayer Perceptron (Feedforward Neural Network) with Skip Connection}
\begin{definition}[Single-Layer Feedforward Neural Network with Skip Connection]\label{defn:ff}
Fix a non-affine activation function $\sigma\in C^{\infty}(\R)$ and dimensions $\idim, \ldim, \odim \in \N$.
A feedforward neural network is a map $\neuraln:\mathbb{R}^{\idim}\to \mathbb{R}^{\odim}$ represented for each $x\in \mathbb{R}^{\idim}$ by
\begin{align}\label{eq:FF}
        \neuraln (x) 
    & \eqdef  
        B^{(1)} x + B^{(2)} \big(\sigma\bullet (A\, x + a)\big)
\end{align}
for
$A\in \mathbb{R}^{\ldim \times \idim}$, $a\in \mathbb{R}^{\ldim}$, 
$B^{(1)} \in \mathbb{R}^{\odim \times \idim}$, and $B^{(2)} \in \mathbb{R}^{\odim \times \ldim}$.
\end{definition}
\begin{lemma}[Bound of Neural Networks for Derivative Type]\label{lem:FeedForward}
    In the notation of \Cref{not:DerivativeBounds}, \Cref{lem:Dotp}, and \Cref{defn:ff}, it holds that
    \begin{align*}
        C^{\neuraln}_K(\alpha)
        &\leq
            C^{B^{(1)}}\mathbbm{1}_{\vert \alpha \vert = 1}
        +
            \ldim (\alpha!)^2 C^{B^{(2)}}
            \sum_{m=1}^n  C^\sigma_{{h}[K]}(m)
        \cdot
        (C^{A})^m
        \sum_{\eta, \zeta \in \calP'(\alpha,m)} 
        \prod_{j=1}^{n} \frac{
            \mathbbm{1}_{\vert \zeta^{(j)} \vert \leq 1}
        }{
            \eta^{(j)}!
        },
    \end{align*}
    where $h[K]$ is defined as the image of $h(x) \eqdef Ax + a$ on $K$. 
\end{lemma}
\begin{proof}
    Write $\neuraln(x) =  B^{(1)} x + B^{(2)}((g_i(x))_{i=1}^{\ldim})$, where for $i \in \myset{1}{\ldim}$
    \begin{align*}
        g_i(x) \eqdef \sigma((Ax + a)_i).
    \end{align*}
    If we define $h_i(x) \eqdef h(x)_i$, we follow with \Cref{thm:DerivativeBoundByType}
    \begin{align*}
    C^{g_i}_K(\alpha)
        &\leq
        \alpha! \sum_{m=1}^n  C^\sigma_{{h}[K]}(m)
        \cdot
        (C^{A})^m
        \sum_{\eta, \zeta \in \calP'(\alpha,m)} 
        \prod_{j=1}^{n} \frac{
            \mathbbm{1}_{\vert \zeta^{(j)} \vert \leq 1}
        }{
            \eta^{(j)}!
        }
        \eqdef
    C^{g}_K(\alpha)
        ,
    \end{align*}
    and due to the component wise application of the activation function it holds that \begin{align*}
        \Vert D^\alpha \max_{i \in \myset{1}{\ldim}} g_i(x) \Vert_K = \max_{i \in \myset{1}{\ldim}} C^{g_i}_K(\alpha) \leq C^{g}_K(\alpha).
    \end{align*}
    Using \Cref{cor:DerivativeBoundAffine}, we obtain
    \begin{align*}
        C^{\neuraln}_K(\alpha)
        &\leq
            C^{B^{(1)}}\mathbbm{1}_{\vert \alpha \vert = 1}
        +
            \ldim \alpha! C^{B^{(2)}}
            C^g_K(\alpha).
    \end{align*}
\end{proof}

\begin{corollary}[Bound of Neural Networks for Derivative Level]\label{cor:FeedForward}
    In the setting of \Cref{lem:FeedForward},
    \begin{align*}
        C^{\neuraln}_K(\cleq n)
        &\leq
            C^{B^{(1)}}
        +
            \ldim n! C^{B^{(2)}}
        C^\sigma_{{h}[K]}({\cleq n})
        (C^{A})^{n}
        \Big[ \frac{ 2 n}{e\ln n} ( 1 + o(1) ) \Big]^{n}.
    \end{align*}
If, moreover, $K=[-M_1,M_2]^{\idim}$ then 
\begin{align*}
        C^{\neuraln}_K(\cleq n)
    \in 
        \mathcal{O}
        \Big(
                \|B^{(1)}\|_{\infty}
            +
                \|B^{(2)}\|_{\infty}
                \|A\|_{\infty}^n
                \|\sigma\|_{n:
                        \operatorname{Ball}(a,\sqrt{\idim |M_1+M_2|})
                    }
                \operatorname{Width}(\neuraln)
                \,
                n^{1/2}\big(\frac{n}{e}\big)^n
                C_n^n
        \Big)
\end{align*}
\end{corollary}
\begin{proof}
Arguing analogously to the proof of~\Cref{lem:FeedForward}, barring the usage of \Cref{cor:DerivativeBoundByLevel}, we 
obtain the estimate
\begin{align}
\label{eq:pre_order_estimate}
    C^{\neuraln}_K(\cleq n)
    &\leq
        C^{B^{(1)}}
    +
        \ldim n! C^{B^{(2)}}
    C^\sigma_{{h}[K]}({\cleq n})
    (C^{A})^{n}
    \Big[ \frac{ 2 n}{e\ln n} ( 1 + o(1) ) \Big]^{n}
.
\end{align}
Let $C_n\eqdef\frac{ 2 n}{e\ln n} ( 1 + o(1) ) $
Using Stirling's approximation and the definition of the component-wise $\|\cdot\|_{\infty}$ norm of a matrix,~\eqref{eq:pre_order_estimate} becomes
\begin{align}
    C^{\neuraln}_K(\cleq n)
    &
    \in \mathcal{O}
    \Big(
            \|B^{(1)}\|_{\infty}
        +
            \|B^{(2)}\|_{\infty}
            \|A\|_{\infty}^n
            C^\sigma_{{h}[K]}({\cleq n})
            \ldim 
            \,
            n^{1/2}\big(\frac{n}{e}\big)^n
            C_n^n
    \Big)
.
\end{align}
If, there is some $M_1,M_2\le 0$, such that $K=[0,\beta]^d$ then using the estimate between the $\|\cdot\|_2$ and $\|\cdot\|_{\infty}$ norms on $\mathbb{R}^{\idim}$ and the linearity of $A$ we estimate 
\begin{align*}
    C^\sigma_{{h}[K]}({\cleq n}) 
\le 
    C^\sigma_{
        \operatorname{Ball}(a,\sqrt{\idim |M_1+M_2|})
    }({\cleq n}) 
\le 
    \|\sigma\|_{n:
        \operatorname{Ball}(a,\sqrt{\idim |M_1+M_2|})
    }
.
\end{align*}
Upon $\operatorname{Width}(\neuraln)\eqdef \max\{\idim,\odim,\ldim\}$, the estimate~\eqref{eq:pre_order_estimate} implies that $C^{\neuraln}_K(\cleq n)$ is of the order of
\begin{align}
\label{eq:pre_order_estimate__Kisabox}
    \mathcal{O}
    \Big(
    &
            \|B^{(1)}\|_{\infty}
        +
            \|B^{(2)}\|_{\infty}
            \|A\|_{\infty}^n
            \|\sigma\|_{n:
                    [-\|a\|_{\infty}-\sqrt{\idim |M_1+M_2|},
                    \|a\|_{\infty}+\sqrt{\idim |M_1+M_2|}]
                }
    \\
    \nonumber
    & \times 
            \operatorname{Width}(\neuraln)
            \,
            n^{1/2}\big(\frac{n}{e}\big)^n
            C_n^n
    \Big)
.
\end{align}
\end{proof}

\subsection{Step 2 (B) - Transformers}
\label{s:Proofs__TransformerCsNorms}
We may now merge the computations in Subsection~\ref{sec:RegularityTransfomerBuildingBlocks}, with the Fa`{a} di Bruno-type from Section~\ref{s:Compositions_and__FaadiBruno} to uniformly bound the $C^s$-norms of the relevant class transformer networks.  Our results are derived in two verions: the first is of ``derivative type'' (which is much smaller and more precise but consequentially more complicated) and the second is in ``derivative level'' form (cruder but simpler but also looser).   

\begin{theorem}[By Derivative Type]\label{thm:TransformerBlockSobolevBound} 
    Let $K$ be a compact set, $\tblock$ a transformer block as in \Cref{defn:TransformerBlock}, and $\alpha \in \Order^{M\idim}_n, n \in \N$. Then,
    \begin{align*}
        C^{\tblock}_K(\alpha) &\leq
        \alpha! \sum_{\beta \in \Order^{\odim}_{\leq n}}  N(\beta) C^{\layernorm}_{K\tagA}(\beta)
        \sum_{\eta, \zeta \in \calP'(\alpha,\beta)} 
        \prod_{j=1}^{n} \frac{
            C\tagA_K(\order(\zeta^{(j)}))^{\vert \eta^{(j)} \vert}
        }{
            \eta^{(j)}!(\zeta^{(j)}!)^{|\eta^{(j)}|}
        }, 
    \end{align*}
    where for all $\gamma \in \Order^{M\idim}_{\leq n}$:
    \begin{align*}
        C\tagA_K(\gamma) & \eqdef 
        \gamma! \sum_{\beta \in \Order^{\idim}_{\leq n}}  N(\beta) C^{\neuraln}_{K\tagB}(\beta)
        \sum_{\eta, \zeta \in \calP'(\gamma,\beta)} 
        \prod_{j=1}^{n} \frac{
            C\tagB_K(\order(\zeta^{(j)}))^{\vert \eta^{(j)} \vert}
        }{
            \eta^{(j)}!(\zeta^{(j)}!)^{|\eta^{(j)}|}
        },
    \\
        C\tagB_K(\gamma) & \eqdef
        \gamma! \sum_{\beta \in \Order^{\idim}_{\leq n}}  N(\beta) C^{\layernorm}_{K\tagC}(\beta)
        \sum_{\eta, \zeta \in \calP'(\gamma,\beta)} 
        \prod_{j=1}^{n} \frac{
            C\tagC_K(\order(\zeta^{(j)}))^{\vert \eta^{(j)} \vert}
        }{
            \eta^{(j)}!(\zeta^{(j)}!)^{|\eta^{(j)}|}
        },
    \\
        C\tagC_K(\gamma) & \eqdef \mathbbm{1}_{\vert \gamma \vert = 1} + C^{\multihead}_K(\gamma).
    \end{align*}
    In the above, $K\tagC = \bigcup_{m = 0}^M\multihead_m[K]$, $K\tagB = \layernorm[K\tagC]$, and $K\tagA = \neuraln[K\tagB]$.
    
    For respective multi-indices, a bound for $C^{\layernorm}_{K\tagA}, C^{\layernorm}_{K\tagC}$ is given by \Cref{lem:LayerNorm}, $C^{\neuraln}_{K\tagB}$ is bounded in \Cref{lem:FeedForward}, and a bound for $C^{\multihead}_{K}$ is given in \Cref{cor:MultiHead}.
\end{theorem}
\begin{proof}
    This is a direct consequence of \Cref{thm:DerivativeBoundByType}.
\end{proof}

\begin{theorem}[By Derivative Level]\label{thm:TransformerBlockSobolevBoundLevel} 
    Let $K$ be a compact set, $\tblock$ a transformer block as in \Cref{defn:TransformerBlock}, and $n \in \N$. Then,
    \begin{multline*}
        C^{\tblock}_K(\cleq n) \leq
        C^{\layernorm}_{K\tagA}({\cleq n})
        \left(
            \odim
            C^{\neuraln}_{K\tagB}({\cleq n})
        \right)^n
        \left(
            \idim^2
            C^{\layernorm}_{K\tagC}({\cleq n})
        \right)^{n^2}
        \\
        \cdot
        \left(
             1 + C^{\multihead}_K(\cleq n)
         \right)^{n^3}
        \Big[ \frac{ 2 n}{e\ln n} ( 1 + o(1) ) \Big]^{n + n^2 + n^3}
    \end{multline*}
    where, $K\tagC = \bigcup_{m = 0}^M\multihead_m[K]$, $K\tagB = \layernorm[K\tagC]$, and $K\tagA = \neuraln[K\tagB]$.
    
    A bound for $C^{\layernorm}_{K\tagA}, C^{\layernorm}_{K\tagC}$ is given by \Cref{cor:LayerNorm}, $C^{\neuraln}_{K\tagB}$ is bounded in \Cref{cor:FeedForward}, and a bound for $C^{\multihead}_{K}$ is given in \Cref{cor:MultiHead}.
\end{theorem}
\begin{proof}
    \Cref{cor:DerivativeBoundByLevel} yields
    \begin{align*}
        C^{\tblock}_K(\cleq n) &\leq
        C^{\layernorm}_{K\tagA}({\cleq n})
        C\tagA_K({\cleq n})^{n}
        \Big[ \frac{ 2\odim n}{e\ln n} ( 1 + o(1) ) \Big]^{n}.
    \end{align*}
    where 
    \begin{align*}
        C\tagA_K(\cleq n) & \eqdef 
        C^{\neuraln}_{K\tagB}({\cleq n})
        C\tagB_K({\cleq n})^{n}
        \Big[ \frac{ 2\idim n}{e\ln n} ( 1 + o(1) ) \Big]^{n},
    \\
        C\tagB_K(\cleq n) & \eqdef
        C^{\layernorm}_{K\tagC}({\cleq n})
        C\tagC_K({\cleq n})^{n}
        \Big[ \frac{ 2\idim n}{e\ln n} ( 1 + o(1) ) \Big]^{n},
    \\
        C\tagC_K(\cleq n) & \eqdef 1 + C^{\multihead}_K(\cleq n),
    \end{align*}
    which concludes the proof.
\end{proof}

\begin{theorem}[$C^s$-Norm Bound of Transformers]
\label{prop:transformerSobolevBound}
    Fix $n,L,H,C,D,d,M \in \N_+$ for a transformer class $\transformerclass$. 
    For any $\transformer \in \transformerclass$, 
    any compact $K_0 \subset \R^{M \times D}$, and any $\alpha \in \N^{M \times D}$, $\vert \alpha \vert \eqdef n$ we have
    \begin{align}\label{id:TransformerSobolevBound1}
        C^{\transformer}_{K_{0}}(\alpha)
        &\leq
        \odim^LM\alpha!
        \cdot
        C^A 
        \cdot
        C^L(\alpha),
    \end{align}
    where $C^{1}(\alpha) \eqdef C^{\tblock_1}_{K_{0}}(\alpha)$ and for $l \in \myset{2}{L}$, 
    \begin{align}\label{id:TransformerSobolevBound2}
        C^{l}(\alpha) \eqdef 
        &\leq
        \alpha! \sum_{\beta \in \Order^{\Tilde{d}_l}_{\leq n}}  N(\beta)C^{\tblock_l}_{K_{l-1}}(\beta)
        \sum_{\eta, \zeta \in \calP'(\order(\alpha),\beta)} 
        \prod_{j=1}^{n} \frac{
            C^l(\order(\zeta^{(j)}))^{\vert \eta^{(j)} \vert}
        }{
            \eta^{(j)}!(\zeta^{(j)}!)^{|\eta^{(j)}|}
        }
    \end{align}
    where $K_l \eqdef \tblock_l[K_{l-1}]$,  $\Tilde{d}_{l} \eqdef M_{l}\idim^{l}$, and a bound for $C^{\tblock_l}_{K_{l-1}}(\beta)$ is given by \Cref{thm:TransformerBlockSobolevBound}, only depending on the transformer block class $\tblockclass_l$.
\end{theorem}
\begin{proof}[Proof of {Theorem~\ref{prop:transformerSobolevBound}}]
    The bounds \eqref{id:TransformerSobolevBound1} are a direct consequence of \Cref{thm:DerivativeBoundByType} and \eqref{id:TransformerSobolevBound2} follows directly from \Cref{cor:DerivativeBoundAffine}.
\end{proof}

\subsection{Step 2 (C) - Merging the \texorpdfstring{$C^s$-Norm Bounds}{Regularity Estimates} for Transformers with the Loss Function}
\label{s:Proofs__MergingCsNorms}

In this section, we consider the following generalization of the class in Definition~\ref{def:sGrowthRate}.
{As before, each result holds for input dimensions $d$ just as much as any other input dimension, e.g.\ $Md$, with the only change being relabeling $d\gets Md$.  Therefore, for notational minimality, we chose to label the input dimension $d$ and not $dM$.}
\begin{definition}[Smoothness Growth Rate]
\label{def:sGrowthRate_v2}
Let $d,D\in \mathbb{R}$.  A smooth function $g:\mathbb{R}^d\to \mathbb{R}^D$ is said to belong to the class $C^{\infty}_{poly:C,r}(\mathbb{R}^d,\mathbb{R}^D)$ (resp.\ $C^{\infty}_{exp:C,r}(\mathbb{R}^d,\mathbb{R}^D)$) if there exist $C,r\ge 0$ such that: for each $s\in \mathbb{N}_+$ 
\begin{enumerate}
    \item[(i)] \textbf{Polynomial Growth -  $C^{\infty}_{poly:C,r}(\mathbb{R}^d,\mathbb{R}^D)$:} $\|g\|_{C^s} \le C\, s^{r}$,
    \item[(ii)] \textbf{Exponential Growth -  $C^{\infty}_{exp:C,r}(\mathbb{R}^d,\mathbb{R}^D)$:} $\|g\|_{C^s} \le C\, e^{s\,r}$,
\end{enumerate}
\end{definition}

The next lemma will help us relate the $C^s$-regularity of a model, a target function, and a loss function to their composition and product.  
We use it to relate the $C^s$-regularity of a transformed model $\mathcal{T}:\mathbb{R}^d\to \mathbb{R}^D$, the target function $f^{\star}:\mathbb{R}^d\to \mathbb{R}^D$, and the loss function $\ell:\mathbb{R}^{2D}\to \mathbb{R}$ to their composition
\begin{equation}
\begin{aligned}
\ell_{\mathcal{T}}:  \mathbb{R}^{d}& \to \mathbb{R} \\
x & \mapsto \ell\big(\mathcal{T}(x),f^{\star}(x)\big)
.
\end{aligned}
\end{equation}
One we computed have the $C^s$-regularity of $\ell_{\mathcal{T}}$, we can apply a concentration of measure-type argument based on an optimal transport-type duality, as in~\cite{amit2022integral,hou2023instance,benitez2023out,kratsios2024tighter}, to obtain our generalization bounds.  A key technical point where our analysis largely deviates from the mentioned derivations, is that we are not relying on any i.i.d.\ assumptions.

More generally, the next lemma allows us to bound the size of $\|\ell(\hat{f},f^{\star})\|_{C^s}$ using bounds on $C^s$ norms of $\transformerclass$ computed in Theorem~\ref{prop:transformerSobolevBound}, the target function $f^{\star}$, and on the loss function $\ell$.  Naturally, to use this result, we must assume a given level of regularity of the target function, as in Definition~\ref{def:sGrowthRate}.
\begin{lemma}[$C^s$-Norm of loss of between two functions]
\label{lem:control_loss__SSE}
Let $d,D,s\in \mathbb{N}_+$, $f_1,f_2:\mathbb{R}^d\to \mathbb{R}^D$ be of class $C^{s}$ and $\ell:\mathbb{R}^{2D}\to\mathbb{R}$ be smooth.  If there are constants $C_1,C_2, \widetilde{C}_1,\dots,\widetilde{C}_s\geq 0$ such that:
$\|f_i\|_{C^s} \le C_i$ for $i=1,2$ and for $j=1,\dots,s$ we have
$\|\ell\|_{C^j} \le \widetilde{C}_j$ then for all $s>0$ large it satisfies 
\begin{align} 
 \label{eq:called_bound} 
  \big\| \ell( f_1, f_2 ) \big\|_{C^s}  
 = 
 \begin{cases} 
  \mathcal{O} \big[ \big( \frac{ 2D s }{e\ln s } (1+o(1)) \big)^s  \big] , \, & \mbox{if $\max_{1\leq k } \widetilde{C}_k  (C_1 C_2 )^k$ is bounded} , \\ 
  \\ 
  \mathcal{O} \big[  \widetilde{C}_{s}   
   \big( C_1 C_2 \frac{ 2D s }{e\ln s } (1+o(1)) \big)^s  \big] , \, & \mbox{if $\max_{1\leq k } \widetilde{C}_k  (C_1 C_2 )^k$ is unbounded} .   
 \end{cases} 
\end{align} 
Particularly, if $\ell \in C^{\infty}_{poly:C,r}(\mathbb{R}^{2D},\mathbb{R})$,  
                    i.e., $\|\ell\|_{C^j}\le C\, j^{r}$ , then 
                    \begin{align} 
                    \big\| \ell( f_1, f_2 ) \big\|_{C^s} = 
\mathcal{O} \big[  C s^r \big( C_1 C_2 \frac{ 2D s }{e\ln s } (1+o(1)) \big)^s  \big] ;  
\label{eq:called_bound__polynomial_case}
                    \end{align}
if $\ell \in C^{\infty}_{exp:C,r}(\mathbb{R}^{2D},\mathbb{R})$, i.e., 
            $\|\ell\|_{C^j} \le C\,e^{j\,r}$, then 
            \begin{align} 
            \big\| \ell( f_1, f_2 ) \big\|_{C^s} = 
\mathcal{O} \big[  C e^{s r} \big( C_1 C_2 \frac{ 2D s }{e\ln s } (1+o(1)) \big)^s  \big] . 
\label{eq:called_bound__exponential_case}
            \end{align} 

\end{lemma}
Lemma~\ref{lem:control_loss__SSE} allows us to obtain a bound on the term $\sup_{\hat{f}\in C^s_R(\mathbb{R}^d)}\,\|\ell(\hat{f},f^{\star})\|_{C^s}$ in~\eqref{eq:risk_bound__decomposition2}, using Theorem~\ref{prop:transformerSobolevBound} and our assumptions on $\ell$ and on $f^{\star}$.  

\begin{proof} [{Proof of Lemma~\ref{lem:control_loss__SSE}}]
We first derive the general bound; which we then specialize to the case where the growth rate of $\ell$ is known. 
We first observe that
\begin{align} 
 \big\| \ell( f_1, f_2 ) \big\|_{C^s} 
  = 
  &
    \underbrace{
        \max_{k=1, \cdots, s-1} \max_{\alpha \in \{1, \cdots, d\}^k} \| D^\alpha \ell(f_1(x), f_2(x)) \|_\infty 
    }_{\term{t:A}}
\notag 
  \\
  + 
  &
  \underbrace{
    \max_{\alpha \in \{1, \cdots, d\}^{s-1} } \operatorname{Lip}\big(  D^\alpha \ell(f_1(x), f_2(x))  \big) 
    }_{\term{t:B}}
.
 \notag 
\end{align} 
\paragraph{{General Case - Term~\Cref{t:A}}:}
By \Cref{cor:DerivativeBoundByLevel}, we have
\allowdisplaybreaks 
\begin{align} 
\Big\| (D^\alpha \ell)(f_1(x), f_2(x)) \Big\|_\infty 
 \leq & 
 \big[ \max_{1\leq k \leq s-1} \widetilde{C}_{k}  (C_1 C_2 )^{ k } \big]  \cdot 
   \mathcal{O} \Big[ \Big( \frac{ 2D k }{e\ln k } ( 1+o(1) ) \Big)^{k} \Big]   , 
\label{eq:bound_t:A}
\end{align} 

From~\eqref{eq:bound_t:A} we have for all large $s>0$ that   
\begin{align} 
\begin{aligned}
&
 \max_{k=1, \cdots, s-1} \max_{\alpha \in \{1, \cdots, d\}^k} \| D^\alpha \ell(f_1(x), f_2(x)) \|_\infty 
 \\
 & = 
 \begin{cases} 
  \mathcal{O} \big[ \big( \frac{ 2D s }{e\ln s } (1+o(1)) \big)^s  \big] , \, & \mbox{if $\max_{1\leq k } \widetilde{C}_k  (C_1 C_2 )^k$ is bounded} , \\ 
  \\ 
  \mathcal{O}\Big[
    \widetilde{C}_s   
   \Big( C_1 C_2 \frac{ 2D s }{e\ln s } (1+o(1)) \Big)^s \Big]   , \, & \mbox{if $\max_{1\leq k } \widetilde{C}_k  (C_1 C_2 )^k$ is unbounded} .   
 \end{cases} 
\end{aligned}
\notag 
\end{align}

\paragraph{{General Case - Term~\Cref{t:B}}:}
For each $\alpha\in \{1, \cdots, d\}^{s-1}$, by the multivariate Fa\`a di Bruno formula, we have    
\begin{align}
        D^\alpha \ell(f_1(x), f_2(x)) 
        =
        \sum_{1 \leq |\beta| \leq s-1} (D^\beta \ell)(f_1(x), f_2(x)) 
        \sum_{\eta, \zeta \in \calP(\alpha,\beta)} \alpha! 
        \prod_{j=1}^{s-1} \frac{
            [D^{\zeta^{(j)}}(f_1(x), f_2(x))]^{\eta^{(j)}}
        }{
            \eta^{(j)}!(\zeta^{(j)}!)^{|\eta^{(j)}|}
        }. 
        \notag 
    \end{align} 
The Lipschitz constants of the derivatives satisfy 
\allowdisplaybreaks
\begin{align}
    &  \operatorname{Lip} \big(  D^\alpha \ell(f_1(x), f_2(x)) \big)  \notag \\ 
    = & 
        \sum_{1 \leq |\beta| \leq s-1} \operatorname{Lip} \big( (D^\beta \ell)(f_1(x), f_2(x)) \big)   
        \sum_{\eta, \zeta \in \calP(\alpha,\beta)} \alpha! 
        \prod_{j=1}^{s-1} \frac{
           \operatorname{Lip}\big( [D^{\zeta^{(j)}}(f_1(x), f_2(x))]^{\eta^{(j)}} \big) 
        }{
            \eta^{(j)}!(\zeta^{(j)}!)^{|\eta^{(j)}|}
        } 
        \notag \\ 
    \leq  & 
        \sum_{1 \leq |\beta| \leq s-1} \widetilde{C}_{|\beta|+1}   
        \sum_{\eta, \zeta \in \calP(\alpha,\beta)} \alpha ! 
        \prod_{j=1}^{s-1} \frac{
           (C_1 C_2)^{|\eta(j)|}  
        }{
            \eta^{(j)}!(\zeta^{(j)}!)^{|\eta^{(j)}|}
        } \notag \\ 
   = & 
        \sum_{1 \leq |\beta| \leq s-1} \widetilde{C}_{|\beta|+1}   (C_1 C_2)^{|\beta |}
        \sum_{\eta, \zeta \in \calP(\alpha,\beta)} \alpha ! 
        \prod_{j=1}^{s-1} \frac{ 1 }{
            \eta^{(j)}!(\zeta^{(j)}!)^{|\eta^{(j)}|}
        } \notag \\ 
    \leq & 
    \big[ \max_{1\leq k \leq s-1} \widetilde{C}_{k+1}   (C_1 C_2)^{k} \big]  
    \sum_{1 \leq |\beta| \leq s-1} 
        \sum_{\eta, \zeta \in \calP(\alpha,\beta)} \alpha ! 
        \prod_{j=1}^{s-1} \frac{ 1 }{
            \eta^{(j)}!(\zeta^{(j)}!)^{|\eta^{(j)}|}
        }  \notag  \\ 
    = & 
 \big[ \max_{1\leq k \leq s-1} \widetilde{C}_{k+1}  (C_1 C_2 )^{ k } \big]  \cdot 
  \mathcal{O} \Big[ \Big( \frac{ 2D s }{e\ln s} (1+o(1)) \Big)^s  \Big] , 
\label{eq:bound_t:B}
    \end{align} 
where the last equality is due to Lemma~\ref{lem:mdFDBestimate}. 

From~\eqref{eq:bound_t:B} we have for all $s>0$ large that 
\begin{align} 
\begin{aligned}
  &\max_{\alpha \in \{1, \cdots, d\}^{s-1} } \operatorname{Lip}\big(  D^\alpha \ell(f_1(x), f_2(x))  \big) 
\\
 & = 
 \begin{cases} 
  \mathcal{O} \big[ \big( \frac{ 2D s }{e\ln s } (1+o(1)) \big)^s  \big] , \, & \mbox{if $\max_{1\leq k } \widetilde{C}_k  (C_1 C_2 )^k$ is bounded} , \\ 
  \\ 
  \mathcal{O} \big[  \widetilde{C}_{s}   
   \big( C_1 C_2 \frac{ 2D s }{e\ln s } (1+o(1)) \big)^s  \big] , \, & \mbox{if $\max_{1\leq k } \widetilde{C}_k  (C_1 C_2 )^k$ is unbounded} .   
 \end{cases} 
\end{aligned}
\notag 
\end{align}



\paragraph{Completing the General Case:}
Combining our estimates for terms~\Cref{t:A} and~\Cref{t:B} respectively obtained in~\eqref{eq:bound_t:A} and~\eqref{eq:bound_t:B},
we obtain an upper-bound for $\| \ell( f_1, f_2 ) \|_{C^s}$ via 
\begin{align} 
  \big\| \ell( f_1, f_2 ) \big\|_{C^s}  
 = 
 \begin{cases} 
  \mathcal{O} \big[ \big( \frac{ 2D s }{e\ln s } (1+o(1)) \big)^s  \big] , \, & \mbox{if $\max_{1\leq k } \widetilde{C}_k  (C_1 C_2 )^k$ is bounded} , \\ 
  \\ 
  \mathcal{O} \big[  \widetilde{C}_{s}   
   \big( C_1 C_2 \frac{ 2D s }{e\ln s } (1+o(1)) \big)^s  \big] , \, & \mbox{if $\max_{1\leq k } \widetilde{C}_k  (C_1 C_2 )^k$ is unbounded} .   
 \end{cases} \notag 
\end{align}

\paragraph{Special Cases of Interest:}
In particular, if $\ell$ belongs either to $C^{\infty}_{poly:C,r}(\mathbb{R}^d,\mathbb{R}^D)$ or to $C^{\infty}_{exp:C,r}(\mathbb{R}^d,\mathbb{R}^D)$, , as in Definition~\eqref{def:sGrowthRate}, then: there exists constants $C_{\ell},r_{\ell}>0$ s.t.\ for each $j=1,\dots,s$ we have
\begin{enumerate}
    \item[(i)] \textbf{Polynomial Growth -  $C^{\infty}_{poly:C,r}(\mathbb{R}^{2D},\mathbb{R})$ Case:} 
    \begin{align*}
            \|\ell\|_{C^j}
        \le
            C\, j^{r}
        \eqdef 
            \tilde{C}_j
    ,
    \end{align*}
    \item[(ii)] \textbf{Exponential Growth -  $C^{\infty}_{exp:C,r}(\mathbb{R}^{2D},\mathbb{R})$ Case:} 
    \begin{align*}
            \|\ell\|_{C^j} 
        \le 
            C\,e^{j\,r}
        \eqdef 
            \tilde{C}_j
    .
    \end{align*}
\end{enumerate}
Consequentially, in cases (i) and (ii), the bound in~\eqref{eq:called_bound} respectively becomes
\begin{enumerate}
    \item[(i)] \textbf{Polynomial Growth -  $C^{\infty}_{poly:C,r}(\mathbb{R}^{2D},\mathbb{R})$ Case:} 
\begin{align}
 \big\| \ell( f_1, f_2 ) \big\|_{C^s} 
  \le & 
    \mathcal{O} \big[  C s^r   
   \big( C_1 C_2 \frac{ 2D s }{e\ln s } (1+o(1)) \big)^s  \big] , 
\notag 
\end{align}
    \item[(ii)] \textbf{Exponential Growth -  $C^{\infty}_{exp:C,r}(\mathbb{R}^{2D},\mathbb{R})$ Case:} 
\begin{align}
 \big\| \ell( f_1, f_2 ) \big\|_{C^s} 
  \le & 
    \mathcal{O} \big[  C e^{s r}   
   \big( C_1 C_2 \frac{ 2D s }{e\ln s } (1+o(1)) \big)^s  \big] . 
\notag 
\end{align}
\end{enumerate}

\end{proof}

\subsection{Step 3 - Combining Steps 1 and 2 and Completing The Proof of \texorpdfstring{Theorem~\ref{thrm:main}}{Our Main Theorem}}
\label{s:Proof__ss:Main}

We are now ready to complete the proof of our main result, namely Theorem~\ref{thrm:main}.  
Before doing so, we state a more technical and general version, which we instead prove and which directly implies the simpler version found in the main body of our manuscript.

We operate under the following more general, but more technical set of assumptions than those considered in the main body of our text (in Setting~\ref{setting}).
\begin{setting}[Generalized Setting]
\label{setting_gensetting}
Let $D,d, L,H,*C^{\prime},C^A,C^b\in \mathbb{N}_+$, set $M\eqdef 0$, and $C \eqdef (*C^{\prime}$ $,C^A,C^b)$, $r_f,r_{\ell},C_f,C_{\ell} \ge 0$.
Suppose that Assumptions~\ref{ass:MC} and~\ref{ass:Exp_Moment} hold.

Fix a target function $f^{\star}:\mathbb{R}^d\to \mathbb{R}^D$ and a loss function $\ell:\mathbb{R}^D\times \mathbb{R}^D\to \mathbb{R}$.  Assume either that:
\begin{enumerate}
    \item[(i)] \textbf{Polynomial Growth:} $f^{\star}\in C^{\infty}_{poly:C_f,r_f}(\mathbb{R}^d,\mathbb{R}^D)$ and $\ell\in C^{\infty}_{poly:C_{\ell},r_{\ell}}(\mathbb{R}^{2D},\mathbb{R})$,
    \item[(ii)] \textbf{Exponential Growth:} $f^{\star}\in C^{\infty}_{exp:C_f,r_f}(\mathbb{R}^d,\mathbb{R}^D)$ and $\ell\in C^{\infty}_{exp:C_{\ell},r_{\ell}}(\mathbb{R}^{2D},\mathbb{R})$
    ,
    \item[(iii)] \textbf{No Growth:} There is a constant $\bar{C}\ge 0$ such that for all $s>0$ we have $\|f^{\star}\|_{C^s},\|\ell\|_{C^s} \le \bar{C}$.
\end{enumerate}
\end{setting}

\begin{example}[Example of Generalized Setting (iii)]
\label{ex:ubderivatives}
For every $d\in \mathbb{R}^d$, the function $f:\mathbb{R}^d\ni x\mapsto \cos\bullet x = \big(\cos(x_i)\big)_{i=1}^d$ satisfies $
\|\frac{\partial^s}{\partial x_i^s} f\|_{\infty} \le 1
$ for each $s\in \mathbb{N}$ and each $i=1,\dots,d$.  Thus, it is an example of a function satisfying Assumption~\ref{setting_gensetting}.
\end{example}

\begin{table}[ht!]
    \centering
	\ra{1.3}
    \caption{Bounds on the terms in defining the constant $C_{\ell,\transformerclass,K,s}$, in Theorem~\ref{thrm:main__technicalversion}, for a single attention block. }
\begin{tabular}{@{}cl@{}}
\cmidrule[0.3ex](){1-2}
Term & Bound $(\mathcal{O})$
\\    
\midrule
$c_{\ell,f^{\star}}$ & $
            C_f^s
            \,
            s^{r_{\ell} + 2s^2} 
            C_s^s 
            $
\\
$\layernorm$
 & 
 $   
    s^{(1+s)/2}
    C_s^s
 $
\\ 
 $\neuraln$ & {\footnotesize $
                \|B^{(1)}\|
            +
                \|B^{(2)}\|
                \|A\|^s
                \|\sigma\|_{s:
                        \big[\pm \|a\|_{\infty}\pm\sqrt{\idim}\big]
                    }
                \operatorname{Width}(\neuraln)
                \,
                \tilde{C}_s^s
        $
        }
\\
 $\multihead$  & 
$
    \|W\|
    \|V\|
        (\Tilde{d} \|Q\|\|K\|)^{s}
    \,\,
        \Big(
            s^2
             \big(\frac{s}{e}\big)^{2s}
            C_s^s
        \Big)
$
\\
\bottomrule
\end{tabular}
\caption*{Here $C_s\eqdef \frac{ 2 s}{e\ln s} ( 1 + o(1) )$, $\tilde{C}_s\eqdef s^{1/2}\big(\frac{n}{e}\big)^sC_s^s$, $c_d\eqdef 2\max
\{\idim, \kdim, \vdim, \ldim, \odim\}$, $\operatorname{Width}(\neuraln)$ is the width of the neural network $\neuraln$, where $\|\cdot\|$ denotes the componentwise max matrix/vector norm.}
\label{tab:constants}
\end{table}

We are now ready to prove our main theorem, which is a combination of \Cref{thrm:main,thrm:main_b}.

\begin{theorem}[Pathwise Generalization Bounds for Transformers]
\label{thrm:main__technicalversion}
In Setting~\ref{setting_gensetting}, there is a $\kappa\in(0,1)$, depending only on $X_{\cdot}$, and a $t_0\in \mathbb{N}_0$ such that: for each $t_0\le N \le t \le \infty$ and $\delta \in (0,1]$ the following holds with probability at-least $1-\delta$
\begin{align*}
\resizebox{1\linewidth}{!}{$
        \sup_{\mathcal{T}\in 
            \transformerclass
        }
        \,
            \big|
                \mathcal{R}_{\max\{t,N\}}(\mathcal{T})-\mathcal{R}^{(N)}(\mathcal{T})
            \big|
    \lesssim
        \sum_{s=1}^{\infty}\,
            I_{N\in [\tau_{s},\tau_{s+1})}
            \,
            C_{\ell,\transformerclass,K,s-1}
        \biggl(
                I_{t<\infty}\,
                \kappa^t\,
            +
                \frac{
                    \sqrt{2\ln(1/\delta)}
                }{
                    N^{1/2}
                }
                +\,
                \operatorname{rate}_s(N)
    \biggr)
$}
\end{align*}
where $\operatorname{rate}_s(N)$ is defined in~\eqref{eq:rate_function__definition}, the constant 
$
C_{\ell,\transformerclass,K,s} \eqdef \sup_{\transformer \in \transformerclass} \Vert \ell(\transformer, f^*) \Vert_{C^s}, 
$ is of order 
\begin{align*}
\resizebox{0.95\linewidth}{!}{$
    \mathcal{O}\Big(
        \underbrace{
            C^{\ell,f^{\star}}
            \vphantom{\biggl(\biggr)^{s^4}}
        }_{\text{Loss \& Target}}
        \,
        \underbrace{
            C^{\layernorm}_{K\tagA}({\cleq s})^s
            C^{\layernorm}_{K\tagC}({\cleq s})^{s^3}
            \vphantom{\biggl(\biggr)^{s^4}}
        }_{\text{Layernorms}}
        \underbrace{
            C^{\neuraln}_{K\tagB}({\cleq s})^{s^2}
            \vphantom{\biggl(\biggr)^{s^4}}
        }_{\text{Perceptron}}
        \,
            \underbrace{
                \biggl(
                    1 + C^{\multihead}_K(\cleq s)
                \biggr)^{s^4}
                \vphantom{\biggl(\biggr)^{s^4}}
             }_{\text{Multihead Attention}}
        \underbrace{
            D^{s^2}\,d^{2s^3}
            \vphantom{\biggl(\biggr)^{s^4}}
        }_{\text{dimensions}}
        \,
        \underbrace{
             c_s
             ^{s^s + s^3 + s^4}
             \vphantom{\biggl(\biggr)^{s^4}}
        }_{\text{Generic: $s$-th order Derivative}}
    \Big)
$}
\end{align*}
with terms according to Table~\ref{tab:constants} and the \textit{transition phases} $(\tau_s)_{s=0}^{\infty}$ are given iteratively by $\tau_0\eqdef 0$ and for each $s\in \mathbb{N}_+$
\begin{align*}
\resizebox{1\hsize}{!}{$
    \tau_s
\eqdef 
    \inf\,
    \biggl\{ 
        t\ge \tau_{s-1}
    :\,
            C_{\ell,\transformerclass,K,s}
            ( \kappa^t + \operatorname{rate}_s(N)+
            \frac{\sqrt{\log(1/\delta)}}{\sqrt{N}}
            ) 
        \le 
            C_{\ell,\transformerclass,K,s-1}
            ( \kappa^t + \operatorname{rate}_{s-1}(N)+
            \frac{\sqrt{\log(1/\delta)}}{\sqrt{N}}
            )
    \biggr\}
.
$}
\end{align*}
Furthermore, $c\eqdef 1-\kappa$, $c_2\eqdef c^{s/d}$, $\kappa^{\infty}\eqdef \lim\limits_{t\to \infty}\, \kappa^t= 0$, and $\lesssim$ hides an absolute constant.
\end{theorem}

\begin{proof}[{Proof of Theorem~\ref{thrm:main}}]
Since $N$ is given, we may pick $s\in \mathbb{N}_+$ to ensure that $N\in [\tau_s,\tau_{s+1})$; where these are defined as in the statement of Theorem~\ref{thrm:main__technicalversion}.  

Since we are in Setting~\ref{setting}, then $\ell \in C^{\infty}_{poly:C_{\ell},r_{\ell}}(\mathbb{R}^{2D},\mathbb{R})$ (resp.\ $\ell \in C^{\infty}_{exp:C_{\ell},r_{\ell}}(\mathbb{R}^{2D},\mathbb{R})$) and $f^{\star}:\mathbb{R}^d\to\mathbb{R}^D$ is smooth.  
Therefore, Lemma~\ref{lem:control_loss__SSE} implies that there is an absolute constant $c_{\operatorname{abs}}>0$ such that for any transformer network $\transformer\in \transformerclass$, the following bound holds
\begin{enumerate}
    \item[(i)] \textbf{No Growth Case:} Using~\eqref{eq:called_bound} we find that
\begin{align}
\label{eq:FCsBound__NoGrowth}
    \big\| \ell( \transformer, f^{\star} ) \big\|_{C^s} 
\le & 
    c_{\operatorname{abs}}\,
        \big(\frac{ 2D s }{e\ln s } (1+o(1)) \big)^s  
        \,
        \|\transformer\|_{C^s}^s
\end{align}
    \item[(ii)] \textbf{Polynomial Growth Case -  $\ell \in C^{\infty}_{poly:C_{\ell},r_{\ell}}(\mathbb{R}^{2D},\mathbb{R})$ Case:} 
\begin{align}
\label{eq:FCsBound__PolynomialCase}
    \big\| \ell( \transformer, f^{\star} ) \big\|_{C^s} 
\le & 
    c_{\operatorname{abs}}\,
            s^{r_{\ell}}   
            \big(\frac{ 2D s }{e\ln s } (1+o(1)) \big)^s  
            \,
            \|f^{\star}\|_{C^s}^s
            \,
            \|\transformer\|_{C^s}^s
\end{align}
    \item[(iii)] \textbf{Exponential Growth -  $C^{\infty}_{exp:C_{\ell},r_{\ell}}(\mathbb{R}^{2D},\mathbb{R})$ Case:} 
\begin{align}
\label{eq:FCsBound__ExponentialCase}
    \big\| \ell( \transformer, f^{\star} ) \big\|_{C^s} 
\le & 
    c_{\operatorname{abs}}
    \,
            e^{s \, r_{\ell}}   
            \big(\frac{ 2D s }{e\ln s } ( 1+o(1) ) \big)^s  
            \,
            \|f^{\star}\|_{C^s}^s
            \,
            \|\transformer\|_{C^s}^s
.
\end{align}
\end{enumerate}
Since we have assumed that $f^{\star} \in C^{\infty}_{poly:C_f,r_f}(\mathbb{R}^d,\mathbb{R}^{D})$ (resp.\ $C^{\infty}_{exp:C_f,r_f}(\mathbb{R}^d,\mathbb{R}^{D})$ or the ``no growth condition'' in Setting~\ref{setting_gensetting} (iii)) then the bounds in~\eqref{eq:FCsBound__NoGrowth},~\eqref{eq:FCsBound__PolynomialCase}, and~\eqref{eq:FCsBound__ExponentialCase}, respectively, imply that
\begin{enumerate}
    \item[(i)] \textbf{No Growth Case:} 
    \begin{align}
    \label{eq:FCsBound__NoGrowth__v2}
        \big\| \ell( \transformer, f^{\star} ) \big\|_{C^s} 
    \le & 
        c_{\operatorname{abs}}\,
            \big(\frac{ 2D s }{e\ln s } (1+o(1)) \big)^s  
            \,
            C^{\transformerclass}_{K}(s)^s
    \end{align}
    \item[(ii)] \textbf{Polynomial Growth Case -  $\ell \in C^{\infty}_{poly:C_{\ell},r_{\ell}}(\mathbb{R}^{2D},\mathbb{R})$ Case:} 
\begin{align}
\label{eq:FCsBound__PolynomialCase__v2}
    \big\| \ell( \transformer, f^{\star} ) \big\|_{C^s} 
\le & 
    c_{\operatorname{abs}}\,
            s^{r_{\ell} + 2s^2} 
            \Big(\frac{ C_f \, 2D  }{e\ln s } (1 + o(1)) \Big)^s  
            \,
            C^{\transformerclass}_{K}(s)^s
\end{align}
    \item[(iii)] \textbf{Exponential Growth -  $C^{\infty}_{exp:C_{\ell},r_{\ell}}(\mathbb{R}^{2D},\mathbb{R})$ Case:} 
\begin{align}
\label{eq:FCsBound__ExponentialCase__v2}
    \big\| \ell( \transformer, f^{\star} ) \big\|_{C^s} 
\le & 
    c_{\operatorname{abs}}\,
            e^{s \, r_{\ell} + s^2 r_f} \Big(\frac{ 2D \, s }{ e \ln s} (1+o(1) )\Big)^s
            \,
            C_f^s  
            \,
            C^{\transformerclass}_{K_{0}}(s)^s
,
\end{align}
\end{enumerate}
where we have used the definition of the constant $C^{\transformerclass}_{K}(s)$ as a uniform upper bound of $\sup_{\transformer \in \transformerclass}$.  
Using Theorem~\ref{thm:TransformerBlockSobolevBound} for the ``derivative type estimate'' (resp.\ref{thm:TransformerBlockSobolevBoundLevel} for the ``derivative level estimate'') concludes the implies yields a uniform upper bound (of ``derivative type'' or ``derivative level'' respectively) on $C^{\transformerclass}_{K_{0}}(s)$, i.e.\ independent of the particular transformer instance $\mathcal{T}\in \transformerclass$.  In either case, we respectively define $R>0$ to be the right-hand side of~\eqref{eq:FCsBound__PolynomialCase__v2} or~\eqref{eq:FCsBound__ExponentialCase__v2} depending on the respective assumptions made on $\ell$ and on $f^{\star}$.

The conclusion now follows upon applying Proposition~\ref{prop:Concentation} due to the inequality in~\eqref{eq:risk_bound__decomposition}.    
\end{proof}

\section{Example of Additive Noise Using Stochastic Calculus}
\label{a:stoch_example}

In this appendix, we briefly discuss why the seemingly \textit{realizable} learning setting which we have placed ourselves in, i.e.~$Y_n=f^{\star}(X_n)$, does not preclude additive noise.  Our illustration considers the class of following Markov processes.
\begin{assumption}[Structure on $X_{\cdot}$]
\label{ass:Markov_Structure_X}
Let $g:\mathbb{R}^d\to [0,1]^d$ be a twice continuously differentiable function.  Let $W_{\cdot}\eqdef (W_t)_{t\ge 0}$ be $d$-dimensional Brownian motion and, for each $n\in \mathbb{N}$, define
\begin{align*}
    X_n \eqdef g(W_n).
\end{align*}
\end{assumption}
By construction, the boundedness of the change of variables-type function $g$ in Assumption~\ref{ass:Markov_Structure_X}, implies that the process $X_{\cdot}=(X_n)_{n\in \mathbb{N}}$ is bounded (and can easily be seen to be Markovian since Brownian motion has the strong Markov property).  However, we can say more, indeed under Assumption~\ref{ass:Markov_Structure_X}, the It\^{o} Lemma (see e.g.~\citep[Theorem 14.2.4]{CohenElliotBook_2015}) implies that $X_n$ is given as the following stochastic differential equation (SDE) evaluated at integer times $n\in \mathbb{N}$
\begin{align}
\label{eq:SDE}
X_n  
&=
    g(0) 
+ 
    \int_0^n\,
        \mu_s
  \, ds
+ 
    \int_0^n\,
        \sigma_t^{\top} \, dW_s
\end{align}
where $\mu_{\cdot}=(\mu_t)_{t\ge 0}$ and $\sigma_{\cdot}=(\sigma_t)_{t\ge 0}$ are given by
\begin{align*}
\mu_t  \eqdef \frac{1}{2} \operatorname{tr} \big( H(g)(W_s) \big)  
\mbox{ and }
\sigma_t  \eqdef \nabla g(W_t)
\end{align*}
and $H(g)$ is the Hessian of $g$ and $\operatorname{tr}$ is the trace of a matrix.
\begin{example}
\label{ex:simple_example}
Set $d=1$ and $g(x)=(\sin(x)+1)/2$.  Then, for each $n\in \mathbb{N}$ we have 
\begin{align*}
X_n = \int_0^n -\sin(W_s)/4 ds + \int_0^t \cos(W_s)/2 \,dW_s.
\end{align*}
\end{example}
In particular, the expression~\eqref{eq:SDE} shows that the input process $X_{\cdot}$ is also defined for all intermediate times between non-negative integer times; i.e.~for each $t\ge 0$ the process
\begin{align}
\label{eq:SDE_cnttimeext}
    X_t
&=
    g(0) 
+ 
    \int_0^t\,
        \mu_s
  \, ds
+ 
    \int_0^t\,
        \sigma_t^{\top} \, dW_s
\end{align}
is well-defined and coincides with $X_n$ whenever $t=n\in \mathbb{N}$.  
We may, therefore, also consider the ``continuous-time extension'' $Y_{\cdot}\eqdef (Y_t)_{t\ge 0}$ of the target process defined for all intermediate times using~\eqref{eq:SDE_cnttimeext} by
\begin{align*}
        Y_t
    \eqdef 
        f^{\star}(X_t)
    .
\end{align*}
Note that $Y_t$ coincides with the target process on non-negative integer times, as defined in our main text, by definition.  

The convenience of these continuous-time extensions, of the discrete versions considered in our main text, is that now $Y_{\cdot}$ is the transformation of a continuous-time (It\^{o}) process of satisfying the SDE~\eqref{eq:SDE_cnttimeext} by a smooth function\footnote{Note that $f^{\star}$ was assumed to be smooth in our main result (Theorem~\ref{thrm:main}).}, namely $f^{\star}$.  Therefore, we may again apply the It\^{o} Lemma (again see e.g.~\citep[Theorem 14.2.4]{CohenElliotBook_2015}) this time to the process $X_{\cdot}$ to obtain the desired signal and noise decomposition of the target process $Y_{\cdot}$ (both in discrete and continuous time).  Doing so yields the following  decomposition
\begin{equation}
\label{eq:signal_noise_decomposition}
\begin{aligned}
    Y_t 
= &
    \underbrace{
        f^{\star}(X_0)
        +
        \int_0^t
        \,
            \Big(
                \left (\nabla f^{\star}(X_s) \right)^{\top}
                \mu_t 
                + 
                \frac{1}{2} 
                \operatorname{tr} \big(
                    \sigma_s^{\top} 
                    H(f^{\star})(X_s)
                    \sigma_s 
                \big)
            \Big) 
        \, 
        ds
    }_{\text{Signal (Target)}}
    \\
    & + 
    \underbrace{
    \int_0^t
    \,
        \left (\nabla f^{\star} \right)^{\top} \sigma_s
    \, dW_s
    }_{\text{Additive Noise}}
.
\end{aligned}
\end{equation}
This shows that even if it a priori seemed that we are in the \textit{realizable PAC setting} due to the structural assumption that $Y_n=f^{\star}(X_n)$ made when defining the target process, we are actually in the standard setting where the target data $(Y_n)_{n=0}^{\infty}$ can be written as a signal plus an additive noise term.  Indeed, when $X_{\cdot}$ is simply a transformation of a Brownian motion by a bounded $C^2$-function, as in Assumption~\ref{ass:Markov_Structure_X}, then Assumption~\ref{ass:Comp_Support} held and $Y_n$ admitted the signal-noise decomposition in~\eqref{eq:signal_noise_decomposition}.




\end{document}